\documentclass{article}

\usepackage{microtype}
\usepackage{graphicx}
\usepackage{subcaption}
\usepackage{booktabs} 
\usepackage{multirow}

\usepackage[dvipsnames]{xcolor}

\usepackage{aliascnt}

\usepackage{hyperref}

\usepackage[preprint]{icml2026}

\usepackage{amsmath}
\usepackage{amssymb}
\usepackage{mathtools}
\usepackage{amsthm}

\usepackage{enumerate}
\usepackage{enumitem}

\usepackage[capitalize,noabbrev]{cleveref}

\definecolor{cmnistLightBlue}{HTML}{3F68E0}
\definecolor{cmnistBlue}{HTML}{0000FF}
\definecolor{cmnistRed}{HTML}{FD4400}
\definecolor{cmnistBeige}{HTML}{FDDCAC}
\definecolor{shapesBlue}{HTML}{0067FD}
\definecolor{shapesLightGreen}{HTML}{00FC71}
\definecolor{shapesRed}{HTML}{FD0000}
\definecolor{shapesPink}{HTML}{FE009F}
\definecolor{shapesDarkBlue}{HTML}{3E00FE}

\theoremstyle{plain}

\newaliascnt{proposition}{theorem}
\newtheorem{proposition}[proposition]{Proposition}
\aliascntresetthe{proposition}

\theoremstyle{definition}

\theoremstyle{remark}

\usepackage[textsize=tiny]{todonotes}
\usepackage{comment}

\usepackage{bm}
\usepackage{stmaryrd}

\usepackage{environ}

\newif\ifmoveprooftoend
\newcommand\showDeferredProofs{}

\makeatletter
\NewEnviron{movableProof}[2][Proof]{%
        \ifmoveprooftoend
                \edef\next{\noexpand\g@addto@macro\noexpand\showDeferredProofs{%
                        \noexpand\cref{#2}\begin{proof}[#1]\unexpanded\expandafter{\BODY}\noexpand\end{proof}}}
                \next
        \else
                \begin{proof}[#1]\BODY\end{proof}
        \fi}{}
\makeatother
\moveprooftoendtrue

\usepackage[most]{tcolorbox}
\tcbset{
  mytheobox/.style={
    enhanced,
    colback=gray!10,      
    colframe=gray!40,     
    boxrule=0.4pt,
    arc=2pt,
    left=6pt,right=6pt,top=2pt,bottom=2pt,
    before skip=8pt,
    after skip=8pt,
  }
}

\newcommand{\bX}{\bm{X}}      %
\newcommand{\bx}{\bm{x}}      %
\newcommand{\bW}{\bm{W}}      
\newcommand{\RR}{\mathbb{R}}  

\usepackage{empheq}
\newcommand*\widefbox[1]{\fbox{\hspace{2em}#1\hspace{2em}}}

\newcommand{\method}{\textsc{LoGDiff}}

\icmltitlerunning{Logical Guidance for Diffusion Models}

\begin{document}

\twocolumn[

  \icmltitle{
  Logical Guidance for the Exact Composition of Diffusion Models
  }



  \icmlsetsymbol{equal}{*}

  \begin{icmlauthorlist}
    \icmlauthor{Francesco Alesiani}{equal,nle}
    \icmlauthor{Jonathan Warrell}{equal,nla}
    \icmlauthor{Tanja Bien}{equal,ustutt} \\
    \icmlauthor{Henrik Christiansen}{nle}
    \icmlauthor{Matheus Ferraz }{noi} 
    \icmlauthor{Mathias Niepert}{nle,ustutt}
  \end{icmlauthorlist}

  \icmlaffiliation{nle}{NEC Laboratories Europe}
  \icmlaffiliation{ustutt}{University of Stuttgart}
  \icmlaffiliation{nla}{NEC Laboratories America}
  \icmlaffiliation{noi}{NEC OncoImmunity}  

  \icmlcorrespondingauthor{Francesco Alesiani}{francesco.alesiani@neclab.eu}
  \icmlcorrespondingauthor{Jonathan Warrell}{jwarrell@nec-labs.com}    
  \icmlcorrespondingauthor{Tanja Bien}{tanja.bien@ki.uni-stuttgart.de} 
  \icmlcorrespondingauthor{Mathias Niepert}{mathias.niepert@ki.uni-stuttgart.de}

  \icmlkeywords{Diffusion Models, Probabilistic Circuits, Antibody design}

  \vskip 0.3in
]



\printAffiliationsAndNotice{\icmlEqualContribution}

\begin{abstract}  
We propose \method{} (\textbf{Lo}gical \textbf{G}uidance for the Exact Composition of \textbf{Diff}usion Models), a guidance framework for diffusion models that enables principled constrained generation with complex logical expressions at inference time. We study when exact score-based guidance for complex logical formulas can be obtained from guidance signals associated with atomic properties. First, we derive an exact Boolean calculus that provides a sufficient condition for exact logical guidance. Specifically, if a formula admits a circuit representation in which conjunctions combine conditionally independent subformulas and disjunctions combine subformulas that are either conditionally independent or mutually exclusive, exact logical guidance is achievable. In this case, the guidance signal can be computed exactly from atomic scores and posterior probabilities using an efficient recursive algorithm. Moreover, we show that, for commonly encountered classes of distributions, any desired Boolean formula is compilable into such a circuit representation. Second, by combining atomic guidance scores with posterior probability estimates, we introduce a hybrid guidance approach that bridges classifier-guidance and classifier-free guidance, applicable to both compositional logical guidance and standard conditional generation. We demonstrate the effectiveness of our framework on multiple image and protein structure generation tasks.
\end{abstract}

\section{Introduction}

Diffusion models have achieved remarkable success in generating high-fidelity data across diverse modalities, from image and video generation \cite{du:2023, liu:2021, liu:2022, zhu:2024} to protein design \cite{yang2024rethinking, abramson2024accurate} and planning \cite{ajay:2023, janner:2022}. A key property of these models is their steerability, which allows controlling outputs at inference time via guidance \cite{ho:2021, dhariwal:2021, bansal:2023}. This capability extends to compositional generation, where complex guidance terms are composed from combinations of individual concepts \cite{hinton:2002, liu:2022, du:2023}. 

Existing compositional guidance methods, however, remain limited.
Most approaches combine conditions by heuristically averaging conditional outputs
\citep{liu:2022}. While effective for simple conjunctions, such
heuristics fail to capture the structure of general logical reasoning, and in
particular do not extend naturally to disjunctions, negations, or more complex
Boolean expressions. Other recent approaches consider the superposition of
distributions \citep{skreta:2025}, but do not provide a general framework for compositional reasoning. As a result, there is no formal framework for translating complex logical expressions into corresponding guidance dynamics for diffusion models.

To address this issue, we introduce \method{}, 
\emph{\textbf{Lo}gical \textbf{G}uidance for the Exact Composition of \textbf{Diff}usion Models},
a framework connecting Boolean logic and compositional diffusion. We formalize logical constraints as probabilistic events, deriving an exact Boolean calculus in which the combination of conditional outputs is dynamic and depends on the time-varying probability of clauses, rather than on constant weights, as shown in \cref{fig:overview}. We derive constructible, recursive guidance rules that implement this calculus using only standard diffusion outputs and posterior probability estimators.

\begin{figure*}[t!]
\caption{
\textbf{Logical Compositional Guidance.} Visualization of logical composition using logical scores $s_t(\varphi, \bx)$ for two specific queries $\varphi$. Our framework replaces constant mixing weights with probability-dependent coefficients derived explicitly from posterior probabilities, allowing for mathematically grounded compositions.
}
\label{fig:overview}
\centering
\includegraphics[width=.9\linewidth]{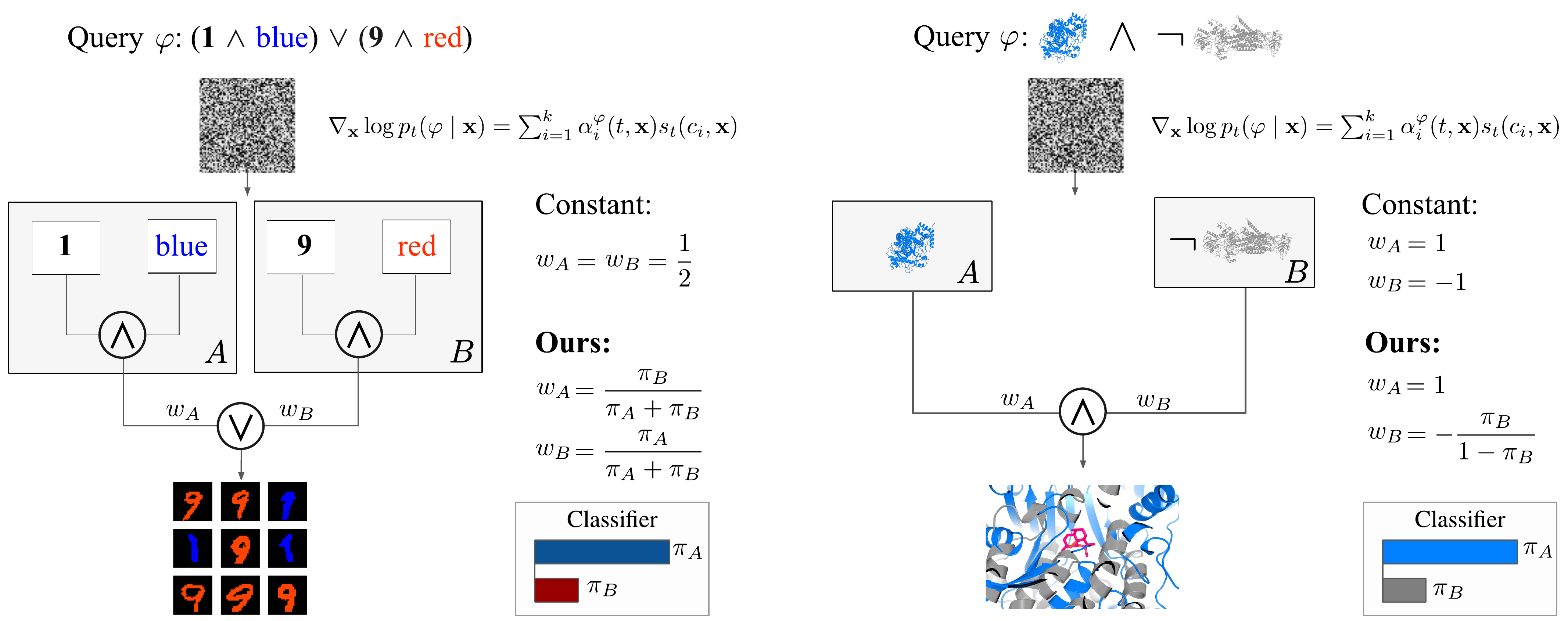}
\end{figure*}

Our main contributions are as follows: 
\begin{itemize} 
\item \textbf{Exact Boolean calculus for composition.} We derive an exact calculus for composing models defined by Boolean formulas over atomic predicates. We provide sufficient conditions on the circuit representation of Boolean formulas for our calculus to apply, and show that for certain commonly encountered classes of distribution, any desired Boolean formula is compilable into such a circuit representation.
\item \textbf{Constructible logical guidance.} 
We derive practical guidance rules that realize Boolean operators using standard (conditional) diffusion outputs and posterior likelihood scalars, extending classifier-free guidance to logical composition.
\item \textbf{Hybrid classifier-assisted guidance.} 
We propose an efficient guidance strategy that combines standard classifier-free guidance with posterior probability estimates to compute the posterior conditioning term.
\end{itemize}

\section{Preliminaries}

To develop our framework, we build on Classifier-free diffusion guidance \citep{ho:2021} and Boolean compositionality \citep{brown2003boolean}.

\subsection{Classifier-Free Guidance of Diffusion Models}

Diffusion models describe the generative process as two Stochastic Differential
Equations (SDEs): the forward process during training and the reverse process at
inference.
We denote by $\bX_t\in\mathcal X\subseteq\RR^d$ the (random) state of the
\emph{reverse-time} diffusion at time $t\in[0,T]$.
Boldface $\bX=(\bX_t)_{t\in[0,T]}$ denotes the entire stochastic process.
Lowercase $\bx_t$ denotes realizations.
We treat $t=0$ as the terminal time, and the sampler integrates the reverse-time
SDE from $t=T$ down to $t=0$.
The reverse-time SDE used for generation is
\begin{equation}\label{eq:revSDE}
    \mathrm d\bX_t
    =
    b_t(\bX_t)\,\mathrm dt
    +
    \sigma_t(\bX_t)\,\mathrm d\bW_t,
    \qquad t\in[0,T],
\end{equation}
where $\bW_t$ is a standard $d$-dimensional Wiener process, $b_t:\mathcal X\to\RR^d$
is the drift and $\sigma_t:\mathcal X\to\RR^{d\times d}$ the diffusion coefficient.
We write $a_t(\bx):=\sigma_t(\bx)\sigma_t(\bx)^\top$ for the diffusion matrix and
$p_t(\bx)$ for the marginal density of $\bX_t$.
The reverse-time drift may be expressed as
\begin{equation}
\label{eq:reverse-drift}
    b_t(\bx)
    =
    f_t(\bx)
    -
    a_t(\bx)\,\nabla_{\bx}\log p_t(\bx),
\end{equation}
where $f_t$ is the drift term of the forward diffusion process.
We refer to $\nabla_{\bx}\log p_t(\bx)$ as the unconditional score.

We recall the principle of \emph{classifier-free guidance} (CFG) that underlies
most conditional diffusion samplers.
Let $p_t(\bx)$ denote the unconditional diffusion marginal at time $t$ and
$p_t(\bx\mid c)$ the conditional marginal corresponding to a condition or
attribute $c$.
By Bayes’ rule,
\begin{align}
\label{eq:bayes-decomp-cfg}
p_t(\bx\mid c)
&\propto
p_t(c\mid \bx)\,p_t(\bx),
\qquad\Rightarrow\qquad \nonumber\\
\nabla_{\bx}\log p_t(\bx\mid c)
&=
\nabla_{\bx}\log p_t(\bx)
+
\nabla_{\bx}\log p_t(c\mid \bx).
\end{align}
Hence, the conditional score decomposes into two terms:
the unconditional score $\nabla_{\bx}\log p_t(\bx)$ and the
\emph{posterior score} $\nabla_{\bx}\log p_t(c\mid \bx)$.

Classifier-free guidance uses a single diffusion model trained both
unconditionally and conditionally, providing estimates of $\nabla_{\bx}\log p_t(\bx)$ and $\nabla_{\bx}\log p_t(\bx\mid c)$. Interpolating between these scores yields the guided score
\begin{equation}\label{eq:cfg-standard}
\nabla_{\bx}\log p_t(\bx) + w\big(\nabla_{\bx}\log p_t(\bx\mid c) - \nabla_{\bx}\log  p_t(\bx) \big),
\end{equation}
where $w\ge 0$ controls the conditioning strength.
For $w=1$ this recovers the exact conditional score
$\nabla_{\bx}\log p_t(\bx\mid c)$, while $w>1$ amplifies the effect of conditioning.

\begin{table*}[t!]
\centering
\caption{Recursive guidance rules for posterior probabilities and scores, which can be implemented as a recursive algorithm (\cref{alg:logical_score}). }
\label{tab:guidance_rules}
\resizebox{0.94\textwidth}{!}{
\renewcommand{\arraystretch}{2.2} 
\begin{tabular}{l @{\hspace{2em}} l @{\hspace{2em}} l}
\toprule
\textbf{Expression} $\varphi$ & \textbf{Posterior Probability} $\widehat\pi(\varphi)$ & \textbf{Score} $s_t(\varphi, \bx)$ \\
\hline
Atom $c \in \mathcal{C}$
    & $p_t(c \mid \bx)$
    & $\nabla_{\bx}\log p_t(\bx\mid c) - \nabla_{\bx} \log p_t(\bx)$ \\
Negation $\lnot\psi$
    & $1-\widehat\pi(\psi)$
    & $-\dfrac{\widehat\pi(\psi)}{1-\widehat\pi(\psi)} \, s_t(\psi,\bx)$ \\
Conjunction $\psi\land\chi$
    & $\widehat\pi(\psi)\widehat\pi(\chi)$
    & $s_t(\psi,\bx) + s_t(\chi,\bx)$ \\
Disjunction (OR-CI) $\psi\lor\chi$
    & $\widehat\pi(\psi)+\widehat\pi(\chi)-\widehat\pi(\psi)\widehat\pi(\chi)$
    & $\dfrac{\widehat\pi(\psi)\widehat\pi(\lnot \chi) s_t(\psi,\bx) + \widehat\pi(\chi)\widehat\pi(\lnot \psi) s_t(\chi,\bx)}{\widehat\pi(\psi)+\widehat\pi(\chi)-\widehat\pi(\psi)\widehat\pi(\chi)}$ \\
Disjunction (OR-ME) $\psi\lor\chi$
    & $\widehat\pi(\psi)+\widehat\pi(\chi)$
    & $\dfrac{\widehat\pi(\psi) s_t(\psi,\bx) + \widehat\pi(\chi) s_t(\chi,\bx)}{\widehat\pi(\psi)+\widehat\pi(\chi)}$ \\
\bottomrule
\end{tabular}}
\end{table*}

\subsection{Boolean Compositionality}

\paragraph{Atoms and formulas.}
Let $\mathcal C=\{c_1,\dots,c_n\}$ be atomic predicates. Each $c_i$ induces an event $A_i\subseteq\mathcal X$ on terminal states (time $0$). 
Boolean formulas $\varphi$ are generated from $\mathcal C$ by $\land,\lor,\lnot$, representing AND, OR and NOT operators respectively, and $\top,\bot$, representing true and false respectively; their semantics are sets $\llbracket\varphi\rrbracket\subseteq\mathcal X$ with $x\models \varphi \iff x\in \llbracket\varphi\rrbracket$.  We note that $\llbracket\top\rrbracket=\mathcal{X}$ and $\llbracket\bot\rrbracket=\emptyset$.

\paragraph{Backward truth probability and logical score.}
For $t\in[0,T]$ and $\bx\in\mathcal X$ define
\begin{eqnarray}\label{eq:backwardTruth}
    p_t(\varphi\mid \bx) &:=& \mathbb P\!\left(\bX_0\in \llbracket\varphi\rrbracket \middle| \bX_t=\bx\right),
    \nonumber \\
    s_t(\varphi,\bx) &:=& \nabla_{\bx}\log p_t(\varphi\mid \bx).
\end{eqnarray}
Let $p_t(\bx\mid \varphi)$ denote the marginal density of $\bX_t$ under the
process conditioned on the terminal event $\{\bX_0\in\llbracket\varphi\rrbracket\}$.
Following the above, conditioning on $\varphi$ corresponds to replacing the
unconditional score $\nabla_{\bx}\log p_t(\bx)$ in \eqref{eq:reverse-drift} by the
conditional score $\nabla_{\bx}\log p_t(\bx\mid \varphi)$, 
\begin{equation}
b_t^\varphi(\bx) = f_t(\bx) - a_t(\bx)\,\nabla_{\bx}\log p_t(\bx\mid \varphi).
\end{equation}
Now, by Bayes' rule,
\begin{equation}
b_t^\varphi(\bx) = f_t(\bx) - a_t(\bx)\big(\nabla_{\bx}\log p_t(\bx)+s_t(\varphi,\bx)\big).
\end{equation}

\section{Logical Guidance Framework}

We explore the problem of composing guidance scores for complex logical expressions. Let $\mathcal C = \{c_1,\dots,c_N\}$ be a finite collection of Boolean predicates, and let $\varphi$ be a Boolean formula over $\mathcal C$, constructed using $\land$, $\lor$, and $\lnot$. Our objective is to understand when the score $s_t(\varphi,\bx)$ can be evaluated \emph{exactly} by composing the scores of the atomic events.
Exact compositionality depends on the interaction between the logical structure of $\varphi$ and the factorization properties of the joint distribution of the atomic events induced by the diffusion posterior at $(t,\bx)$. In general, Boolean formulas do not permit exact composition. 


\subsection{Compositional Calculus}\label{sec:comp_calc}

The atomic posterior logical scores $s_t(c_i,\bx)$ are available from conditional and unconditional diffusion networks.  We will show that an exact composition is achievable whenever the structure of a formula $\varphi$ admits a circuit representation in which
\begin{itemize}
    \item every conjunction node $\psi\land\chi$ combines conditionally independent subformulas (decomposability),
    \item every disjunction node $\psi\lor\chi$ combines subformulas that are either conditionally independent (OR-CI) or mutually exclusive (OR-ME), so that either the independent-disjunction rule or the mutually-exclusive-disjunction rule applies.
\end{itemize}
These conditions mirror the semantics of probabilistic circuits with decomposable product nodes and either decomposable or deterministic sum nodes \cite{darwiche2022tractable,vergari2021compositional}.
Under these structural assumptions, each formula $\varphi$ is associated
with two recursively computed quantities: a posterior
$\widehat\pi(\varphi) := p_t(\varphi\mid \bx)$ and a logical score
$s_t(\varphi, \bx) := \nabla_x\log p_t(\varphi\mid \bx)$.  The recursive rules are provided in \cref{tab:guidance_rules}. 
For each disjunction node $\psi\lor\chi$, we select the OR-CI rules
if the terminal events of
$\psi$ and $\chi$ are conditionally independent given $\bX_t=\bx$, and the
OR-ME rule if they are mutually exclusive.  The following proposition summarizes the above (proof provided in Appendix \ref{annex:proof}):

\begin{proposition}
\label{prop:exact-meci-fragment}
Let $\varphi$ be a propositional formula over atoms $\{c_i\}$.
Suppose that $\varphi$ admits a circuit representation whose internal nodes are $\land$, $\lor$, and $\lnot$,
and whose $\land$- and $\lor$-nodes satisfy, for every $t\in(0,T]$ and every $\bx\in\mathcal X$:
\begin{enumerate}
\item For every conjunction node $\psi\land\chi$,
\[
p_t(\psi\land\chi\mid \bx)=p_t(\psi\mid \bx)\,p_t(\chi\mid \bx).
\]
\item For every disjunction node $\psi\lor\chi$, either
\[
p_t(\psi\land\chi\mid \bx)=p_t(\psi\mid \bx)\,p_t(\chi\mid \bx)
\quad\text{(OR-CI),}
\]
or
\[
p_t(\psi\land\chi\mid \bx)=0
\quad\text{(OR-ME).}
\]
\end{enumerate}
Assume furthermore that $\widehat\pi(\psi)=p_t(\psi\mid \bx)$ for all subformulas $\psi$ of $\varphi$,
for every $t\in(0,T]$ and every $\bx\in\mathcal X$, and that for every subformula $\psi$ appearing in the circuit,
the map $\bx\mapsto p_t(\psi\mid \bx)$ is differentiable and $0<p_t(\psi\mid \bx)<1$ for all $t\in(0,T]$, $\bx\in\mathcal X$.
Then the recursive rules reproduce exactly the true posterior and logical score for $\varphi$ for every $t\in(0,T]$ and $\bx\in\mathcal X$:
\[
\widehat\pi(\varphi)=p_t(\varphi\mid \bx),
\qquad
s_t(\varphi,\bx)=\nabla_{\bx}\log p_t(\varphi\mid \bx).
\]
\end{proposition}

\begin{movableProof}{prop:exact-meci-fragment}

Fix arbitrary $t\in(0,T]$ and $\bx\in\mathcal X$. We argue by induction over a topological ordering of the circuit nodes.

\emph{Base case (atoms).}
For an atomic predicate $c_i$ we have by assumption $\widehat\pi(c_i)=p_t(c_i\mid \bx)$ and
$s_t(c_i,\bx)= \nabla_{\bx}\log p_t(\bx\mid c_i) - \nabla_{\bx} \log p_t(\bx)$.
Thus the claim holds.

\emph{Inductive step (negation).}
Suppose the claim holds for $\psi$.  Then by definition,
$p_t(\lnot\psi\mid \bx)=1-p_t(\psi\mid \bx)$ and
\begin{eqnarray}
\nabla_{\bx}\log p_t(\lnot\psi\mid \bx)
&=& \nabla_{\bx}\log\big(1-p_t(\psi\mid \bx)\big) \nonumber \\
&=& -\frac{p_t(\psi\mid \bx)}{1-p_t(\psi\mid \bx)} \nabla_{\bx} \log p_t(\psi\mid \bx), \nonumber
\end{eqnarray}
which matches the recursive rule for $s_t(\lnot\psi, \bx)$ when
$\widehat\pi(\psi)=p_t(\psi\mid \bx)$.  Likewise
$\widehat\pi(\lnot\psi)=1-\widehat\pi(\psi)=p_t(\lnot\psi\mid \bx)$.

\emph{Inductive step (conjunction).}
Let $\varphi=\psi\land\chi$ and assume the inductive hypothesis for
$\psi$ and $\chi$.  By assumption,
$p_t(\psi\land\chi\mid \bx)=p_t(\psi\mid \bx)p_t(\chi\mid \bx)$, so with
$\widehat\pi(\psi)=p_t(\psi\mid \bx)$ and
$\widehat\pi(\chi)=p_t(\chi\mid \bx)$ we obtain
\[
\widehat\pi(\psi\land\chi)
= \widehat\pi(\psi)\widehat\pi(\chi)
= p_t(\psi\mid \bx)p_t(\chi\mid \bx)
= p_t(\psi\land\chi\mid \bx).
\]
Moreover,
\begin{eqnarray}
\nabla_{\bx}\log p_t(\psi\land\chi\mid \bx)
&=& \nabla_{\bx}\log p_t(\psi\mid \bx)
  + \nonumber \\
  &&\nabla_{\bx}\log p_t(\chi\mid  \bx) \nonumber \\
&=& s_t(\psi, \bx)+s_t(\chi, \bx), \nonumber
\end{eqnarray}
which is exactly the recursive rule for $s_t(\psi\land\chi, \bx)$.

\emph{Inductive step (disjunction).}
Let $\varphi=\psi\lor\chi$ and assume the inductive hypothesis for
$\psi$ and $\chi$.  By inclusion--exclusion,
\[
p_t(\psi\lor\chi\mid \bx)
= p_t(\psi\mid \bx)+p_t(\chi\mid \bx)-p_t(\psi\land\chi\mid \bx).
\]
There are two cases.

If $\psi$ and $\chi$ are conditionally independent, then
$p_t(\psi\land\chi\mid \bx)=p_t(\psi\mid \bx)p_t(\chi\mid \bx)$ and
\begin{eqnarray}
&&\nabla_{\bx}\log p_t(\psi\lor\chi\mid \bx)
= \nonumber \\
&& \quad \frac{
p_t(\psi\mid \bx)\big(1-p_t(\chi\mid \bx)\big) \nabla_{\bx}\log p_t(\psi\mid \bx)
}{
p_t(\psi\mid \bx)+p_t(\chi\mid \bx)-p_t(\psi\mid \bx)p_t(\chi\mid \bx)
} + \nonumber \\
&& \quad \frac{
p_t(\chi\mid \bx)\big(1-p_t(\psi\mid \bx)\big) \nabla_{\bx}\log p_t(\chi\mid \bx)
}{
p_t(\psi\mid \bx)+p_t(\chi\mid \bx)-p_t(\psi\mid \bx)p_t(\chi\mid \bx)
}, \nonumber 
\end{eqnarray}
which coincides with OR-CI rule for the scores when
$\widehat\pi(\psi)=p_t(\psi\mid \bx)$ and
$\widehat\pi(\chi)=p_t(\chi\mid \bx)$.  The posterior recursion
$\widehat\pi(\psi\lor\chi)$ for CI matches
$p_t(\psi\lor\chi\mid \bx)$ by the same identity.

If $\psi$ and $\chi$ are mutually exclusive, then
$p_t(\psi\land\chi\mid \bx)=0$ and
$p_t(\psi\lor\chi\mid \bx)=p_t(\psi\mid \bx)+p_t(\chi\mid \bx)$, and
\begin{eqnarray}
\nabla_{\bx}\log p_t(\psi\lor\chi\mid \bx)
&=&
\frac{
p_t(\psi\mid\bx) \nabla_{\bx}\log p_t(\psi\mid \bx)
}{
p_t(\psi\mid \bx)+p_t(\chi\mid \bx)
}+ \nonumber \\
&&\frac{
p_t(\chi\mid \bx) \nabla_{\bx}\log p_t(\chi\mid \bx)
}{
p_t(\psi\mid \bx)+p_t(\chi\mid \bx)
}, \nonumber
\end{eqnarray}
which coincides with $s_t(\psi\lor\chi, \bx)$ for ME under the same
identification of $\widehat\pi$ with $p_t(\cdot\mid \bx)$.  The ME
posterior recursion
$\widehat\pi(\psi\lor\chi)=\widehat\pi(\psi)+\widehat\pi(\chi)$
matches $p_t(\psi\lor\chi\mid \bx)$.

\smallskip
In all cases, the recursive rules reproduce $p_t(\cdot\mid \bx)$ and
$\nabla_{\bx}\log p_t(\cdot\mid \bx)$ at the node $\varphi$ assuming they do
so for its children.  By structural induction over the circuit of
$\varphi$, the claim holds for all subformulas and in particular for
$\varphi$ itself.
\end{movableProof}

We note that, for any propositional formula $\varphi$ admitting a circuit representation satisfying the sufficient conditions of \cref{prop:exact-meci-fragment}, it follows that the corresponding logical score can be expressed as a linear combination of the atomic logical scores, with coefficients that are functions of the posterior estimates $\widehat\pi(\cdot)$ (and hence of the atomic posteriors).  In particular, there exist scalar functions $\alpha_i^{\varphi}(t,\bx)$ such that 
\begin{equation}
\label{eq:linear_combination_atoms}
\nabla_{\bx}\log p_t(\varphi\mid \bx) = \sum_{i=1}^{k} \alpha_i^{\varphi}(t,\bx) s_t(c_i, \bx),
\end{equation}
where $k$ is the number of atoms appearing in $\varphi$ and each coefficient $\alpha_i^{\varphi}(t,\bx)$ is determined by recursive application of the rules in \cref{tab:guidance_rules}.

When a formula $\varphi$ admits a circuit representation satisfying
\cref{prop:exact-meci-fragment}, exact logical guidance reduces to a single
evaluation of that circuit at each diffusion step.
Each atomic predicate $c_i$ requires exactly one posterior $\pi(c_i)$ and one score $s_t(c_i,\bx)$ evaluation.
As a result, no additional diffusion model evaluations are needed for
composite formulas, and the cost of evaluating $s_t(\varphi,\bx)$ scales
linearly with the circuit's size. This mirrors tractability properties of probabilistic circuits with decomposable and deterministic structure.

In several situations, guidance properties can be expressed through a set of conditionally independent categorical variables. Examples include combinations of discrete attributes, such as color, shape, or residue type at multiple positions within a protein. Users often wish to specify constraints directly in terms of which joint configurations of the categorical variables are allowed. For instance, a user may allow only a small set of valid attribute combinations, or forbid specific combinations while leaving all others unconstrained. The following proposition shows that such queries ({\em categorical-CI queries}) can be evaluated exactly under our logical guidance framework, provided the categorical variables are conditionally independent (proof provided in Appendix \ref{annex:proof}).

\begin{proposition}
\label{prop:categorical_simple}
Fix $t\in(0,T]$ and $\bx\in\mathcal X$.
Let $Z_1,\dots,Z_M$ be conditionally independent categorical variables given $\bX_t=\bx$,
with finite domains $\Omega_1, \dots,\Omega_M$.
For each $m \in [M]$ and $v\in\Omega_m$, define the Boolean predicate $c_{m,v}:=(Z_m=v)$.
Let $\mathcal A\subseteq \Omega_1\times\cdots\times\Omega_M$ be a set of joint assignments and define
\[ \varphi_{\mathcal A}
:=\bigvee_{(v_1,\dots,v_M)\in\mathcal A}\;
\bigwedge_{m=1}^M c_{m,v_m}. \]
Then the circuit of $\varphi_{\mathcal A}$ satisfies the sufficient conditions of
\cref{prop:exact-meci-fragment}:
every $\land$-node is conditionally independent and every $\lor$-node is mutually exclusive.
Consequently, the recursive rules compute $p_t(\varphi_{\mathcal A}\mid\bx)$ and
$s_t(\varphi_{\mathcal A},\bx)$ exactly.
\end{proposition}

\begin{movableProof}{prop:categorical_simple}
Fix $(t,\bx)$.
For any assignment $\alpha=(v_1,\dots,v_M)\in\mathcal A$, define
\[
Y_\alpha := \bigwedge_{m=1}^M c_{m,v_m}.
\]
By conditional independence of $Z_1,\dots,Z_M$,
\begin{flalign}
p_t(Y_\alpha\mid\bx) = & p_t\left(\bigwedge_{m=1}^M (Z_m=v_m)\,\middle|\,\bx\right) \\
= & \prod_{m=1}^M p_t(Z_m=v_m\mid\bx), 
\end{flalign}
so each $\land$-node is CI.

If $\alpha\neq\beta$, then there exists $m$ such that $v_m\neq w_m$. Since $Z_m$ is categorical, the events $(Z_m=v_m)$ and $(Z_m=w_m)$ are mutually exclusive, hence
$Y_\alpha\wedge Y_\beta=\bot$ and $p_t(Y_\alpha\wedge Y_\beta\mid\bx)=0$. Thus, the $\lor$-node is OR-ME. 

Therefore, the circuit satisfies the conditions of
Prop.~\ref{prop:exact-meci-fragment}, and the claim follows.
\end{movableProof}

In addition to such categorical-CI queries, we show in Appendix \ref{annex:theoretical_results} that our framework handles exactly a large class of {\em taxonomy queries}, which relate to properties expressible in a tree-structured taxonomy (\cref{prop:taxonomy_queries}).  Further, we show in Appendix \ref{annex:theoretical_results} that such large classes of queries may be used to demonstrate that our logical guidance framework is complete for all Boolean queries in distributions over predicates with the properties stated.  Particularly, we introduce the notion of {\em compilability} of a Boolean formula $\varphi$, where $\varphi$ is {\em compilable} iff we may find a probabilistic circuit evaluable using our logical guidance rules corresponding to a formula $\varphi'$, possibly different from $\varphi$, where the events specified by $\varphi$ and $\varphi'$ are identical in the class of distributions specified.  We show in \cref{prop:completeness} that our logical guidance framework is complete in this sense (any Boolean formula is compilable) for distributions consisting of (1) a collection of conditionally independent categorical variables and (2) properties expressible in a taxonomy.  We note that (1) includes as a special case the case where the properties are a set of events, all of which are CI.  Further, we show that for (2), such completeness of compilability is guaranteed at all time points if it holds for the terminal distribution. Finally, we discuss the case of discrete-time updates in Appendix \ref{annex:discrete}, and show that an analogue of \cref{prop:exact-meci-fragment} can be derived (\cref{prop:exact-discrete-fragment}).

\subsection{Hybrid Conditional Guidance}

The exact rules in \Cref{tab:guidance_rules} assume access to the posterior probabilities $p_t(c_i \mid \bx)$ for all atomic predicates $c_i$, which in practice are not directly available. We therefore introduce estimators $\widehat\pi(c_i)\approx p_t(c_i\mid \bx)$ for the atomic predicates, and use them to compute the coefficients $\widehat\pi$ for the composition rules. For example, $\widehat\pi(c_i)$ can be obtained by training a noise-aware classifier on $\bx_t$  across diffusion times $t$.

\textbf{Gradient-free guidance.} This provides a hybrid guidance mechanism that separates direction from weighting. The scores are obtained from standard conditional and unconditional diffusion networks 
(\cref{eq:bayes-decomp-cfg})
, while the mixture weights required to compose complex formulas are provided by $\widehat\pi(\cdot)$ (\cref{eq:linear_combination_atoms}).  Importantly, this does not require backpropagating through a classifier to obtain gradients: probability estimates enter only as scalar coefficients in closed-form composition rules, while all score gradients are supplied by the diffusion model itself.  Hence, our approach bridges classifier-guidance and classifier-free guidance: it preserves the stability
and the efficiency of classifier-free score estimation while enabling
probabilistically-correct composition for disjunctions and negations.

Moreover, when $\widehat\pi(c_i)$ is estimated directly from the diffusion
model, e.g., with score-based probability estimators
\citep{li2023your, skreta:2025}, 
the resulting method requires only a trained conditional diffusion model with the atomic scores available under classifier-free guidance. 

\textbf{Adaptive repulsive guidance.} Beyond compositional queries, \method{} can also improve standard single-condition generation. When conditioning on a single class $A$, a common failure mode of diffusion models is confusion with other similar classes. Our framework naturally supports adaptive \emph{repulsive guiding} by conditioning on formulas of the form $A\land \lnot B$, where $B$ denotes a competing class. The resulting guidance direction is a probability-weighted combination of the atomic scores for $A$ and
$B$, with weights determined by $\widehat\pi(A\mid \bx)$ and
$\widehat\pi(B\mid \bx)$, unlike heuristic approaches that rely on constant weights \cite{shenoy:2024}. As a result, repulsive guidance is applied strongly only in regions where $B$ is locally probable, and vanishes when $B$ is already unlikely.

\begin{table}[t!]
\centering
\caption{\textbf{Conformity score (\%, $\uparrow$) on CMNIST and Shapes3D.} Evaluation of atomic, binary, and complex queries, where $N=2\dots5$ denotes random formulas with $N$ AND/OR operators.}
\label{tab:cs_cmnist_shapes3d}
\resizebox{\linewidth}{!}{
\begin{tabular}{lrrrrrrrr}
\toprule
Method & AND & NOT & OR-ME & OR-CI & $N=2$ & $N=3$ & $N=4$ & $N=5$ \\
\midrule
\multicolumn{9}{l}{\textbf{CMNIST}} \\ 
Uncond. & 0.9 & 90.0 & 19.2 & 19.1 & 28.3 & 35.1 & 38.8 & 48.9 \\
Constant & 80.4 & \textbf{99.5} & 77.0 & 76.7 & 76.1 & 66.7 & 68.3 & 75.2 \\
\method{} & 80.4 & 96.5 & \textbf{98.0} & \textbf{97.2} & \textbf{93.8} & \textbf{93.3} & \textbf{94.2} & \textbf{94.4} \\
\midrule
\multicolumn{9}{l}{\textbf{Shapes3D}} \\ 
Uncond. & 1.3 & 87.5 & 21.9 & 23.4 & 21.5 & 25.6 & 26.7 & 33.2 \\
Constant & 82.4 & \textbf{98.7} & 63.3 & 74.4 & 67.2 & 59.4 & 58.4 & 57.9 \\
\method{} & 82.4 & 96.4 & \textbf{94.3} & \textbf{96.0} & \textbf{88.8} & \textbf{88.6} & \textbf{85.1} & \textbf{87.6} \\

\bottomrule
\end{tabular}
}
\end{table}

\section{Experiments}

We evaluate \method{} on image and molecular generation tasks 
to assess both logical controllability and generation quality. First, we use synthetic and real-world image tasks to measure logical conformity and robustness to compositional depth, and the effect of guidance on conformity and diversity. We then analyze repulsive guiding as an inference-time mechanism for improving sample quality. Finally, we apply \method{} to structure-based drug design, demonstrating that the proposed guidance rules generalize to protein-conditioned molecular generation and multi-target objectives.

\subsection{Image generation} \label{sec:img_experiments}
To evaluate our method, we employ three distinct experimental settings: 1) synthetic datasets (CMNIST, Shapes3D) to establish the model's ability to satisfy complex logical queries; 2) CelebA to ensure that our guidance mechanism preserves image quality in real-world domains; and 3) ImageNet and the synthetic datasets to analyze the impact of repulsive guiding on sample quality.

\textbf{Datasets and metrics.} To evaluate logical guidance, we utilize Colored MNIST (CMNIST) \cite{lecun:1998, gaudi:2025} and Shapes3D \cite{kim:2018}. We quantify performance using the Conformity Score (CS) \cite{gaudi:2025}, defined as the percentage of generated samples that satisfy the target logical formula according to a pre-trained classifier. To ensure that high conformity does not come at the cost of reduced diversity, we report the Mean Batch Joint Entropy $H$. For real-world attributes on CelebA \cite{liu:2015}, we additionally report Fréchet Inception Distance (FID) to ensure visual fidelity is maintained. For more detail see Appendix \ref{app:metrics}.

\begin{figure}[t]
    \centering
    \setkeys{Gin}{width=\linewidth}
    \hspace*{0.02\linewidth}%
    \begin{minipage}{0.31\linewidth}
       \scriptsize  \centering \textbf{4 $\lor$ 1}
    \end{minipage}%
    \hfill
    \begin{minipage}{0.31\linewidth}
        \scriptsize \centering \textbf{1 $\lor$ pink}
    \end{minipage}%
    \hfill
    \begin{minipage}{0.31\linewidth}
    \centering 
    \scriptsize 
    \textbf{(orange wall $\lor$ sphere) $\land$} \\
    \textbf{((red $\lor$ yellow) object)}
\end{minipage}

    \vspace{0.5em}
    \begin{minipage}[c]{0.03\linewidth}
        \centering
        \scriptsize 
        \rotatebox{90}{\textbf{Constant}}
    \end{minipage}%
    \hfill
    \begin{minipage}[c]{0.95\linewidth}
        \begin{subfigure}{0.32\linewidth}
            \includegraphics{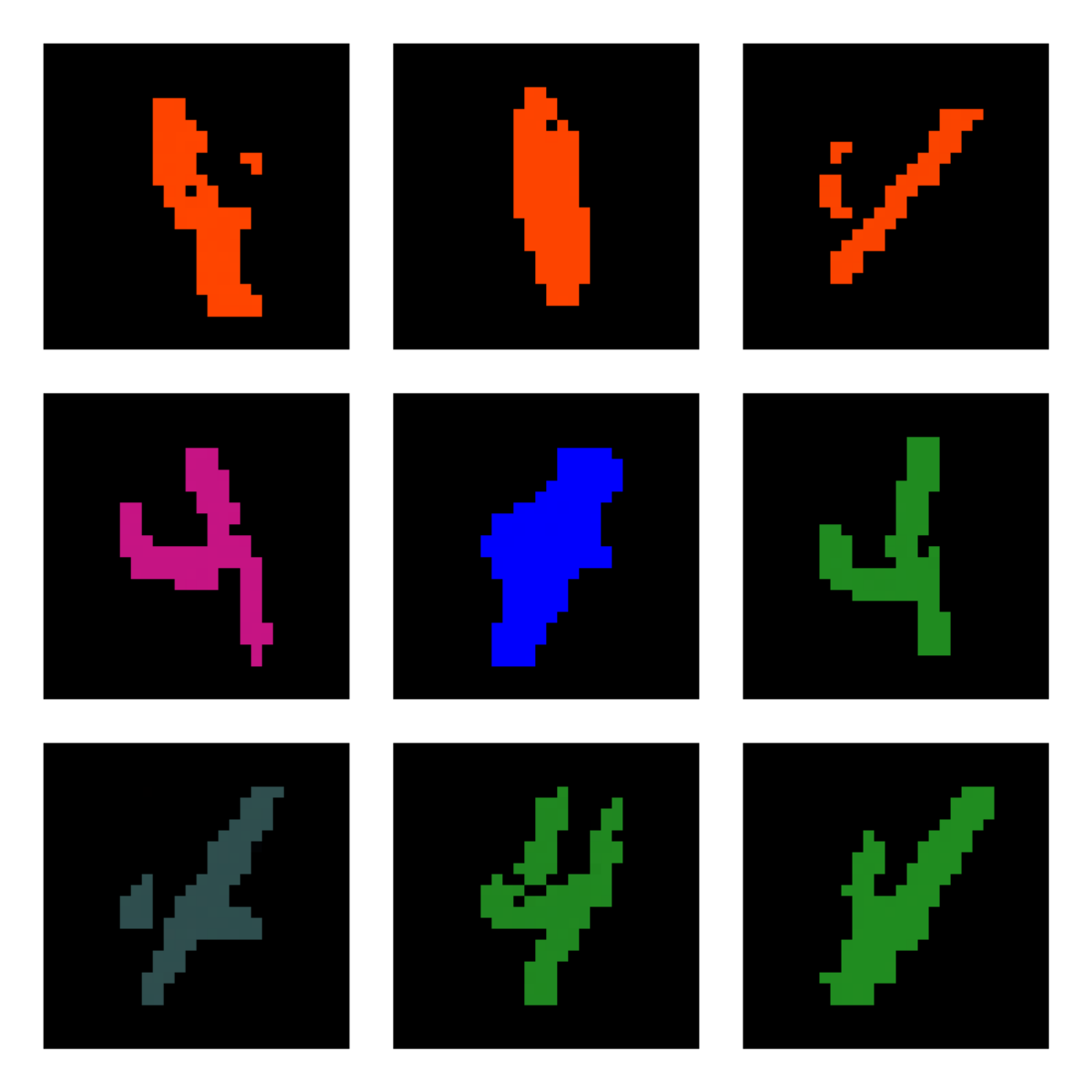}
        \end{subfigure}
        \hfill
        \begin{subfigure}{0.32\linewidth}
            \includegraphics{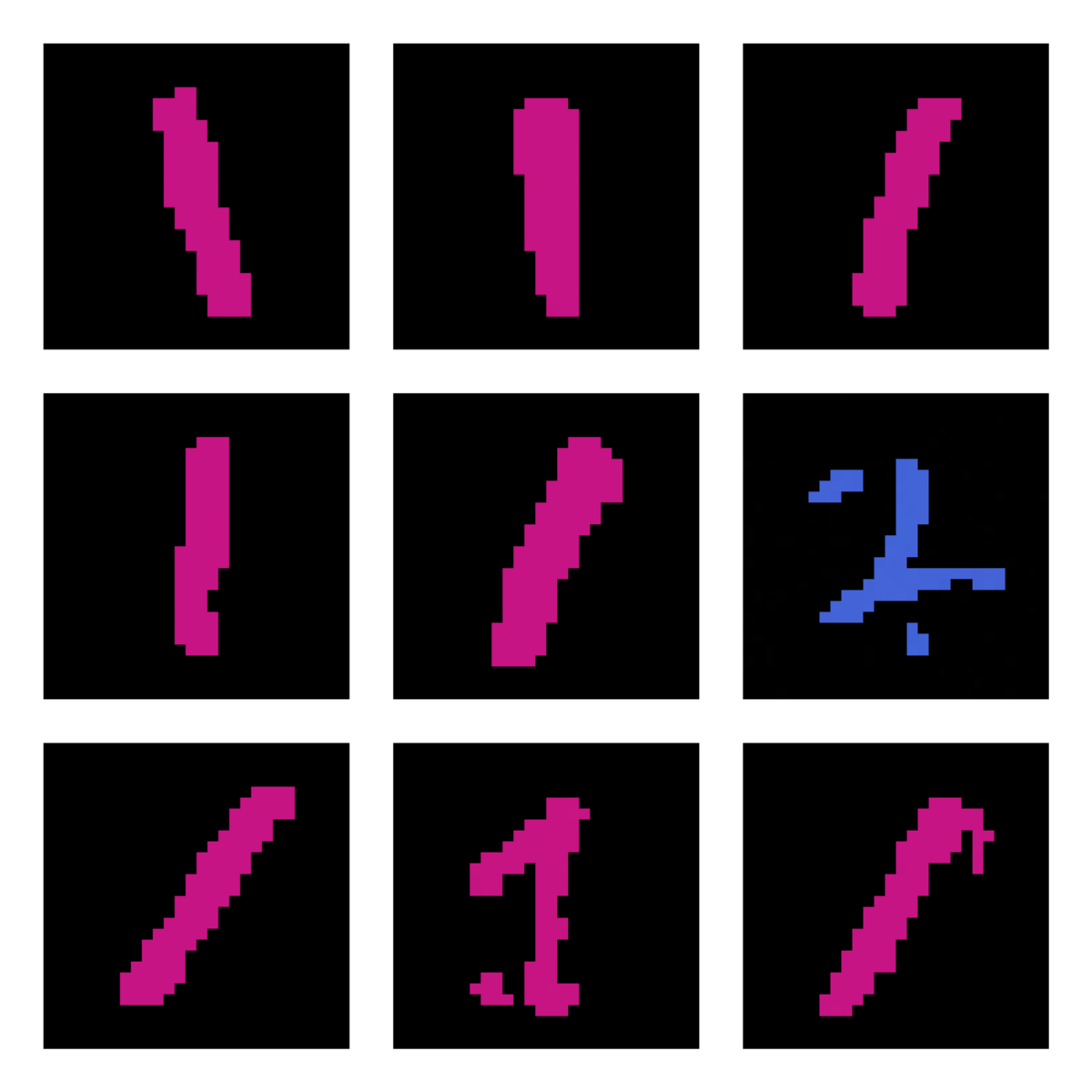}
        \end{subfigure}
        \hfill
        \begin{subfigure}{0.32\linewidth}
            \includegraphics{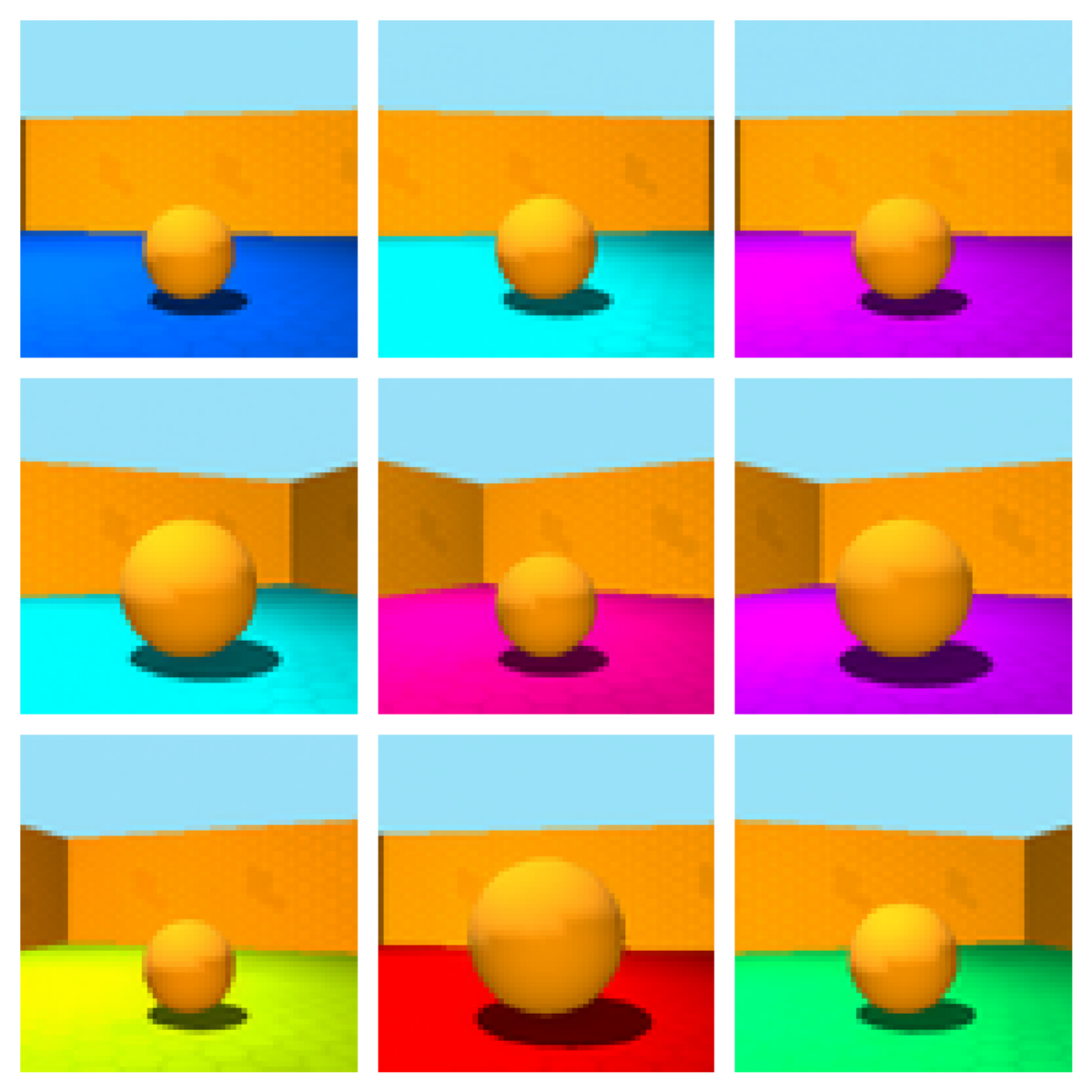}
        \end{subfigure}
    \end{minipage}

    \vspace{0.5em} 

    \begin{minipage}[c]{0.03\linewidth}
        \centering
        \scriptsize 
        \rotatebox{90}{\textbf{\method{}}}
    \end{minipage}
    \hfill
    \begin{minipage}[c]{0.95\linewidth}
        \begin{subfigure}{0.32\linewidth}
            \includegraphics{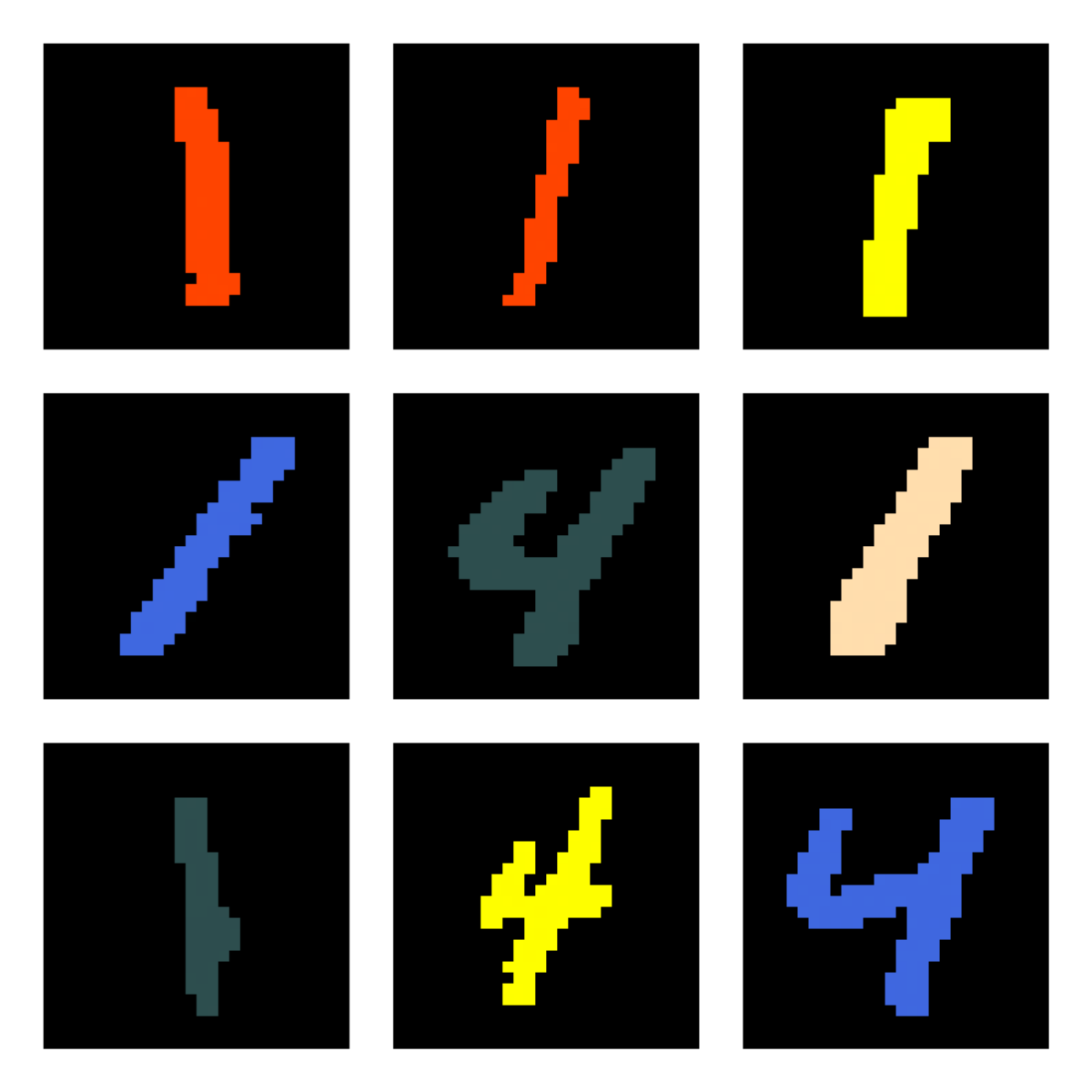}
        \end{subfigure}%
        \hfill
        \begin{subfigure}{0.32\linewidth}
            \includegraphics{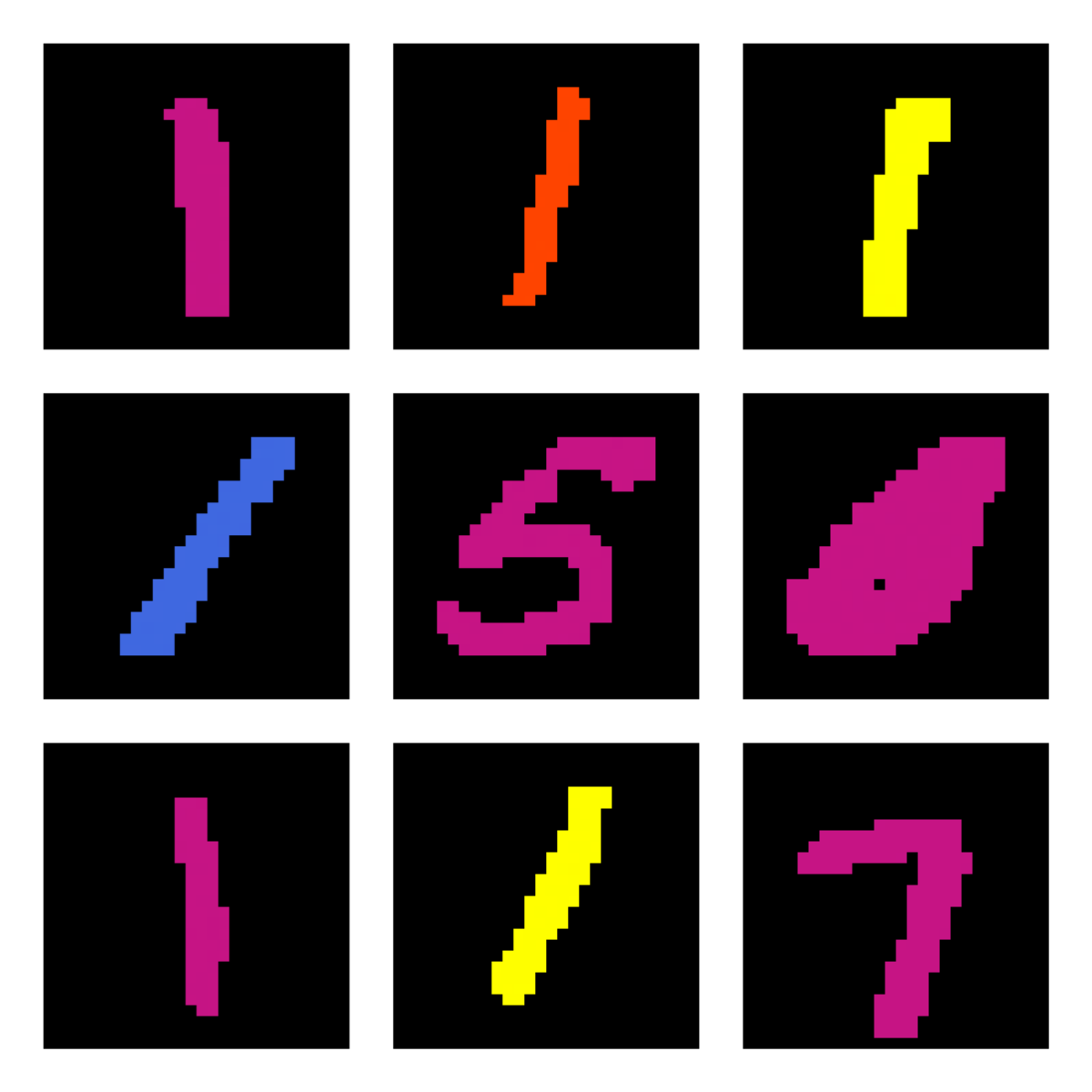}
        \end{subfigure}%
        \hfill
        \begin{subfigure}{0.32\linewidth}
            \includegraphics{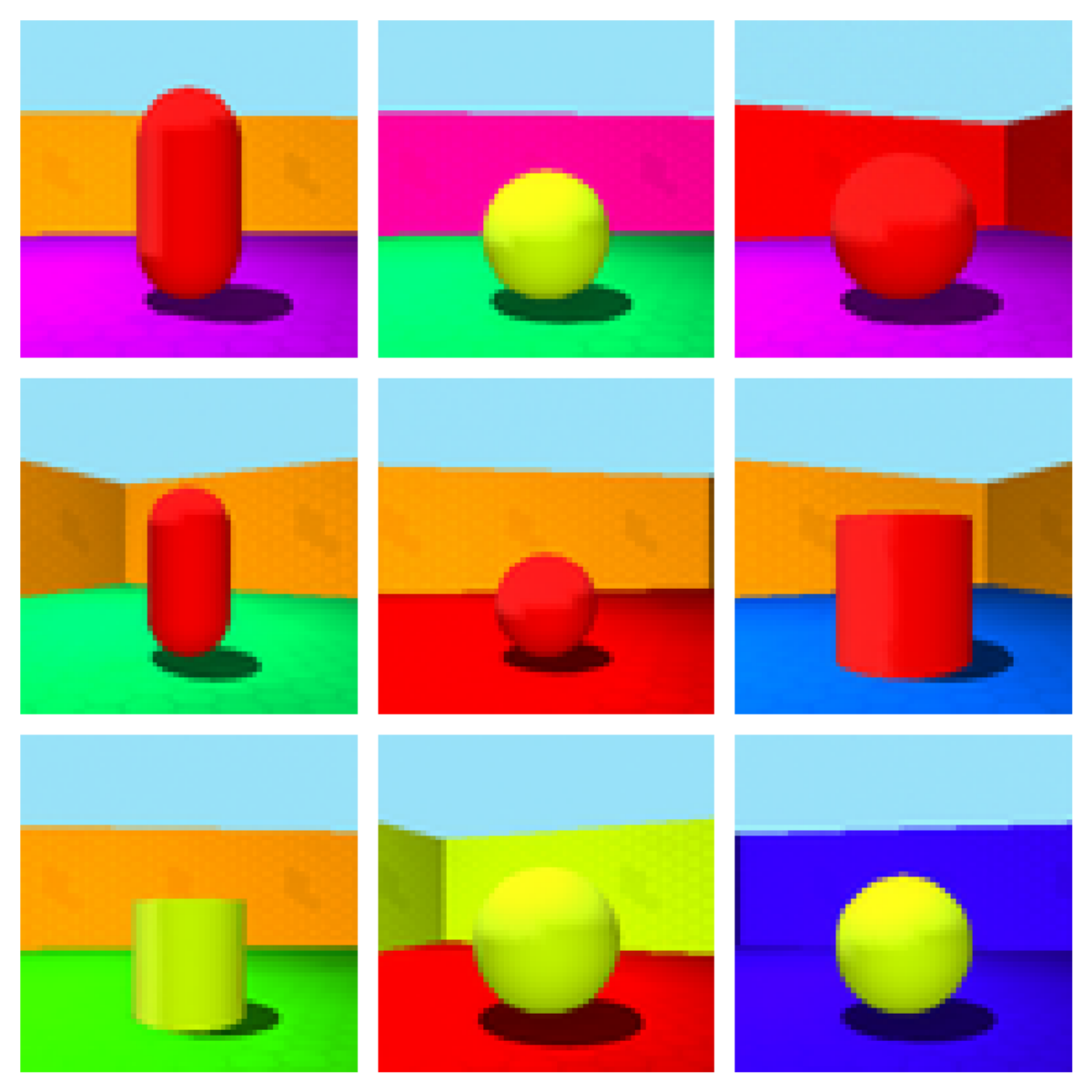}
        \end{subfigure}
    \end{minipage}

    \caption{\textbf{Failure cases of constant baseline.} Constant baseline (top) and \method{} (bottom). The constant baseline struggles with disjunctions, mixing attributes (left), collapsing to an intersection (AND behavior) (middle), or failing for complex queries (right). These failures worsen with higher guidance scales ($w=2.5$).}
    \label{fig:qualitative_comparison_shapes3d_cmnist}
\end{figure}

\textbf{Baselines.} We compare our adaptive method against a static composition baseline. 
This baseline follows standard composition rules \cite{liu:2022, du:2023} and uses fixed scalar weights. In contrast, our proposed framework uses dynamic, time-dependent probability estimates from \Cref{tab:guidance_rules}. The baseline guidance rules are defined as: 
\begin{itemize}[itemsep=1pt]
    \item \textbf{Atom ($c_i$):} $s_t(c_i,\bx)=s_i(t,\bx) - s_\emptyset(t,\bx)$.
    \item \textbf{Negation ($\lnot \psi$):} $-s_t(\psi,\bx)$. 
    \item \textbf{Conjunction ($\psi \land \chi$):} $s_t(\psi,\bx)+s_t(\chi,\bx)$. 
    \item \textbf{Disjunction ($\psi \lor \chi$):} $\frac{1}{2} s_t(\psi,\bx)+ \frac{1}{2} s_t(\chi,\bx)$. 
\end{itemize}
For disjunctions, the baseline uses a constant 0.5 mixture ratio for all cases, as the true normalizing constants are unknown. For completeness, an extended comparison against other compositional baselines and a discussion regarding their limitations can be found in Appendix \ref{sec:appendix_baselines}.

\textbf{Recursive task complexity.} We evaluate the methods through logical queries of increasing complexity, quantified by the count $N$ of AND/OR operators. The evaluation ranges from single-operator baselines to nested formulas with up to $N=5$ logical operators. Queries are generated randomly while enforcing logical validity (e.\,g., avoiding impossible intersections like $3\land 4$, see Appendix \ref{app:generation}).

\begin{figure}[t!]
  \begin{center}
    \centerline{\includegraphics[width=\linewidth]{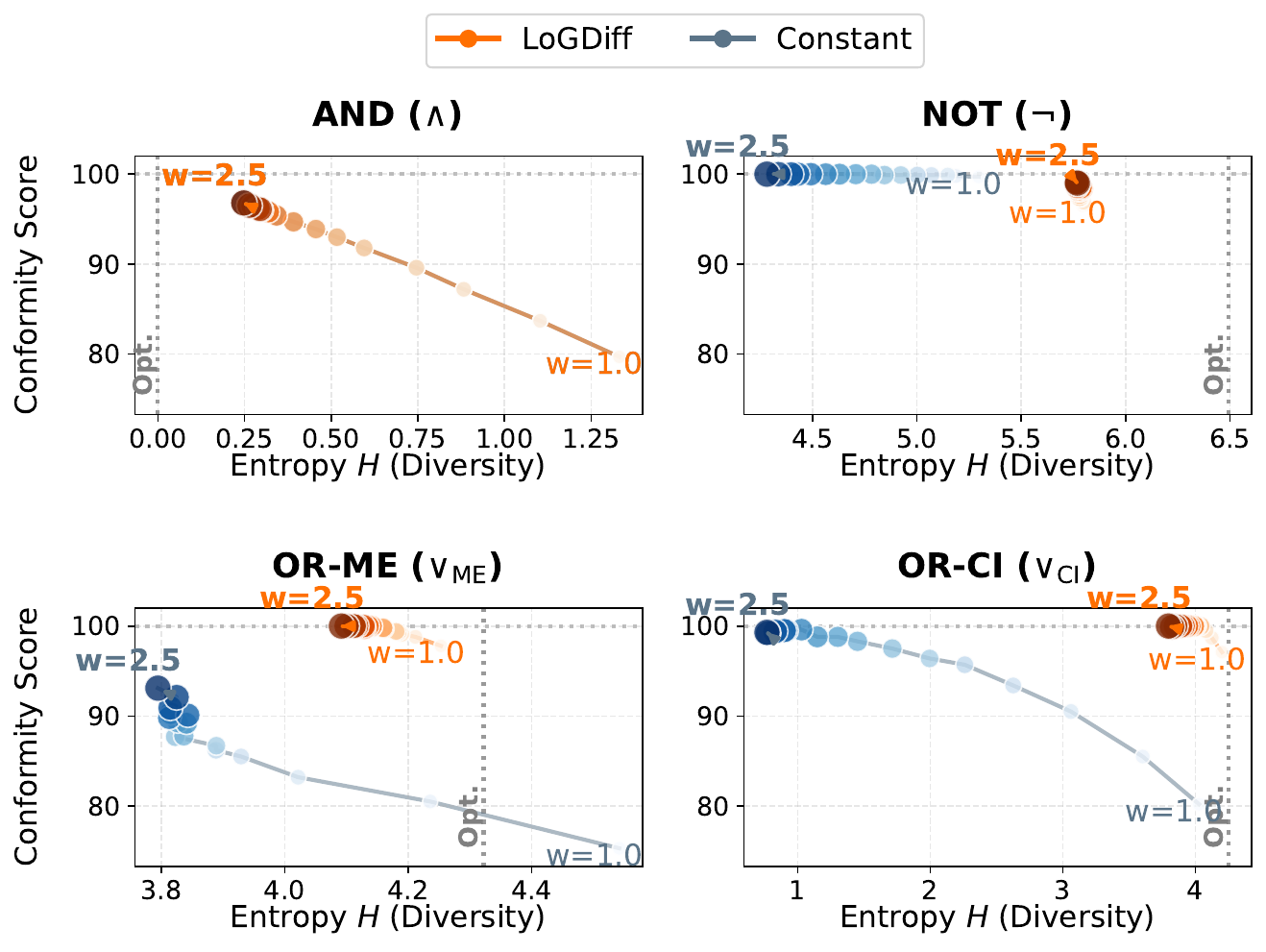}}
    \caption{
      \textbf{Conformity-diversity trade-off on CMNIST}. Conformity Score $\uparrow$ vs. Joint Shannon Entropy across varying guidance scales $w \in [1.0, 2.5]$. Vertical dotted lines indicate the theoretical optimal entropy for each task (note that for AND, the optimal entropy is low as the solution space is highly constrained). While the constant baseline (blue) suffers from low entropy, indicating mode collapse, as guidance strength increases, our method (orange) successfully maintains high sample diversity while achieving high conformity scores.
    }
    \label{fig:gw_acc_entropy_tradeoff}
  \end{center}
\end{figure}

\textbf{Results on synthetic datasets.} As detailed in Table \ref{tab:cs_cmnist_shapes3d}, while results are comparable for intersection and negation, the static baseline yields considerably lower CS on disjunctive and recursive queries.  The constant baseline has a conformity gap of $>20\%$. Furthermore, our method demonstrates better robustness to complexity, maintaining high scores on recursive queries ($N=2\ldots5$) where the baseline consistently yields lower CS. Unconditional generation results are provided as a lower-bound reference, illustrating the baseline probability of satisfying constraints by chance.

\textbf{Conformity-diversity trade-off.} While conformity generally improves with higher guidance scales, \Cref{fig:gw_acc_entropy_tradeoff} indicates a trade-off for the static baseline: increasing guidance results in reduced joint entropy. As shown in \Cref{fig:qualitative_comparison_shapes3d_cmnist}, this quantitative reduction corresponds to the model collapsing to a conjunction or averaging attributes (e.\,g., mixing red and yellow into orange) instead of modeling the full disjunction. In contrast, our method is less affected by this trade-off, preserving diversity even at high guidance levels. Qualitative results can be found in Appendix \ref{app:additional_results}.

\textbf{Results on real-world images.} On CelebA, we restrict evaluation to binary attributes (Blond/Non-blond and Male/Female). In \Cref{tab:celeba}, we report FID alongside CS to confirm that imposing logical constraints does not degrade visual quality and to demonstrate that logical compositional guidance is applicable to real-world datasets. Notably, our method achieves a substantially lower FID for the negation operation compared to the constant baseline, which often suffers from quality degradation.

\begin{table}[h]
\caption{Conformity Score (CS) and FID on CelebA. }
\label{tab:celeba}
\begin{center}
\resizebox{0.9\linewidth}{!}{
\begin{tabular}{lrrrrrr}
\toprule
 & \multicolumn{2}{c}{\textbf{AND}} & \multicolumn{2}{c}{\textbf{NOT}} & \multicolumn{2}{c}{\textbf{OR-CI}} \\
\cmidrule(lr){2-3} \cmidrule(lr){4-5} \cmidrule(lr){6-7} 
\textbf{Method} & CS $\uparrow$ & FID $\downarrow$ & CS $\uparrow$ & FID $\downarrow$ & CS $\uparrow$ & FID $\downarrow$ \\
\midrule
Uncond. & 0.21 & 21.60 & 0.52 & 21.60 & 0.75 & 21.60 \\
\hline
Constant & 0.63 & 19.02 & 0.75 & 32.87 & 0.93 & \textbf{17.11} \\
\method{} & 0.63 & 19.02 & \textbf{0.80} & \textbf{23.61} &\textbf{ 0.97} & 18.47 \\
\bottomrule
\end{tabular}}
\end{center}
\end{table}

\textbf{Enhancing generation quality via repulsive guiding.} We investigate whether \textit{repulsive guiding} can enhance the generation quality of single-class conditional generation for ImageNet (with optimal FID settings \citep{karras:2024}) and compositional guidance for synthetic datasets. Concretely, we replace every atomic condition $A$ with a logical query of the form $A\land \lnot B$, where $A$ is the desired class and $B$ is a competing class. Rather than fixing $B$ globally, we select it adaptively at each diffusion step based on the current noisy sample $\bx_t$, choosing the most probable non-$A$ class under a noise-aware classifier $\widehat\pi$. We then apply our compositional construction to compute the corresponding logical score using the atomic diffusion scores together with the estimates $\widehat\pi(B\mid \bx_t)$. This results in an attractive term toward $A$ and a repulsive term away from the currently most plausible $B$, but only in regions where $B$ is probable. In this way, our method provides a principled guidance mechanism that empirically improves FID by suppressing class confusions while preserving sample diversity. As shown in Figure \ref{fig:neg_guidance_fid} and Table \ref{tab:neg_guiding}, our method provides a principled, state-dependent guidance mechanism that empirically improves FID and CS by suppressing class confusions while preserving sample diversity.

\begin{table}[tbh]
\centering
\caption{Effect of repulsive guidance (RG) on conformity score (\%, $\uparrow$) on CMNIST and Shapes3D.
}
\label{tab:neg_guiding}
\resizebox{\linewidth}{!}{
\begin{tabular}{lrrrrrrrr}
\toprule
Method & AND & NOT & OR-ME & OR-CI & $N=2$ & $N=3$ & $N=4$ & $N=5$ \\
\midrule
\multicolumn{9}{l}{\textbf{CMNIST}} \\ 
\method{} & 80.4 & 96.5 & \textbf{98.0} & 97.2 & 93.8 & 93.3 & 94.2 & 94.4 \\
\method{} + RG & \textbf{83.6} & \textbf{98.4} & 97.9 & \textbf{98.0} & \textbf{94.7} & \textbf{94.3} & \textbf{95.1} & \textbf{95.5} \\
\midrule
\multicolumn{9}{l}{\textbf{Shapes3D}} \\ 
\method{} & 82.4 & 96.4 & 94.3 & 96.0 & 88.8 & 88.6 & 85.1 & 87.6 \\
\method{} + RG & \textbf{86.4} & \textbf{99.3} & \textbf{95.6} & \textbf{96.7} & \textbf{91.8} & \textbf{90.5} & \textbf{89.8} & \textbf{89.9} \\
\bottomrule
\end{tabular}
}
\end{table}

\begin{figure}
    \centering
    \includegraphics[width=1\columnwidth]{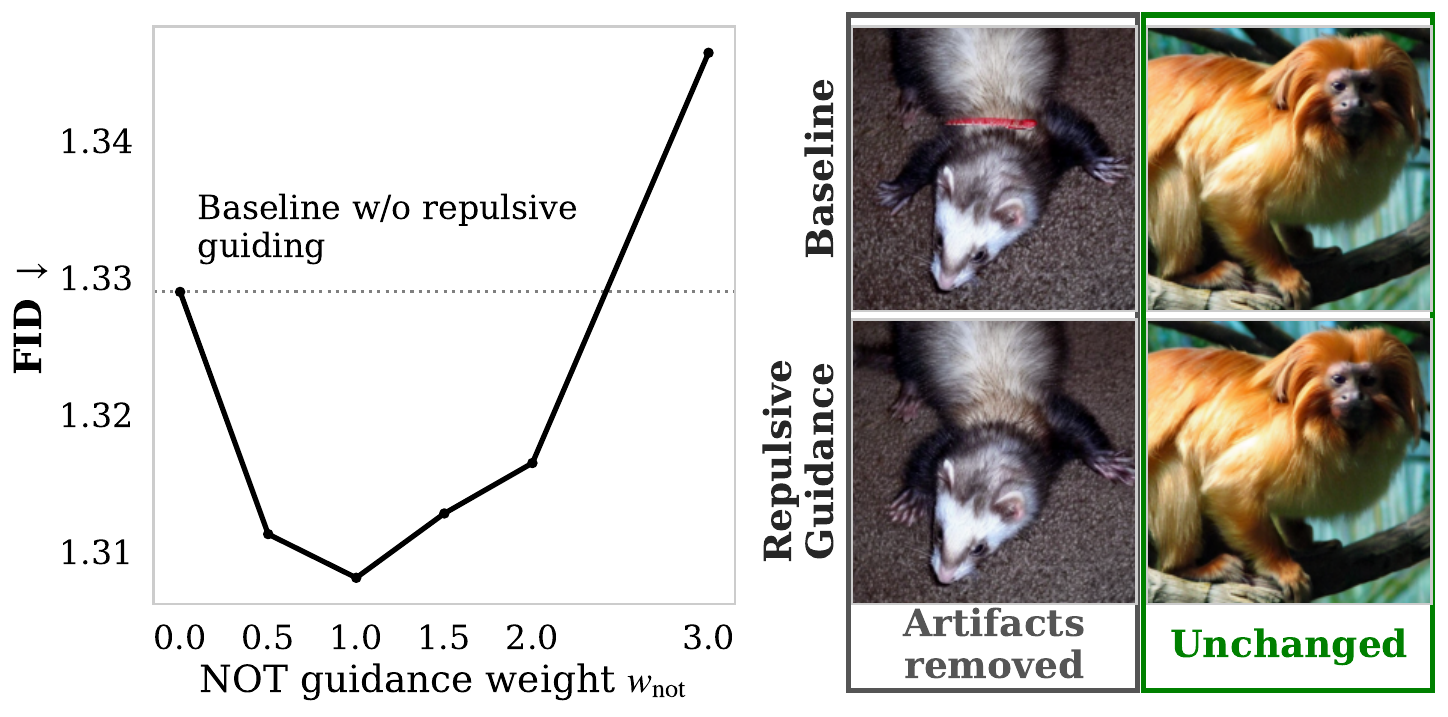}
    \caption{\textbf{Impact of adaptive repulsive guidance.} (Left) Moderate guidance weights improve FID scores compared to no repulsive guiding ($w_\text{not} = 0.0$). (Right) Repulsive guiding removes artifacts while clearly defined samples remain unchanged.}
    \label{fig:neg_guidance_fid}
\end{figure}

\subsection{Molecular Generation}

We explore \method{} for multi-target structure-based drug design, where the objective is to design ligands (i.e., small drug-like molecules) given a target protein in its three-dimensional structure.
We experiment with the dual-target drug design, in which the ligand simultaneously binds two proteins.
Dual-target drug design is of interest for combating various cancers and neurodegeneration~\citep{ramsay2018perspective}, or in reducing drug resistance ~\citep{yang2024rethinking}. We study the performance of compositional guidance, where the two atomic conditions are two target proteins, using the experimental setup of \citep{skreta2025feynman}.
We evaluate our framework on the GRM5-RRM1 protein pair (UniProt: P41594, P23921), where GRM5 is a metabotropic glutamate receptor implicated in neurological disorders, and RRM1 is a ribonucleotide reductase subunit that plays a central role in DNA synthesis and is a validated oncology target.
The generated ligand performance is assessed by the docking score to each target protein using AutoDock Vina~\citep{eberhardt2021autodock}, where $32$ ligands of size $23$ are generated over $8$ experiments.

\cref{tab:FKC_AND,tab:FKC_NOT} list the results of \method{}, TargetDiff~\citep{guan20233d}, and  DualDiff~\citep{yang2024rethinking}, where we perform guidance with fixed mixing weights. 
\cref{tab:FKC_AND_correction,tab:FKC_NOT_correction} show the impact of Feynman-Kac Correction (FKC) \citep{skreta2025feynman}. 
\cref{tab:FKC_AND}, similar to ~\citep{yang2024rethinking,skreta2025feynman}, evaluates the performance for dual targets (AND and OR), while \cref{tab:FKC_NOT} lists the results for one on-target and one off-target use case (AND-NOT and XOR). We report the average docking score for each target. A lower docking score indicates better binding. 

The impact of logical guidance on ligand geometry is illustrated in Figure \ref{fig:ligand_structure}, where $A \land B$ ligands occupy the shared binding pocket with complementary interactions to both targets, while $A \land \lnot B$ ligands adopt distinct poses that favor GRM5 binding while minimizing RRM1 contacts. A detailed $2$D interaction analysis illustrating the distinct binding profiles of ligands generated under different logical constraints is provided in Supplementary  \cref{fig:ligand_interactions}.

\begin{figure}[t!]
    \centering
    \includegraphics[width=0.89\columnwidth]{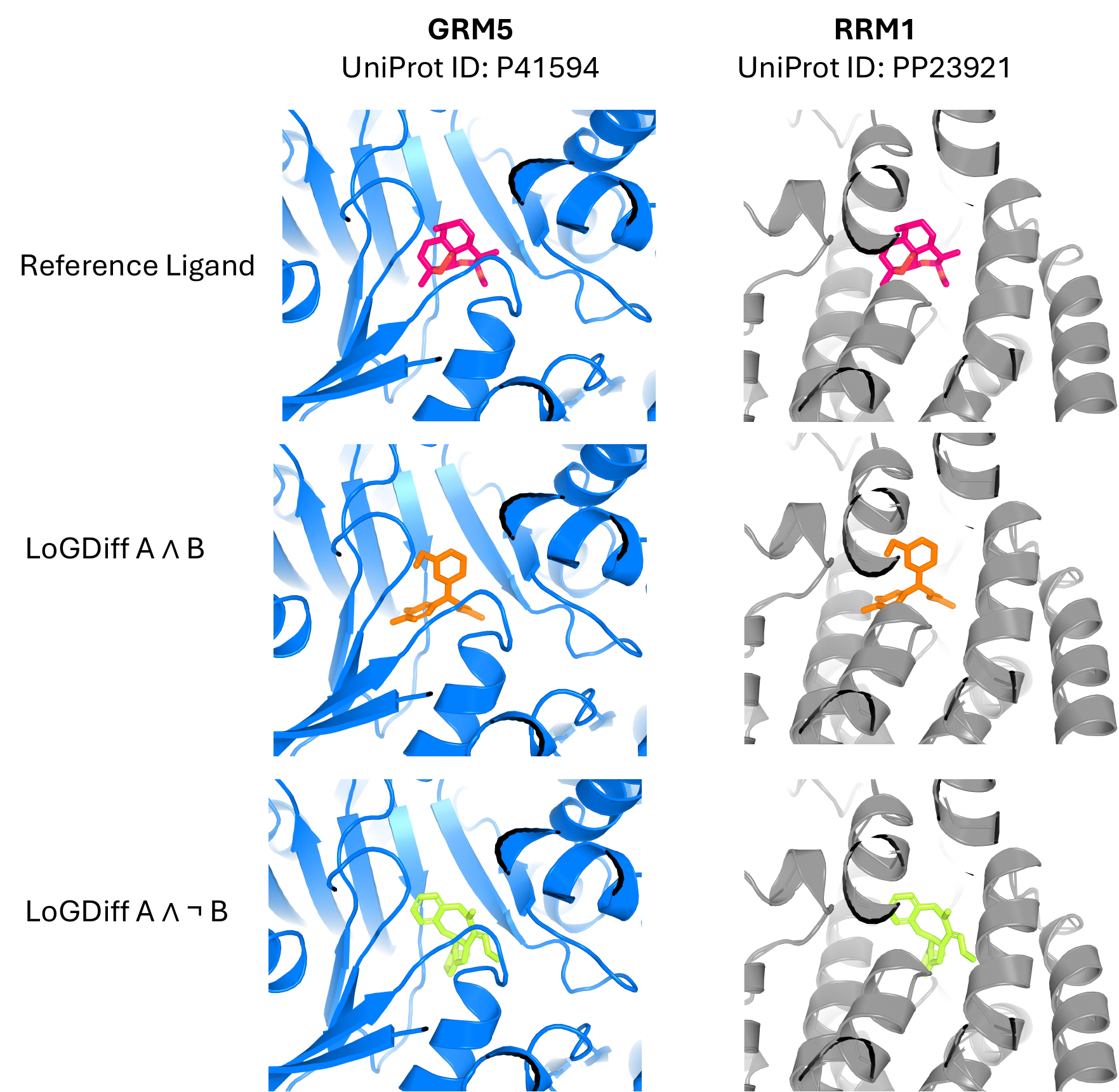}
    \caption{\label{fig:ligand_structure}\textbf{Visualization of ligands in the GRM5-RRM1 dual-target binding site.} (Top) Reference ligand in the aligned binding pocket. (Middle) Representative ligand generated under the guidance term ($A \land B$). (Bottom) Representative ligand generated under the selective constraint ($A \land \lnot B$), designed to engage GRM5 while avoiding RRM1 binding. Protein surfaces are shown for GRM5 (target A) and RRM1 (target B), with the ligand displayed as sticks.}
\end{figure}

We also report the average difference between the minimum and maximum docking scores under guidance. In \cref{tab:FKC_NOT}, the expected behavior is an increase in the maximum score and a decrease in the minimum score, consistent with improved on-target binding and reduced off-target binding. To the best of our knowledge, this use case has not been systematically studied in prior ligand–protein drug design work.
The \textsc{and-not} 
formulation enables selective target engagement, generating ligands that bind one protein while avoiding another. As shown in \cref{tab:FKC_NOT}, \method{} maintains stable generation compared to other methods. The \textsc{xor} composition ($A \oplus B$) allows the generative model to autonomously select which target to prioritize, achieving comparable performance.

The validity and uniqueness (V. \& U.) of the generated ligands are also presented, as well as diversity and quality metrics summarizing the drug likeness (QED~\citep{bickerton2012quantifying}) and their synthetic accessibility (SA ~\citep{ertl2009estimation}). 
Notably, the OR composition (A $\lor$ B) with tempering yields product scores approaching those of the AND composition while providing greater flexibility in target engagement. The method preserves molecular quality metrics
competitive with specialized dual-target methods like DualDiff.
We extract posterior probability estimates from the diffusion model for the logical guidance rules as explained in \cref{annex:classifier}.

\newcommand{\grouphdr}[1]{%
  \addlinespace[0.4em]
  \cmidrule(lr){1-7}
  \multicolumn{7}{c}{\bfseries #1}\\
  \addlinespace[0.2em]
}

\newcommand{\grouphdrNO}[1]{%
  \midrule
  \multicolumn{7}{c}{\bfseries #1}\\
  \addlinespace[0.2em]
}

\begin{table}[t!]
\caption{Evaluation of \method{} to generate drug candidates with dual on-targets, along with the baselines: TargetDiff~\citep{guan20233d}, and DualDiff~\citep{yang2024rethinking}. A higher average docking score should correlate with a higher binding affinity to both targets.}
\label{tab:FKC_AND} 
\centering
\setlength{\tabcolsep}{3pt}
\vspace{-0.9em}
\resizebox{\linewidth}{!}{
\begin{tabular}{lcccccc}
\toprule

& (\texttt{A} * \texttt{B}) $\uparrow$ & \texttt{A} $\downarrow$ & \texttt{B} $\downarrow$ & Div. $\uparrow$ & V. \& U. $\uparrow$ & Qual. $\uparrow$ \\ \midrule 

TargetDiff & $71.01_{\pm 2.96}$ & $-8.98_{\pm 0.21}$ & $-7.86_{\pm 0.15}$ & $0.88_{\pm 0.01}$ & $0.94_{\pm 0.04}$ & $0.28_{\pm 0.04}$ \\ 

\grouphdr{$A \land B$}

DualDiff  & $71.87_{\pm 3.33}$ & $-8.85_{\pm 0.24}$ & $-8.08_{\pm 0.17}$ & $0.89_{\pm 0.01}$ & $0.97_{\pm 0.03}$ & $0.27_{\pm 0.10}$ \\ 

\method{}   & $73.20_{\pm 3.18}$ & $-8.99_{\pm 0.24}$ & $-8.11_{\pm 0.17}$ & $0.89_{\pm 0.01}$ & $0.98_{\pm 0.02}$ & $0.26_{\pm 0.08}$ \\ 

\grouphdr{$A \lor B$}

\method{}   & ${\bf 73.91}_{\pm 0.71}$ & $-9.01_{\pm 0.03}$ & $-8.16_{\pm 0.10}$ & $0.90_{\pm 0.00}$ & $0.99_{\pm 0.01}$ & $0.16_{\pm 0.05}$ \\

\bottomrule
\end{tabular}
}
\end{table}

\begin{table}[t!]
\setlength{\tabcolsep}{3pt}
\caption{Evaluation of \method{} with on and off targets, along with the baselines: TargetDiff and  DualDiff. The $A \land \lnot B$ uses $A$ as on-target and $B$ as off-target; in the $A \oplus B$, the roles of on- and off-targets are left to the generative model. 
}
\label{tab:FKC_NOT} 
\centering
\vspace{-0.9em}
\resizebox{\linewidth}{!}{
\begin{tabular}{lcccccc}
\toprule

& $\Delta$(\texttt{A}, \texttt{B}) $\uparrow$ & \texttt{A} $\downarrow$ & \texttt{B} $\uparrow$ & Div. $\uparrow$ & V. \& U. $\uparrow$ & Qual. $\uparrow$ \\ 

\grouphdrNO{$A \land \lnot B$}

DualDiff & $0.28_{\pm 0.08}$ & $-13.35_{\pm 1.07}$ & $-13.53_{\pm 1.23}$ & $0.85_{\pm 0.01}$ & $0.91_{\pm 0.09}$ & $0.00_{\pm 0.00}$ \\ 

SDE & $0.86_{\pm 0.13}$ & $-9.06_{\pm 0.16}$ & $-8.24_{\pm 0.08}$ & $0.89_{\pm 0.00}$ & $0.99_{\pm 0.03}$ & $0.28_{\pm 0.05}$ \\ 

\method{} & ${\bf 0.94}_{\pm 0.24}$ & $-9.01_{\pm 0.19}$ & $-8.11_{\pm 0.14}$ & $0.90_{\pm 0.00}$ & $1.00_{\pm 0.00}$ & $0.26_{\pm 0.11}$ \\ 

\grouphdr{$A \oplus B$}

& $\Delta$(\texttt{A}, \texttt{B}) $\uparrow$ & \texttt{A} & \texttt{B} & Div. $\uparrow$ & V. \& U. $\uparrow$ & Qual. $\uparrow$ \\ 

\method{}   & $0.89_{\pm 0.13}$ & $-9.00_{\pm 0.03}$ & $-8.15_{\pm 0.11}$ & $0.90_{\pm 0.00}$ & $0.99_{\pm 0.01}$ & $0.16_{\pm 0.05}$ \\

\bottomrule
\end{tabular}
}
\end{table}

\section{Related Work}
The composition of diffusion models is typically realized either through their connection to energy-based models (EBMs) \cite{du:2023, nie:2021, ajay:2023} or via linear combinations of learned score functions \cite{liu:2022, kong:2025}. The latter view can be interpreted as a generalized form of guidance \cite{ho:2021, dhariwal:2021, bansal:2023}, in which conditional score estimates are combined using fixed scalar weights. In this work, we adopt this guidance-based perspective but extend it to complex Boolean expressions by replacing fixed weights with probability-dependent posterior coefficients.

A parallel line of work studies \emph{model composition}, combining separate pre-trained models to merge their capabilities. This is achieved either through weight averaging \cite{zhong:2024, biggs:2024} or via inference-time density composition, including probability-weighted estimators \cite{skreta:2025}, fuzzy logic operators \cite{blohm:2025}, and classifier guidance \cite{garipov:2023}. In contrast to approaches that compose multiple models, we focus on attribute composition using scores derived from a single conditional diffusion model.
Gaudi et al.~\cite{gaudi:2025} also consider logical composition, but address it by modifying training objectives to encourage attribute independence. Their approach complements ours and can be naturally integrated with the proposed logical guidance rules. Prior work has explored limited forms of distribution composition for molecular systems, including ligand–protein design \citep{yang2024rethinking}, distribution annealing and product-of-experts formulations \citep{skreta2025feynman}, combinations of learned distributions \cite{skreta:2025}, and antibody design \citep{alesiani2025guidance}. By comparison, we introduce a general and principled framework for computing logical scores via recursive composition rules.

\begin{table}[t!]
\setlength{\tabcolsep}{3pt}
\caption{
Evaluation of FKC \citep{skreta2025feynman} 
when generating drug candidates with dual on-targets.
FKC improves performance for the AND case while reducing performance 
for OR.
}
\label{tab:FKC_AND_correction} 
\centering
\vspace{-0.9em}
\resizebox{\linewidth}{!}{
\begin{tabular}{lcccccc}
\toprule

& (\texttt{A} * \texttt{B}) $\uparrow$ & \texttt{A} $\downarrow$ & \texttt{B} $\downarrow$ & Div. $\uparrow$ & V. \& U. $\uparrow$ & Qual. $\uparrow$ \\ 

\grouphdrNO{$A \land B$ (FKC)}

\method{} & $80.89_{\pm 6.60}$ & $-9.57_{\pm 0.47}$ & $-8.43_{\pm 0.34}$ & $0.73_{\pm 0.04}$ & $0.77_{\pm 0.16}$ & $0.23_{\pm 0.20}$ \\

\grouphdr{$A \lor B$ (FKC)}

\method{} & $73.51_{\pm 5.42}$ & $-9.02_{\pm 0.33}$ & $-8.13_{\pm 0.30}$ & $0.69_{\pm 0.04}$ & $0.84_{\pm 0.07}$ & $0.25_{\pm 0.11}$ \\

\bottomrule
\end{tabular}
}
\end{table}

\begin{table}[t!]
\setlength{\tabcolsep}{3pt}
\caption{
Evaluation of FKC 
when generating drug candidates with one on-target and one off-target.
As with the dual-target case, FKC improves the performance of \method{}, especially for the XOR case.  
}
\label{tab:FKC_NOT_correction} 
\centering
\vspace{-0.9em}
\setlength{\tabcolsep}{3pt}
\resizebox{\linewidth}{!}{
\begin{tabular}{lcccccc}

\grouphdrNO{$A \land \lnot B$ (FKC)}

& $\Delta$(\texttt{A}, \texttt{B}) $\uparrow$ & \texttt{A} $\downarrow$ & \texttt{B} $\uparrow$ & Div. $\uparrow$ & V. \& U. $\uparrow$ & Qual. $\uparrow$ \\

SDE & $0.96_{\pm 0.34}$ & $-8.75_{\pm 0.32}$ & $-7.79_{\pm 0.46}$ & $0.77_{\pm 0.08}$ & $0.73_{\pm 0.24}$ & $0.30_{\pm 0.24}$ \\ 

\method{} & $0.96_{\pm 0.35}$ & $-8.84_{\pm 0.84}$ & $-7.92_{\pm 0.93}$ & $0.78_{\pm 0.07}$ & $0.83_{\pm 0.15}$ & $0.25_{\pm 0.20}$ \\

\grouphdr{$A \oplus B$ (FKC)}

& $\Delta$(\texttt{A}, \texttt{B}) $\uparrow$ & \texttt{A}  & \texttt{B} & Div. $\uparrow$ & V. \& U. $\uparrow$ & Qual. $\uparrow$ \\ 
\method{} & $1.00_{\pm 0.57}$ & $-9.52_{\pm 0.50}$ & $-8.61_{\pm 0.55}$ & $0.69_{\pm 0.02}$ & $0.83_{\pm 0.04}$ & $0.33_{\pm 0.24}$ \\ 

\bottomrule
\end{tabular}
}
\end{table}

\section{Discussion}

\method{} translates complex logical expressions into exact probabilistically consistent guidance terms: for disjunctions and negations, the guidance depends on the posterior of each formula, rather than fixed coefficients. This explains the failure modes of the constant baselines on these queries
(\Cref{fig:qualitative_comparison_shapes3d_cmnist}). In contrast, the
recursive rules in \Cref{tab:guidance_rules} adapt the coefficients,
which improves the conformity-diversity trade-off (\Cref{fig:gw_acc_entropy_tradeoff}). The sufficient conditions in \Cref{prop:exact-meci-fragment} connect logical guidance to tractable inference in probabilistic circuits. For a formula admitting a circuit representation with decomposable conjunctions and conditionally-independent or mutually-exclusive disjunctions, evaluating $s_t(\varphi,\bx)$ reduces to a linear-time circuit pass at each diffusion step. This connection also clarifies when certain queries can be evaluated exactly, such as those involving constraints defined by mutually independent categorical variables (\Cref{prop:categorical_simple}).

\textbf{Limitations.} The exactness of the method relies on properties (CI/ME) that may only hold approximately. The method also depends on accurate posterior probability estimates. Deviations of these estimates can be amplified by the nonlinear coefficients, especially for negation through $\widehat\pi/(1-\widehat\pi)$. Practical stabilizations (e.g., clipping) can improve robustness but deviate from the exact theory. Finally, while circuit evaluation is cheap, the overall cost still scales with the number of atomic predicates in the logical expressions.

\section*{Impact Statement}
This paper presents work whose goal is to advance the field of Machine Learning. There are many potential societal consequences of our work, none which we feel must be specifically highlighted here.

\bibliography{bib/guidance,bib/ab_pref,bib/alignab,bib/FKC,bib/references}
\bibliographystyle{icml2026}

\newpage
\appendix
\onecolumn
\section*{Supplementary Material}

\section{Proofs}\label{annex:proof}

\showDeferredProofs

\newpage
\section{Example of recursive construction}
We illustrate the recursive construction of posterior coefficients and logical scores on the formula from \Cref{fig:overview}
\[
\varphi = (1 \land \text{blue}) \lor (9 \land \text{red}).
\]
We assume conditional independence of subformulas and use $\widehat\pi_i \approx p_t(c_i\mid\bx)$ and
$s_t(c_i,\bx)=s_i(t,\bx)-s_\emptyset(t,\bx)$ as atomic input

\smallskip
\noindent\emph{(1) Posterior coefficients.}
\begin{align*}
\widehat\pi(1) &= \widehat\pi_1, & \\
\widehat\pi(\text{blue}) &= \widehat\pi_\text{blue}, & \\
\widehat\pi(3) &= \widehat\pi_9, & \\
\widehat\pi(\text{red}) &= \widehat\pi_\text{red}, & \\
\widehat\pi(\psi) &:= \widehat\pi(1\land \text{blue}) = \widehat\pi_1\widehat\pi_\text{blue}\\
\widehat\pi(\chi) &:= \widehat\pi(9\land \text{red}) = \widehat\pi_9\widehat\pi_\text{red}\\
\widehat\pi(\varphi)
&= \widehat\pi(\psi)+\widehat\pi(\chi) = \widehat\pi_1\widehat\pi_\text{blue} + \widehat\pi_9\widehat\pi_\text{red}&
\end{align*}

\smallskip
\noindent\emph{(2) Logical scores.}
\begin{align*}
s_t(c_i,\bx) &= s_i(t,\bx)-s_\emptyset(t,\bx), \\[0.3em]
s_t(\psi,\bx) &= s_t(1,\bx)+s_t(\text{blue},\bx) \\[0.3em]
s_t(\chi,\bx) &= s_t(9,\bx)+s_t(\text{red},\bx) \\[0.3em]
s_t(\varphi,\bx) &= \dfrac{\widehat\pi(\psi) s_t(\psi,\bx) + \widehat\pi(\chi) s_t(\chi,\bx)}{\widehat\pi(\psi)+\widehat\pi(\chi)}
\end{align*}
Substituting the posterior coefficients and logical scores gives the fully expanded expression:
\[
\boxed{
    s_t(\varphi,\bx) =\dfrac{\widehat\pi_1\widehat\pi_\text{blue} \big(s_t(1,\bx)+s_t(\text{blue},\bx)\big) + \widehat\pi_9\widehat\pi_\text{red} \big(s_t(9,\bx)+s_t(\text{red},\bx)\big)}{\widehat\pi_1\widehat\pi_\text{blue}+\widehat\pi_9\widehat\pi_\text{red}}
}
\]

All $\widehat\pi_i$ are scalar coefficients,
and all spatial derivatives come from the diffusion model via
$s_t(c_i,\bx)=s_i-s_\emptyset$.

\newpage
\section{Additional Theoretical Results}\label{annex:theoretical_results}

\subsection{Taxonomy Queries}

In many practical settings, the atomic predicates are not unrelated attributes, but form a \emph{taxonomy}. Typical examples include biological classifications (e.g., phylum, genus, species), object hierarchies in vision (e.g., vehicle, car, sedan), or semantic type systems. In such cases, predicates are either mutually exclusive (siblings in the taxonomy) or strictly nested (a child predicate implies its parent). Users often wish to specify constraints at different levels of this hierarchy, for example, allowing any instance of a broad category (e.g., mammal) or restricting generation to a small number of specific subcategories. The following proposition shows that any such taxonomy query, specified simply as a set of allowed nodes in the hierarchy, admits an exact logical guidance rule under our framework.

\begin{proposition}
\label{prop:taxonomy_queries}
Fix $t\in(0,T]$ and $\bx\in\mathcal X$.
Let $\mathcal T$ be a finite taxonomy of propositions, that is, a set of predicates $\{c_u : u\in V\}$ indexed by nodes of a rooted tree $(V,\preceq)$, with root $r$ such that
\[
c_r \equiv \top,
\qquad\text{and}\qquad
u \preceq v \ \Rightarrow\  \llbracket c_u\rrbracket \subseteq \llbracket c_v\rrbracket,
\]
and siblings are mutually exclusive:
if $u$ and $v$ are distinct children of the same parent, then
$\llbracket c_u\rrbracket \cap \llbracket c_v\rrbracket = \emptyset$.

A \emph{taxonomy query} is specified by a set of allowed nodes $A\subseteq V$, interpreted as the event
\[ \varphi_A := \bigvee_{u\in A} c_u. \]
There exists a semantically equivalent formula $\psi_A$ whose circuit uses only $\neg$-nodes and OR-ME $\vee$-nodes.
Consequently, the recursive rules in \Cref{tab:guidance_rules} compute
$p_t(\varphi_A\mid\bx)$ and $s_t(\varphi_A,\bx)$ exactly.
\end{proposition}

\begin{proof}
For each node $u\in V$, define its \emph{exclusive refinement} (``exactly $u$ and none of its descendants'')
\[
r_u := c_u \wedge \neg\Big(\bigvee_{v\in \mathrm{ch}(u)} c_v\Big),
\]
where $\mathrm{ch}(u)$ are the children of $u$.
Because siblings are mutually exclusive, the disjunction $\bigvee_{v\in \mathrm{ch}(u)} c_v$ is OR-ME.

Moreover, the family $\{r_u : u\in V\}$ is pairwise mutually exclusive, and its union is $\top$:
every terminal sample belongs to exactly one ``most specific'' taxonomy node along its root-to-leaf path.
Hence any taxonomy query $\varphi_A=\bigvee_{u\in A} c_u$ is semantically equivalent to the (mutually exclusive) disjunction of all
exclusive refinements contained in it:
\begin{eqnarray}\label{eq:tax_comp}
\varphi'_A := \bigvee_{u\in V:\ \llbracket r_u\rrbracket \subseteq \llbracket \varphi_A\rrbracket} r_u.
\end{eqnarray}
The right-hand side is an OR-ME disjunction of subformulas built using only negation and OR-ME disjunctions, and we have $\llbracket \varphi_A \rrbracket=\llbracket \varphi'_A\rrbracket$.
Therefore $\varphi'_A$ admits a circuit satisfying Prop.~\ref{prop:exact-meci-fragment}, and exactness of the recursive rules follows.
\end{proof}

\subsection{Completeness Conditions for Circuit Compilability}

We presented in \cref{prop:exact-meci-fragment} a set of exact recursive composition rules.  We now show that these recursive rules are complete for two classes of distributions of general interest.

The first class concerns properties that may be partitioned into multiple groups, where properties within the same group are mutually exclusive and those between groups are conditionally independent.  This situation arises, for example, when we have multiple categorical properties, each of which are treated as independent.  For instance, we may have color and digit-class as the groups, and the properties $\{\text{red},\text{green},\text{blue}\}$ in the former group and $\{0,...,9\}$ in the latter.  We note that this general case also covers multiple conditionally independent binary properties, since the groups are singletons. 

The second class concerns sets of properties that are strictly nested, hence for any two properties, we have either that they are mutually exclusive, or that one property implies the other.  This case covers the situation where we have a strictly nested ontology, for instance, a taxonomy of microbial species, genera and phyla.

To demonstrate completeness, we introduce the notion of the {\em compilability} of a Boolean formula $\varphi$.  We write $PC(\varphi)$ for the probabilistic circuit formed by directly parsing $\varphi$, so that conjunctions of subformulas are mapped to conjunction nodes, and similarly for disjunctions and negations.  $PC(\varphi)$ provides a direct representation of $\varphi$, and the function $PC(.)$ provides a one-to-one mapping between Boolean formulas and probabilistic circuits.  However, the circuit $PC(\varphi)$ may be inadequate for certain purposes; in particular, it may not satisfy the conditions of \cref{prop:exact-meci-fragment}.  A {\em compilation} of $\varphi$, denoted $PC^*(\varphi)$, allows $PC(\varphi)$ to be further manipulated according to a compilation scheme, to produce a circuit $PC^*(\varphi)=PC(\varphi')$ with desirable properties, which maintains the semantics of the original Boolean formula.  Hence, we must have $\llbracket \varphi\rrbracket = \llbracket \varphi'\rrbracket$.  A formula $\varphi$ is compilable for a given distribution and class of valid circuits iff $PC^*(\varphi)$ belongs to the class of valid circuits.

In the following proposition, we show that the logical composition rules we provide in \cref{tab:guidance_rules} are complete for the classes of distribution noted above, in the sense that any Boolean formula $\varphi$ may be compiled into a probabilistic circuit fulfilling the conditions of \cref{prop:exact-meci-fragment}, which can therefore be exactly evaluated.  We note that, for the second class of distribution, all formulas compile to probabilistic circuits in a more restricted set (excluding the $\wedge$-CI rule), and we have the additional guarantee that, if these conditions are satisfied at terminal time ($t=0$), they are satisfied for all $\bx_t$ and $t>0$.

\vspace{0.5cm}
\begin{proposition}
\label{prop:completeness}
Given a distribution over $\bx_0$ conditioned on $\bX_t=\bx_t$ with $t\in(0,\;T]$, and predicates $\{c_1,...c_N\}$ over $\bx_0$, for any Boolean formula $\varphi$ over the same predicates such that $\llbracket \varphi\rrbracket \neq \llbracket \bot\rrbracket$, we have that $\varphi$ may be compiled into a probabilistic circuit evaluable by the rules of \cref{tab:guidance_rules} at $\bX_t=\bx_t$ when the atomic predicates satisfy:

\begin{enumerate}
\item The predicates fall into subsets, $S_1,...,S_M$, where $S_m\subseteq\{1...N\}$, $m_1\neq m_2 \Rightarrow S_{m_1} \cap S_{m_2} = \emptyset$ and $\cup_mS_m=\{1...N\}$, such that (a) $n_1, n_2 \in S_m$ and $n_1\neq n_2$ implies that the events $c_{n_1}$ and $c_{n_2}$ are ME given $\bX_t=\bx_t$, (b) for all tuples $(n_{m=1},...,n_{m=M})$ such that $\forall m.(n_m\in S_m)$, we have that the events $c_{n_1},...,c_{n_M}$ are mutually CI given $\bX_t=\bx_t$, and (c) for all $m$, $p(\neg \vee_{n \in S_m}c_n)=0$. 
\end{enumerate}

Moreover, we have that $\varphi$ may be compiled into a probabilistic circuit evaluable by the rules of \cref{tab:guidance_rules} at all $t\geq0$ and settings $\bX_t=\bx_t$ when, at terminal time $(t=0)$, the atomic predicates satisfy:

\begin{enumerate}
\setcounter{enumi}{1}
\item For each pair of predicates, $c_i$ and $c_j$ ($i\neq j$), we have that $\llbracket c_i\rrbracket\cap\llbracket c_j\rrbracket\in\{\emptyset,\llbracket c_i\rrbracket,\llbracket c_j\rrbracket \}$ (i.e., they are ME or nested).
\end{enumerate}

\end{proposition}

\vspace{0.25cm}
\begin{proof}
We consider the two cases in the proposition separately.

{\em Case 1:} For case (1) of \cref{prop:completeness}, we consider an arbitrary predicate, $\varphi$.  We may write $\varphi$ in Full Disjunctive Normal Form (FDNF) as:
\begin{eqnarray}\label{eq:DNF}
\varphi \equiv \text{FDNF}(\varphi) :=\bigvee_{i\in R(\varphi)} (\bigwedge_{j\in\{1...N\}} y_{ij})
\end{eqnarray}
where $y_{ij}=\neg c_j$ if $\text{mod}(i,2^j)<2^{j-1}$ and $y_{ij}=c_j$ otherwise; and $R(\varphi)\subseteq\{0...(2^N-1)\}$ such that $n\in R(\varphi) \iff \varphi\wedge  (\bigwedge_j y_{nj})\not\equiv\bot$.  We note that \cref{eq:DNF} is a disjunction of subformulas, $Y_i=\bigwedge_j y_{ij}$, each of which is, by definition, mutually exclusive.  By the assumptions of case (1) of the proposition, if, for a given $i$, there exists an $m$ such that $(\sum_{j\in S_i}\mu(i,j))\neq1$, where $\mu(i,j)=[\text{mod}(i,2^j)\geq 2^{j-1}]$  and $[.]$ denotes the Iverson bracket, we must have $\llbracket Y_i\rrbracket = \llbracket \bot\rrbracket$.  Hence, we may consider $R'(\varphi)=R(\varphi)\backslash\{i|\exists m.(\sum_{n\in S_m}\mu(i,j))\neq 1\}$.  Moreover, since for any $i\in R'(\varphi)$, there is exactly one $j$ from any set $S_m$ such that $\mu(i,j)=1$, we may consider the predicate:
\begin{eqnarray}
\varphi' := \bigvee_{i\in R'(\varphi)} (\bigwedge_{m\in\{1...M\}} Y'_{im})
\end{eqnarray}
where $Y'_{im}=c_{j(i,m)}$ such that $j(i,m)$ is the unique $j\in S_m$ such that $\mu(i,j)=1$.  We thus have $\llbracket \varphi\rrbracket=\llbracket \text{FDNF}(\varphi)\rrbracket =\llbracket \varphi'\rrbracket$, where the first equality holds by definition, and the second holds since $\varphi'$ is formed by excluding only those terms from the disjunction $\text{FDNF}(\varphi)$ corresponding to the event $\emptyset$.  We observe that $\varphi'$ is a categorical-CI query, as defined in \cref{sec:comp_calc}, since conditions $(a)$ and $(c)$ in the proposition ensure that the properties $\{c_n|n\in S_m\}$ may be treated as values in the domain of a categorical variable, only one of which may be selected, and condition $(b)$ ensures that the categorical variables $1...M$ are mutually CI, since for any configuration 
$(c_{n_{m}}=1,...,c_{n_{M}}=1)$ where $\forall m.(n_m\in S_m)$, we have $p(c_{n_{m}}=1,...,c_{n_{M}}=1)=\prod_mp(c_{n_{m}}=1)$. Hence $PC(\varphi')$ may be evaluated exactly according to our framework, following \cref{prop:categorical_simple}. $\varphi$ may thus be compiled to the circuit $PC^*(\varphi)=PC(\varphi')$, which can be evaluated exactly as required.

{\em Case 2:} For case (2) of \cref{prop:completeness}, we have a Boolean formula $\varphi$ over atoms $\mathcal{C}=\{c_1,...,c_N\}$ satisfying the condition that any two distinct atoms are ME or nested.  We first consider an expansion of the original system to $\mathcal{C}^{\dagger}\supseteq\mathcal{C}$. To form $\mathcal{C}^{\dagger}$, we begin by adding an atom $c_{\top}$ which is semantically equivalent to $\top$, if no such atom exists in $\mathcal{C}$; this preserves the property in (2), since all other atoms must be nested inside $\top$.  Further, for every atom $c_{1...N}$, we add a new atom $c'_n$, corresponding to the event:
\begin{eqnarray}\label{eq:expansion}
\llbracket c'_n\rrbracket = \llbracket c_n\wedge \neg(\vee_{m\in C_n}c_m) \rrbracket 
\end{eqnarray}
where $C_n=\{m\leq N|\llbracket c_n\rrbracket\supset\llbracket c_m\rrbracket\wedge \neg\exists l\leq N.\llbracket c_n\rrbracket\supset\llbracket c_l\rrbracket\supset\llbracket c_m\rrbracket\}$, for all $n\leq N$ except those for which $\llbracket c'_n\rrbracket=\emptyset$.  The expanded system $\mathcal{C}^{\dagger}$ will retain the property in (2), since for all $c'_n$, we have that $c'_n$ is nested inside $c_n$, and for any other atom $c_m$ in $\mathcal{C}^{\dagger}$, if $c_n$ is ME or nested inside $c_m$, $c'_n$ will likewise be ME or nested inside $c_m$ respectively, and if $c_m$ is nested inside $c_n$, $c_m$ and $c'_n$ will be ME.  We observe that the model defined over $\mathcal{C}^{\dagger}$ is equivalent in expressive power to $\mathcal{C}$, since any formula $\varphi$ over $\mathcal{C}$ is also a formula over $\mathcal{C}^{\dagger}$, and any formula $\phi$ over $\mathcal{C}^{\dagger}$ may be mapped to a semantically equivalent formula $\varphi=K(\phi)$ over $\mathcal{C}$ by replacing all occurrences of $c_{\top}$ with $\top$, and all occurrences of $c'_n$ with $c_n\wedge \neg(\vee_{m\in C_n}c_m)$.  

We now consider compiling an arbitrary Boolean circuit $\varphi$ over $\mathcal{C}$ given the assumption in (2).  We first lift $\varphi$ to a formula $\phi$ over $\mathcal{C}^{\dagger}$; as noted, we may simply set $\phi=\varphi$.  We now consider $\text{FDNF}(\phi)$, defined analogously to Eq. \ref{eq:DNF} (where we note that we now have $N^{\dagger}=|\mathcal{C}^{\dagger}|$ atoms).  For each $i$ let $J'_i=\{j|\mu(i,j)=1\}$.  If $\exists j\in J_i$ s.t. $\forall j'\in J_i$, $\llbracket c^{\dagger}_j \wedge c^{\dagger}_{j'}\rrbracket=\llbracket c^{\dagger}_j\rrbracket$, and further that $\neg(\exists j'\not\in J_i.\llbracket c^{\dagger}_j\wedge c^{\dagger}_{j'}\rrbracket = \llbracket c^{\dagger}_{j'}\rrbracket)$, then we have $\llbracket Y_i\rrbracket = \llbracket c^{\dagger}_{j(i)}\rrbracket$, where we write $j(i)$ for the unique such $j$.  For all other $i$, we have $\llbracket Y_i\rrbracket = \llbracket \bot\rrbracket$, since either there must exist $j,j'\in J_i$, $j\neq j'$ such that $c^{\dagger}_j$ and $c^{\dagger}_{j'}$ are ME, or $\forall j' s.t. \llbracket c^{\dagger}_j\wedge c^{\dagger}_{j'}\rrbracket = \llbracket c^{\dagger}_{j'}\rrbracket.j'\not\in J_i$.  Hence, we may consider $R''(\phi)=R(\phi)\backslash\{i|\exists j\in J_i.(\forall j'\in J_i.\llbracket c^{\dagger}_j \wedge c^{\dagger}_{j'}\rrbracket=\llbracket c^{\dagger}_j)\rrbracket\wedge\neg(\exists j'\not\in J_i.\llbracket c^{\dagger}_j\wedge c^{\dagger}_{j'}\rrbracket = \llbracket c^{\dagger}\rrbracket_{j'})\}$.  We now consider the predicate:
\begin{eqnarray}
\phi' := \bigvee_{i\in R''(\phi)} c^{\dagger}_{j(i)}
\end{eqnarray}
By definition, $\llbracket \phi\rrbracket = \llbracket \phi'\rrbracket = \llbracket \varphi\rrbracket$. We also have that $PC(\phi')$ is a {\em taxonomy query} as defined above, and so is exactly evaluable in our framework by \cref{prop:taxonomy_queries}.  However, $\phi'$ is defined in the expanded system $\mathcal{C}^{\dagger}$, and so it may include atoms not in the original system.  We therefore map $\phi'$ to a predicate $\varphi'=K(\phi')$ over $\mathcal{C}$ using the transformation $K(.)$ defined above.  This will result in a formula  $\varphi'$ which is a disjunction of ME subformulas, where the latter are all either atoms or negations of disjunctions of ME atoms (since $C_n$ contains only the direct children of $c_n$, which by definition cannot be nested).  Like $\phi'$, $\varphi'$ is evaluable in our framework, since all nodes in $PC(\varphi')$ are either $\neg$ or $\vee$-MI nodes (we note that $\wedge$-CI and $\vee$-CI are not necessary).  Hence, $\varphi$ may be compiled to the circuit $PC^*(\varphi)=PC(\varphi')$.  If $c_i$ and $c_j$ are ME ($\llbracket c_i\wedge c_j\rrbracket=\emptyset$) at terminal time $t=0$, they will be ME at all $\bx$ and $t>0$.  Hence, it is sufficient that (2) is satisfied at terminal time for $\varphi$ to be compilable at all $t$ and $\bX_t=\bx_t$, as required.

\end{proof}

\subsection{Discrete Logical Guidance}\label{annex:discrete}

We briefly consider the here case in which, instead of score functions for the unconditional and atomic conditional models, we instead are supplied with transition functions for these generative processes for a set of discrete time-steps.  The underlying space $\mathcal{X}$ may be either continuous or discrete.  We denote the unconditional transition function as $\tau_t(\bx_{t-1}|\bx_t)$, and the transition function conditioned on predicate $c_i$ as $\tau_t(\bx_{t-1}|c_i,\bx_t)$.  Then, under the same structural assumptions as \cref{prop:exact-meci-fragment}, $\tau_t(\bx_{t-1}|\varphi,\bx_t)$ and $\widehat\pi(\varphi)$ may be calculated by mutual recursion, where the rules for calculating $\widehat\pi(\varphi)$ are as in \cref{tab:guidance_rules}, and the rules for calculating $\tau_t(\bx_{t-1}|\varphi,\bx_t)$ follow an analogous pattern to those for $s(\varphi,\bx)$ in \cref{tab:guidance_rules}:

\begin{eqnarray}
\tau_t(\bx_{t-1}|\bx_t), \tau_t(\bx_{t-1}|c_i,\bx_t) && \text{provided} \nonumber \\ 
\tau_t(\bx_{t-1}|\lnot\psi,\bx_t)
&=& 
 \quad \frac{\tau_t(\bx_{t-1}|\bx_t) - \widehat\pi(\psi)\tau_t(\bx_{t-1}|\psi,\bx_t)}{1-\widehat\pi(\psi)} ,\nonumber\\[0.5em]
\tau^{\mathrm{CI}}_t(\bx_{t-1}|\psi\land\chi,\bx_t)
&=&
\quad\frac{\tau_t(\bx_{t-1}|\psi,\bx_t)\tau_t(\bx_{t-1}|\chi,\bx_t)}{\tau_t(\bx_{t-1}|\bx_t)},\nonumber\\[0.5em]
\tau^{\mathrm{ME}}_t(\bx_{t-1}|\psi\lor\chi,\bx_t)
&=&
\quad \frac{
\widehat\pi(\psi)\tau_t(\bx_{t-1}|\psi,\bx_t)+\widehat\pi(\chi)\tau_t(\bx_{t-1}|\chi,\bx_t)
}{
\widehat\pi(\psi)+\widehat\pi(\chi)
}, \nonumber\\[0.5em] 
\tau^{\mathrm{CI}}_t(\bx_{t-1}|\psi\lor\chi,\bx_t)
&=& \frac{\widehat\pi(\psi)\tau_t(\bx_{t-1}|\psi,\bx_t)+\widehat\pi(\chi)\tau_t(\bx_{t-1}|\chi,\bx_t)-\widehat\pi(\psi)\widehat\pi(\chi)\tau^{\mathrm{CI}}_t(\bx_{t-1}|\psi\land\chi,\bx_t)}{\widehat\pi(\psi)+\widehat\pi(\chi)-\widehat\pi(\psi)\widehat\pi(\chi)}. \label{eq:rec_tau} \nonumber \\
\end{eqnarray}

We thus have the following discrete equivalent of \cref{prop:exact-meci-fragment}:

\begin{proposition}
\label{prop:exact-discrete-fragment}
Let $\varphi$ be a propositional formula over atoms $\{c_i\}$.  Suppose that $\varphi$ admits a circuit
representation whose internal nodes are  $\land$, $\lor$, and
$\lnot$, whose $\land$- and $\lor$-nodes satisfy conditions (1) and (2) from \cref{prop:exact-meci-fragment}, and that that $\widehat\pi(\psi)=p_t(\psi\mid \bx)$ for all subformulas
$\psi$ of $\varphi$.  Then the recursive rules above in \cref{eq:rec_tau} reproduce exactly the true posterior and conditional transition functions for $\varphi$:
\[
\widehat\pi(\varphi) = p_t(\varphi\mid \bx),
\qquad
\tau_t(\bx_{t-1}|\varphi,\bx_t) = p(\bx_{t-1}|\varphi,\bx_t).
\]
\end{proposition}

\begin{proof}
We argue by structural induction on the circuit of $\varphi$.  Since the inductive steps for $\widehat\pi(.)$ are identical to those in \cref{prop:exact-meci-fragment}, we omit them from the proof below.

\emph{Base case (atoms).}
For an atomic predicate $c_i$ we have by assumption that $\tau_t(\bx_{t-1}|c_i,\bx_t)$ are provided.  Thus the claim holds.

\emph{Inductive step (negation).}
Suppose the claim holds for $\psi$.  Then by definition we have:
\begin{eqnarray}
p(\bx_{t-1} \mid \lnot\psi) &=& \frac{p(\lnot\psi\mid\bx_{t-1},\bx_t)\tau_t(\bx_{t-1}|\bx_t)}{1-\widehat\pi(\psi)} \nonumber \\
&=& \frac{(1-p(\psi\mid\bx_{t-1},\bx_t)))\tau_t(\bx_{t-1}|\bx_t)}{1-\widehat\pi(\psi)} \nonumber \\
&=& \frac{\tau_t(\bx_{t-1}|\bx_t) - \tau_t(\bx_{t-1}|\bx_t)p(\psi\mid\bx_{t-1},\bx_t))}{1-\widehat\pi(\psi)} \nonumber \\
&=& \frac{\tau_t(\bx_{t-1}|\bx_t) - \widehat\pi(\psi)\tau_t(\bx_{t-1}|\psi,\bx_t)}{1-\widehat\pi(\psi)} \nonumber
\end{eqnarray}
which matches the recursive rule for $\tau_t(\bx_{t-1}|\lnot\psi,\bx_t)$.

\emph{Inductive step (conjunction).}
Let $\varphi=\psi\land\chi$ and assume the inductive hypothesis for
$\psi$ and $\chi$.  By assumption,
$p_t(\psi\land\chi\mid \bx)=p_t(\psi\mid \bx)p_t(\chi\mid \bx)$; hence, we have:
\begin{eqnarray}
p(\bx_{t-1} \mid \psi\land\chi) &=& \frac{p(\psi\land\chi\mid \bx_{t-1},\bx_t)\tau_t(\bx_{t-1}|\bx_t)}{p(\psi\land\chi\mid \bx_t)} \nonumber \\
&=& \frac{p(\psi\mid \bx_{t-1},\bx_t)p(\chi\mid \bx_{t-1},\bx_t)\tau_t(\bx_{t-1}|\bx_t)}{\widehat\pi(\psi)\widehat\pi(\chi)} \nonumber \\
&=& \frac{\tau_t(\bx_{t-1}|\psi,\bx_t)\tau_t(\bx_{t-1}|\bx_t)}{\tau_t(\bx_{t-1}|\bx_t)} \nonumber
\end{eqnarray}

which is exactly the recursive rule for $\tau_t(\bx_{t-1}|\psi\wedge \chi,\bx_t)$.

\emph{Inductive step (disjunction).}
Let $\varphi=\psi\lor\chi$ and assume the inductive hypothesis for
$\psi$ and $\chi$.  By inclusion--exclusion,
\[
p_t(\psi\lor\chi\mid \bx)
= p_t(\psi\mid \bx)+p_t(\chi\mid \bx)-p_t(\psi\land\chi\mid \bx).
\]
There are two cases.

If $\psi$ and $\chi$ are mutually exclusive, we have:
\begin{eqnarray}
p(\bx_{t-1} \mid \psi\lor\chi) &=& \frac{p(\psi\lor\chi\mid \bx_{t-1},\bx_t)\tau_t(\bx_{t-1}|\bx_t)}{p(\psi\lor\chi\mid \bx_t)} \nonumber \\
&=& \frac{\left(p(\psi\mid \bx_{t-1},\bx_t)+p(\chi\mid \bx_{t-1},\bx_t)\right)\tau_t(\bx_{t-1}|\bx_t)}{\widehat\pi(\psi)+\widehat\pi(\chi)} \nonumber \\
&=& \frac{\widehat\pi(\psi)\tau_t(\bx_{t-1}|\psi,\bx_t)+\widehat\pi(\chi)\tau_t(\bx_{t-1}|\chi,\bx_t)}{\widehat\pi(\psi)+\widehat\pi(\chi)} \nonumber
\end{eqnarray}

which is exactly the recursive rule for $\tau^{\text{ME}}_t(\bx_{t-1}|\psi\vee \chi,\bx_t)$.

If $\psi$ and $\chi$ are conditionally independent, we have:
\begin{eqnarray}
p(\bx_{t-1} \mid \psi\lor\chi) &=& \frac{p(\psi\lor\chi\mid \bx_{t-1},\bx_t)\tau_t(\bx_{t-1}|\bx_t)}{p(\psi\lor\chi\mid \bx_t)} \nonumber \\
&=& \frac{\left(p(\psi\mid \bx_{t-1},\bx_t)+p(\chi\mid \bx_{t-1},\bx_t)-p(\psi\mid \bx_{t-1},\bx_t)p(\chi\mid \bx_{t-1},\bx_t)\right)\tau_t(\bx_{t-1}|\bx_t)}{\widehat\pi(\psi)+\widehat\pi(\chi)-\widehat\pi(\psi)\widehat\pi(\chi)} \nonumber \\
&=& \frac{\widehat\pi(\psi)\tau_t(\bx_{t-1}|\psi,\bx_t)+\widehat\pi(\chi)\tau_t(\bx_{t-1}|\chi,\bx_t)-\widehat\pi(\psi)\widehat\pi(\chi)\tau^{\mathrm{CI}}_t(\bx_{t-1}|\psi\land\chi,\bx_t)}{\widehat\pi(\psi)+\widehat\pi(\chi)-\widehat\pi(\psi)\widehat\pi(\chi)} \nonumber
\end{eqnarray}

which is exactly the recursive rule for $\tau^{\text{CI}}_t(\bx_{t-1}|\psi\vee \chi,\bx_t)$.

\smallskip
In all cases, the recursive rules reproduce
$\tau_t(\bx_{t-1}|\varphi,\bx_t)$ at the node $\varphi$ assuming they do so for its children.  By structural induction over the circuit of $\varphi$, the claim holds for all subformulas and in particular for $\varphi$ itself.
\end{proof}

Since the conditions on the circuit representation of $\varphi$ are the same in \cref{prop:exact-discrete-fragment} as \cref{prop:exact-meci-fragment}, an analogous completeness result can be shown for the discrete case to \cref{prop:completeness}.  This is shown by substituting the rules in \cref{eq:rec_tau} for those of \ref{tab:guidance_rules}, and the conditions in \cref{prop:exact-discrete-fragment} for those of \cref{prop:exact-meci-fragment} in the statement and proof of \ref{prop:completeness}.  This shows that the rules in \cref{eq:rec_tau} achieve complete compilability for all Boolean formulas under the same conditions (\cref{prop:completeness}) in the discrete case as the continuous case.

\newpage
\section{Implementation Details}
\subsection{Sampling Implementation}
We build our method upon standard diffusion backbones. For simpler benchmarks (e.g., CMNIST), we utilize the Denoising Diffusion Probabilistic Models (DDPM) \cite{ho:2020}. For large-scale experiments on ImageNet, we adopt the EDM2 framework by Karras et al.~\cite{karras:2024}.

The sampling procedure, summarized in \Cref{alg:logical_sampling}, follows the standard reverse diffusion process with a modification analogous to CFG. The logical guidance score is calculated in \Cref{alg:logical_score} recursively implementing the derived rules as stated in \Cref{tab:guidance_rules}. Note that while \Cref{alg:logical_sampling} and \ref{tab:guidance_rules} use $\epsilon$-prediction instead of scores, the composition rules remain unchanged, as $\epsilon$ and the score differ only by a time-dependent scalar factor \cite{karras2022elucidating}.

\begin{algorithm}[th]
\caption{Sampling Implementation with Logical Composition}
\label{alg:logical_sampling}
\begin{algorithmic}[1]
\STATE {\bfseries Input:} Query $\psi$, Diffusion model $\epsilon_\theta$, Classifier $p_\phi$, Timesteps $T$
\item[] 
\STATE $\mathbf{x}_T \sim \mathcal{N}(\mathbf{0}, \mathbf{I})$ 

\FOR{$t = T, \dots, 1$}
    \STATE $\hat{\epsilon}_\emptyset \leftarrow \epsilon_\theta(\mathbf{x}_t, \emptyset, t)$ \quad \textit{// Unconditional score}
    \item[] 
    \STATE \textit{// Guide with Logical Score}
    \STATE $\mathbf{g}_\psi, \_ \leftarrow \textsc{GetLogicalScore}(\psi, \mathbf{x}_t)$ \textit{$\quad$ // \Cref{alg:logical_score}}
    \STATE $\tilde{\epsilon}_t \leftarrow \hat{\epsilon}_\emptyset + \mathbf{g}_\psi$

    \item[] 
    \STATE $\mathbf{x}_{t-1} \leftarrow \textsc{SamplerStep}(\mathbf{x}_t, \tilde{\epsilon}_t, t)$
\ENDFOR

\STATE \textbf{return} $\mathbf{x}_0$
\end{algorithmic}
\end{algorithm}

\begin{algorithm}[ht!]
\caption{Recursive Logical Score Calculation}
\label{alg:logical_score}
\begin{algorithmic}[1]
\STATE {\bfseries Input:} Query $\psi$, Noisy state $\mathbf{x}_t$, Classifier $p_\phi$, Score model $\epsilon_\theta$, Guidance scale $w_\text{guidance}$, Repulsive guidance scale $w_\text{not}$
\item[]
\renewcommand{\algorithmicloop}{\textbf{function} \textsc{GetLogicalScore}($\psi$, $\mathbf{x}_t$)}
\renewcommand{\algorithmicendloop}{\textbf{end function}}

\STATE \textit{// Returns tuple of logical score and probability $(\mathbf{g}_t, \hat{\pi})$ as stated in \Cref{tab:guidance_rules}}

\LOOP

\item[]
    \STATE \textit{// 1. Atomic Condition (Base Case)}
    \IF{$\varphi$ is \textbf{Atom} $c$}
        \STATE $\pi_c \leftarrow p_\psi(c \mid \mathbf{x}_t)$
        \STATE $\mathbf{g}_{t,c} \leftarrow w_\text{guidance} \cdot \big(\epsilon_\theta(\mathbf{x}_t, c) - \hat{\epsilon}_\theta(\mathbf{x}_t, \emptyset)\big)$ 

        \STATE []
        \IF{doRepulsiveGuiding}
            \STATE $c_{\text{not}} \leftarrow \text{argmax}_{k \in \mathcal{C} \setminus \{c\}} \pi_k \quad$  \textit{// Find most probable condition $c_{\text{not}}$ s.t. $c_{\text{not}} \neq c$}
            \STATE $\mathbf{g}_{t,c} \leftarrow w_\text{guidance} \cdot \big(\epsilon_\theta(\mathbf{x}_t, c) - \hat{\epsilon}_\theta(\mathbf{x}_t, \emptyset)\big) - w_\text{not} \cdot\frac{\pi_{c_{\text{not}}}}{1 - \pi_{c_{\text{not}}}} \cdot \big(\epsilon_\theta(\mathbf{x}_t, c_{\text{not}}) - \hat{\epsilon}_\theta(\mathbf{x}_t, \emptyset)\big)$
        \ENDIF
        
        \STATE \textbf{return} $(\mathbf{g}_{t, c}, \pi_c)$
    \ENDIF
\item[]
    \STATE \textit{// 2. Recursive Composition}
    \STATE \textit{// Negation ($\lnot$)}
    \IF{$\varphi$ is $\neg \psi$}
        \STATE $(\mathbf{g}_{t, \psi}, \pi_\psi) \leftarrow \textsc{GetLogicalScore}(\psi, \mathbf{x}_t)$
        \STATE $\pi_\varphi \leftarrow 1 - \pi_\psi$
        \STATE $\mathbf{g}_{t, \varphi} \leftarrow -\frac{\pi_\psi}{1 - \pi_\psi} \cdot \mathbf{g}_{t, \psi}$
        \STATE \textbf{return} $(\mathbf{g}_{t, \varphi}, \pi_\varphi)$
    \ENDIF
\item[]
    \STATE \textit{// Conjunction ($\land$)}
    \IF{$\varphi$ is $\psi \land \chi$}
        \STATE $(\mathbf{g}_{t, \psi}, \pi_\psi) \leftarrow \textsc{GetLogicalScore}(\psi, \mathbf{x}_t)$
        \STATE $(\mathbf{g}_{t, \chi}, \pi_\chi) \leftarrow \textsc{GetLogicalScore}(\chi, \mathbf{x}_t)$
        \STATE $\pi_\varphi \leftarrow \pi_\psi \cdot \pi_\chi$
        \STATE $\mathbf{g}_{t, \varphi} \leftarrow \mathbf{g}_{t, \psi} + \mathbf{g}_{t, \chi}$
        \STATE \textbf{return} $(\mathbf{g}_{t, \varphi}, \pi_\varphi)$
    \ENDIF
\item[]
    \STATE \textit{// Disjunction ($\lor$)}
    \IF{$\varphi$ is $\psi \lor \chi$}
        \STATE $(\mathbf{g}_{t, \psi}, \pi_\psi) \leftarrow \textsc{GetLogicalScore}(\psi, \mathbf{x}_t)$
        \STATE $(\mathbf{g}_{t, \chi}, \pi_\chi) \leftarrow \textsc{GetLogicalScore}(\chi, \mathbf{x}_t)$
        
        \IF{Mutually Exclusive ($\lor_\text{ME}$)}
            \STATE $\pi_\varphi \gets \pi_\psi + \pi_\chi$
            \STATE $w_\psi \gets \frac{\pi_\psi}{\pi_\varphi + \pi_\chi}, \quad w_\chi \gets \frac{\pi_\chi}{\pi_\varphi + \pi_\chi}$
        \ELSIF{Conditional Independent ($\lor_\text{CI}$)}
            \STATE $\pi_\varphi \gets \pi_\psi + \pi_\chi - \pi_\psi \pi_\chi$
            \STATE $w_\psi \gets \frac{\pi_\psi(1 - \pi_\chi)}{\pi_\varphi + \pi_\chi - \pi_\varphi\pi_\chi}, \quad w_\chi \gets \frac{\pi_\chi(1 - \pi_\psi)}{\pi_\varphi + \pi_\chi - \pi_\varphi\pi_\chi}$
        \ENDIF
        
        \STATE $\mathbf{g}_{t, \varphi} \leftarrow w_\psi \cdot \mathbf{g}_{t, \psi} + w_\chi \cdot \mathbf{g}_{t, \chi}$
        \STATE \textbf{return} $(\mathbf{g}_{t, \varphi}, \pi_\varphi)$
    \ENDIF
\ENDLOOP
\end{algorithmic}
\end{algorithm}

\paragraph{Practical Implementation Details.} 
While \Cref{alg:logical_sampling} and \ref{alg:logical_score} outlines the general logic, our practical implementation introduces specific adjustments for numerical stability. All probability computations are performed in log-space to avoid underflow. Additionally, the probability-dependent task-specific guidance scales are clamped to a maximum value of 3. Finally, the global guidance weight is applied directly at the condition level, scaling each conditioning term individually before aggregation. The output of the unconditional model is cached to avoid repeated evaluation.

\begin{table}[th]
    \centering
    \caption{Hyperparameters for CMNIST, Shapes3D, and CelebA models.}
    \label{tab:hyperparameters}
    \vspace{0.2cm}
    \begin{tabular}{lccc}
        \toprule
        \textbf{Hyperparameter} & \textbf{CMNIST} & \textbf{Shapes3D} & \textbf{CelebA} \\
        \midrule
        \multicolumn{4}{l}{\textbf{Optimization}} \\
        Optimizer & AdamW & AdamW & AdamW \\
        Learning Rate & $1.0 \times 10^{-3}$ & $2.0 \times 10^{-4}$ & $2.0 \times 10^{-4}$ \\
        Weight Decay & $1.0 \times 10^{-5}$ & $0.0$ & -- \\
        LR Scheduler & Cosine w/ Warmup & Cosine w/ Warmup & Cosine w/ Warmup \\
        Warmup Steps & 5000 & 5000 & 2000 \\
        Num Training Steps & 50,000 & 100,000 & 500,000 \\
        \midrule
        \multicolumn{4}{l}{\textbf{Diffusion}} \\
        Noise Scheduler & DDPM & DDPM & DDPM \\
        Beta Schedule & Linear & Linear & Squared Cos Cap v2 \\
        Train Timesteps & 1000 & 1000 & 1000 \\
        Prediction Type & $\epsilon$ (epsilon) & $\epsilon$ (epsilon) & $\epsilon$ (epsilon) \\
        \midrule
        \multicolumn{4}{l}{\textbf{Model Architecture}} \\
        Model Type & U-Net & U-Net & SiT (Transformer) \\
        Input Size & $28 \times 28$ & $64 \times 64$ & $16 \times 16$ (Latent) \\
        Channels / Hidden Dim & [56, 112, 168] & [56, 112, 168, 224] & 384 \\
        Layers & 2 per block & 2 per block & 12 (Depth) \\
        Attention & Head Dim: 8 & Head Dim: 8 & Heads: 6 \\
        Dropout & 0.1 & 0.1 & -- \\
        Norm Groups & 8 & 8 & -- \\
        Activation & GELU & SiLU & -- \\
        Patch Size & -- & -- & 1 \\
        \midrule
        \multicolumn{4}{l}{\textbf{Sampling \& Inference}} \\
        Sampler & DDPM & DDPM & DDIM \\
        Timesteps & 50 & 50 & 1000 \\
        Guidance scale & 1.0 & 1.0 & 1.0 \\
        \bottomrule
    \end{tabular}
\end{table}

\subsection{Diffusion Architecture, Training, Sampling}
\Cref{tab:hyperparameters} details the model architectures and hyperparameters for the CMNIST, Shapes3D, and CelebA experiments. For ImageNet-512, we utilize the pre-trained EDM2-XXL model \cite{karras:2024}, adopting the parameters corresponding to the best reported FID ($g=2.05$, $\text{EMA}_m=0.075$, $\text{EMA}_g=0.155$). We note that our reproduced FID result (with no repulsive guiding) deviate slightly from the original publication, which we attribute to the unavailability of the exact sampling seeds.

\subsection{Classifier from diffusion model}
\label{annex:classifier}
To approximate the probabilities $p(c_i|x)$ we can estimate~\citep{li2023your}  the probabilities or their ratio using the pre-trained diffusion model $\epsilon_t(x,c_i)$, in particular, we can estimate the probability of the class $c_i$ given the current sample $x_t$ as   
\begin{align*}
p(c_i|x_t)  
&= \frac{1}{\scriptstyle 1 + \sum_{j \ne i} \exp  \mathbb E_{\epsilon \sim \mathcal N(0,1)} [\| \epsilon - \epsilon_t(x_{t-1},c_i)  \|^2 - \| \epsilon - \epsilon_t(x_{t-1},c_j)  \|^2] } 
\end{align*}
where the noisy samples are generated either from the current samples $x_{t}$ and the noise $\epsilon \sim \mathcal N(0,1)$, 
\begin{align*}
\bx_{t-1} &= \sqrt{\alpha_t } \bx_{t} + \sqrt{1 - \alpha_t } \epsilon 
\end{align*}
We can also estimate the ratio of the probability and its complement 
\begin{align*}
\gamma(c_i|x_{t}) &= \frac{p(c_i|x_t)}{1-p(c_i|x_t)} 
= \frac{1}{\scriptstyle \sum_{j \ne i} \exp  \mathbb E_{\epsilon \sim \mathcal N(0,1)} [\| \epsilon - \epsilon_t(x_{t-1},c_i)  \|^2 - \| \epsilon - \epsilon_t(x_{t-1},c_j)  \|^2] } 
\end{align*}


\subsection{Unconditional Model from score functions}
\label{annex:unconditional}
If the pre-trained diffusion model is not trained as unconditional model, then we can estimate the score of the uncondtional model from the conditional score. 
Let's suppose that we only have two classes $A,B$, we then marginalize and derive the unconditional distribution 
$$
p(\bx) = p(\bx|A)p(A) + p(\bx|B)p(B)
$$
We can then estimate the unconditional probability score function as
\begin{align*}
\nabla \ln p(\bx) &=  \frac{\nabla p(\bx|A)p(A) + \nabla p(\bx|B)p(B)}{p(\bx|A)p(A) + p(\bx|B)p(B)} \\
&= \frac{p(\bx|A)p(A)}{p(\bx|A)p(A) + p(\bx|B)p(B)}\nabla \ln p(\bx|A) 
+ \frac{p(\bx|B)p(B) }{p(\bx|A)p(A) + p(\bx|B)p(B)}\nabla \ln p(\bx|B) \\
&= \frac{p(A|\bx)p(\bx)}{p(A|\bx)p(\bx) + p(B|\bx)p(\bx)}\nabla \ln p(\bx|A) 
+ \frac{p(\bx|B) p(B)}{p(\bx|A)p(A) + p(\bx|B)p(B)}\nabla \ln p(\bx|B).
\end{align*}
After simplification we have that
\begin{align*}
\nabla \ln p(\bx) &= \frac{p(A|\bx)}{p(A|\bx) + p(B|\bx)}\nabla \ln p(\bx|A) 
+ \frac{p(B|\bx)}{p(A|\bx) + p(B|\bx)} \nabla \ln p(\bx|B),
\end{align*}
If we then defined $\{c_i\} = \{ A,B \}$, we can write the score function in a compact form as 
\begin{align*}
\nabla \ln p(\bx) &= \text{soft-max}_{p(c_i|\bx)}   \nabla \ln p(\bx|c_i)
\end{align*}

If we have a finite number of classes $\{c_i\}$, we can write the unconditional probability as
\begin{align*}
p(\bx) = \sum_i p(\bx|c_i)p(c_i),
\end{align*}
and following the previous derivation, we have that 
\begin{empheq}[box=\widefbox]{align}\label{eq:unconditional}  
\nabla \ln p(\bx) &= \text{soft-max}_{p(c_i|\bx)} \nabla \ln p(\bx|c_i)
\end{empheq}

\newpage
\section{Experiment details}
\subsection{Task Generation} \label{app:generation}
Queries are generated via a recursive algorithm where the complexity parameter, $N$ expressions, defines the total number of binary operators. At each step, an operator is chosen uniformly at random from {$\land$, $\lor_\text{CI}$,$\lor_\text{ME}$}, determining how attributes are distributed to child nodes:
\begin{itemize}
    \item AND / OR-CI: The available attribute groups are partitioned into two disjoint sets. The left and right terms are restricted to mutually exclusive sets of attributes (e.,g., the left queries colors, while the right queries digits).
    \item OR-ME: Both left and right expressions are constrained to share the same attribute group (e.\,g., both query Color).
\end{itemize}
Logical negation is applied probabilistically (p=0.05) to any node.
For CMNIST an example of a query with $N=4$ is: $\Big(\big(\lnot6 \lor_\text{ME} (5 \lor_\text{ME} 3)\big) \lor_\text{CI} (\text{red} \lor_\text{ME}  \text{yellow})\Big)$.

\subsection{Metrics} \label{app:metrics}
\paragraph{Conformity score.}
We quantify generation accuracy using the Conformity Score \cite{gaudi:2025}, defined as the percentage of samples that comply with the logical query. Sample attributes are inferred using a ResNet-18 \cite{he:2016} classifier (per-attribute accuracies reported in \Cref{tab:judge_accuracies}) and matched against the query logic. We report the average conformity score calculated over 10,000 samples, generated from 100 unique queries per task type with 100 samples each.

\begin{table}[h]
    \centering
    \vspace{2mm}
    \caption{Feature-wise accuracy of the classifiers used to calculate Conformity Score and Joint Entropy.}
    \begin{tabular}{llc}
        \toprule
        \textbf{Dataset} & \textbf{Feature} & \textbf{Accuracy (\%)} \\
        \midrule
        \multirow{2}{*}{CMNIST} & Digit & 98.41 \\
                                & Color & 100 \\
        \midrule
        \multirow{6}{*}{Shapes3D} & Floor Hue & 100 \\
                                  & Wall Hue & 100 \\
                                  & Object Hue & 100 \\
                                  & Scale & 100 \\
                                  & Shape & 100 \\
                                  & Orientation & 100 \\
        \midrule
        \multirow{2}{*}{CelebA} & Blond & 98.27 \\
                                & Male & 94.75 \\
        \bottomrule
    \end{tabular}
    \label{tab:judge_accuracies}
\end{table}

\paragraph{Joint entropy (Diversity).}
A key requirement for controlled generation is that the model covers the full scope of valid solutions without mode collapse. For example, given the CMNIST query ($9 \lor 3$), the model should generate nines and threes across all available colors, rather than collapsing to a narrow subset.

To quantify this, we calculate the Mean Batch Joint Entropy. Evaluation is performed per batch ($n=100$ images) where the query is fixed. For every image $i$, we extract the attribute vector using the oracle classifier. We treat these vectors as outcomes of a joint random variable and calculate the Shannon entropy based on the counts of unique attribute combinations within the batch.

Let $C$ be the set of unique attribute combinations observed in the batch, and $p(c)$ be the frequency of a specific combination $c \in C$. The joint entropy is defined as:
\begin{equation*}
    H(Batch) = -\sum_{c \in C} p(c) \log_2 p(c).
\end{equation*}

Since the number of valid solutions varies by query (e.\,g., restrictive queries naturally permit lower diversity), the theoretical optimal entropy is task-dependent. We therefore include the theoretical optimal entropy in Figure \ref{fig:gw_acc_entropy_tradeoff} as a reference baseline to contextualize the model's performance.

\paragraph{FID.}
To assess image fidelity and distributional alignment, we calculate the Fréchet Inception Distance (FID). For CelebA, we compute the FID independently for each compositional task. We generate 5,000 samples per task (comprising 100 images for each of 50 queries) and utilize the clean-fid implementation~\cite{seitzer:2020}. For ImageNet, we adopt the EDM2 framework~\cite{karras:2024}, calculating FID over 50,000 generated samples.

\newpage
\section{Ligand-Protein Multi-target structure-based drug design}
Following  \citet{zhou2024reprogramming,skreta2025feynman}, we first align in $3$D the target protein pockets of the two targets. We then apply the logic guidance of the $SO(3)$-equivariant graph neural network to generate the ligand over $1000$ time steps or denoisy steps. 
For each experiments we generated $32$ samples of size $23$. 
At the end of the generation, we use AutoDock Vina to compute the main metrics. We use $\beta=2.0$, the inverse temperature during the generation. We also generate the single target baseline using the TargetDiff~\citep{guan20233d}. 
In the experiments we used GRM5-RRM1 (with UniProt IDs: P41594 and P23921) protein pairs. 
To evaluate the performance of the generated ligand we considered QED and SA scores. The quality indicator~\cite{lee2025genmol} checks that the drug-likeness (QED) $\ge 0.6$ and the synthetic accessibility (SA)  $\le 4.0$.

\subsection{Interaction profile of representative designed ligands}

To illustrate the molecular basis of selectivity induced by logical compositions, we analyzed the $2$D interaction profiles of representative ligands generated by \method{} (Figure~\ref{fig:ligand_interactions}). The ligands were chosen as follows: (i) For $A \land B$, we selected a ligand where Vina's binding score was the lowest for both targets; (ii) For $A \land \lnot B$, we selected a ligand with the lowest score for target A and the highest score for target B. The reference ligand establishes hydrogen bonds with key residues in both binding sites: Ser590A and Ser488A in GRM5, and Val805A and Tyr315A in RRM1. In the representative example of a ligand generated under the $A \land B$ constraint, this dual interaction pattern is preserved, with the ligand forming contacts with His372A, Leu592A, and Pro133A in GRM5, while simultaneously engaging Thr804A, Glu319A, and Val805A in RRM1. In contrast, the representative ligand generated with the $A \land \lnot B$ constraint exhibits a markedly distinct interaction profile: it maintains multiple hydrophobic and hydrogen bond contacts with GRM5 (Phe602A, Val378A, Thr499A, Gly377A, Ser590A) while showing no detectable interactions with RRM1. These representative examples illustrate that the \method{} framework not only satisfies logical constraints in terms of docking scores but can also generate ligands with chemically distinct interaction profiles, supporting the applicability of logical composition for selective drug design.

\begin{figure}
    \centering
    \includegraphics[width=0.65\columnwidth]{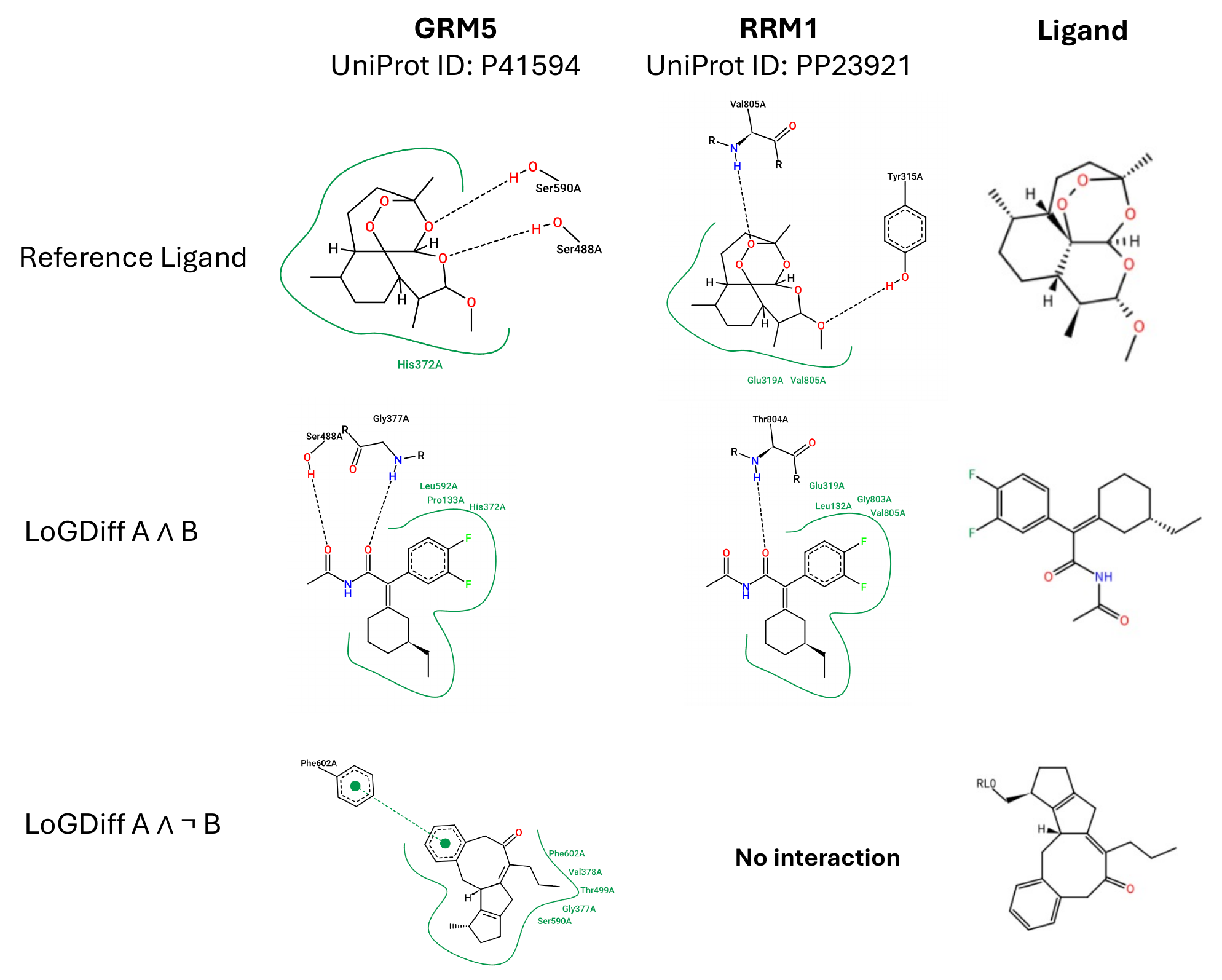}
    \caption{\textbf{2D interaction analysis of representative ligands generated for the GRM5-RRM1 target pair.} (Top) Reference ligand showing hydrogen bonds with key residues in both targets. (Middle) Representative ligand generated under the conjunction constraint ($A \land B$), displaying hydrogen bonds and hydrophobic contacts with residues from both GRM5 and RRM1. (Bottom) Representative ligand generated under the selective constraint ($A \land \lnot B$), maintaining multiple interactions with GRM5 while showing no detectable interactions with RRM1. Green contours indicate hydrophobic regions; black dashed lines represent hydrogen bonds, whereas green dashed lines represenr $\pi$-$\pi$ interactions. Images were rendered with ProteinPlus.}
    \label{fig:ligand_interactions}
\end{figure}

\newpage
\section{Additional results} \label{app:additional_results}
\subsection{Comparison with additional methods} \label{sec:appendix_baselines}

We compare our approach against existing frameworks that define logical composition for diffusion models using linear combinations of scores \cite{liu:2022, skreta:2025}. 
This includes methods that perform composition via a single model trained with conditions and on the whole data distribution (i.\,e. unconditional model) and conditional models in a CFG-style, as well as methods designed for the composition of multiple independent pretrained models. 
Due to our experimental design, we apply also the latter methods to a single model conditioned on varying conditions rather than training multiple separate models.

We note a structural limitation in existing literature: many current methods lack a formal definition for the recursive construction of arbitrarily complex logical formulas for conjunctions, disjunctions and negations. Therefore, we restrict our experiment to single-operand expressions.

We report the comparative results in Table \ref{tab:baselines}. To establish a lower bound, we include an unconditional baseline, which represents the probability of satisfying logical constraints by chance.
For conjunctions, we observe that simply averaging conditional scores or re-weighting individual scores to achieve equal density \cite{skreta:2025} performs considerably worse than the method proposed by us and \citep{liu:2022} using posterior scores. In the case of disjunctions, the conformity scores improve notably when utilizing probability-dependent weights. Both our method and the re-weighting scheme proposed by \citet{skreta:2025} achieve comparable results.
For negation, the constant baseline achieves slightly higher conformity scores than \method{}, though this comes at a the cost to diversity and perceptual quality (see Section \ref{sec:img_experiments}). The EBM-style negation $p(x)/p_A(x)^\gamma$ \cite{du:2023}, requires exhaustive optimization of the hyperparameter $\gamma$. Even after tuning the guidance weight to maximize CS while maintaining sample validity, performance remains only marginally better than the unconditional baseline.

\begin{table*}[]
    \centering
    \caption{\textbf{Logical composition conformity across different frameworks on CMNIST}. We compare our method against logical composition frameworks for diffusion models using linear combinations of scores \cite{liu:2022, du:2023, skreta:2025} using Conformity Scores (CS).}
\begin{tabular}{llllr}
\toprule
Task & use $\epsilon_\emptyset$ & Probabilites & Method & CS  \\
\midrule
AND & Yes & - & $\epsilon_\emptyset$ only  & 0.91  \\
 \cmidrule{2-5}
 & Yes & classifier & $\epsilon_\emptyset + (\epsilon_A - \epsilon_\emptyset) + (\epsilon_B - \epsilon_\emptyset)$ (\method{}, \cite{liu:2022}) & \textbf{80.39}  \\
 & No & - & $\frac{1}{2} \epsilon_A + \frac{1}{2} \epsilon_B$ \cite{liu:2022} & 26.61  \\
 & No & It\^o estimator & equal prob. update \cite{skreta:2025} & 26.20  \\
\hline
 NOT&  &  & $\epsilon_\emptyset$ only & 90.19  \\
 \cmidrule{2-5}
 & Yes & - &  $\epsilon_\emptyset - (\epsilon_A - \epsilon_\emptyset)$ \cite{liu:2022} & \textbf{99.41}\\
 & Yes & classifier & $\epsilon_\emptyset-\frac{p_A}{1-p_A} \, (\epsilon_A - \epsilon_\emptyset)$ (LoGDiff)& 97.38 \\
 & Yes & classifier &  $ \epsilon_\emptyset - w \cdot \epsilon_A$ (with $w=0.07$) \cite{du:2023}& 92.60 \\
\hline
OR-ME  &  &  & $\epsilon_\emptyset$ only & 19.22  \\
 \cmidrule{2-5}
 & Yes & - &  $\epsilon_\emptyset + \frac{1}{2} (\epsilon_A - \epsilon_\emptyset) + \frac{1}{2} (\epsilon_B - \epsilon_\emptyset)$ & 76.95  \\
 & Yes & classifier & $\epsilon_\emptyset +\frac{p_A}{p_A + p_B} (\epsilon_A - \epsilon_\emptyset) + \frac{p_B}{p_A + p_B} (\epsilon_B - \epsilon_\emptyset)$  (\method{})  & \textbf{98.01}  \\
 & No & classifier & $\frac{p_A}{p_A + p_B} \epsilon_A + \frac{p_B}{p_A + p_B} \epsilon_B$  \cite{skreta:2025} & 97.98 \\
 & No & It\^o estimator & $\frac{p_A}{p_A + p_B} \epsilon_A + \frac{p_B}{p_A + p_B} \epsilon_B$ \cite{skreta:2025} & 96.84 \\
\hline
OR-CI &  &  & $\epsilon_\emptyset$ only & 19.09 \\
 \cmidrule{2-5}
 & Yes & - &  $\epsilon_\emptyset + \frac{1}{2} (\epsilon_A - \epsilon_\emptyset) + \frac{1}{2} (\epsilon_B - \epsilon_\emptyset)$ & 76.72 \\
 & Yes & classifier & $\epsilon_\emptyset +\frac{p_A(1-p_B)(\epsilon_A - \epsilon_\emptyset) + p_B(1-p_A) (\epsilon_B - \epsilon_\emptyset)}{p_A + p_B -p_Ap_B} $ (\method{}) & 97.25 \\
 & No & classifier & $\frac{p_A}{p_A + p_B} \epsilon_A + \frac{p_B}{p_A + p_B} \epsilon_B$ \cite{skreta:2025} & 97.34 \\
 & No & It\^o estimator & $\frac{p_A}{p_A + p_B} \epsilon_A + \frac{p_B}{p_A + p_B} \epsilon_B$ \cite{skreta:2025} & \textbf{97.54} \\
\bottomrule
\end{tabular}

    \label{tab:baselines}
\end{table*}

\subsection{Qualitative effect of guidance weight}
Figures \ref{fig:gw_cmnist}, \ref{fig:gw_shapes} and \ref{fig:gw_cmnist_not} qualitatively illustrates the trade-off between conformity and diversity as the guidance weight $w$ increases which is quantitatively reported in Figure \ref{fig:gw_acc_entropy_tradeoff}. While higher guidance scales improve conformity scores, they visibly reduce sample diversity. Notably, the constant baseline struggles with disjunctions of conditional independent attributes, often collapsing them into conjunctions (Fig. \ref{fig:gw_cmnist}). For mutually exclusive attributes, the baseline attempts to blend features. This failure mode is particularly evident when the mixture resembles a valid but incorrect condition (e.\,g., the conditions red or yellow produce orange, as shown in Figure \ref{fig:gw_shapes}). The reduced diversity is also visible for the negation, when the constant baseline is used (see Figure \ref{fig:gw_cmnist_not}). 

\begin{figure}[h]
    \centering
    \includegraphics[width=0.8\linewidth]{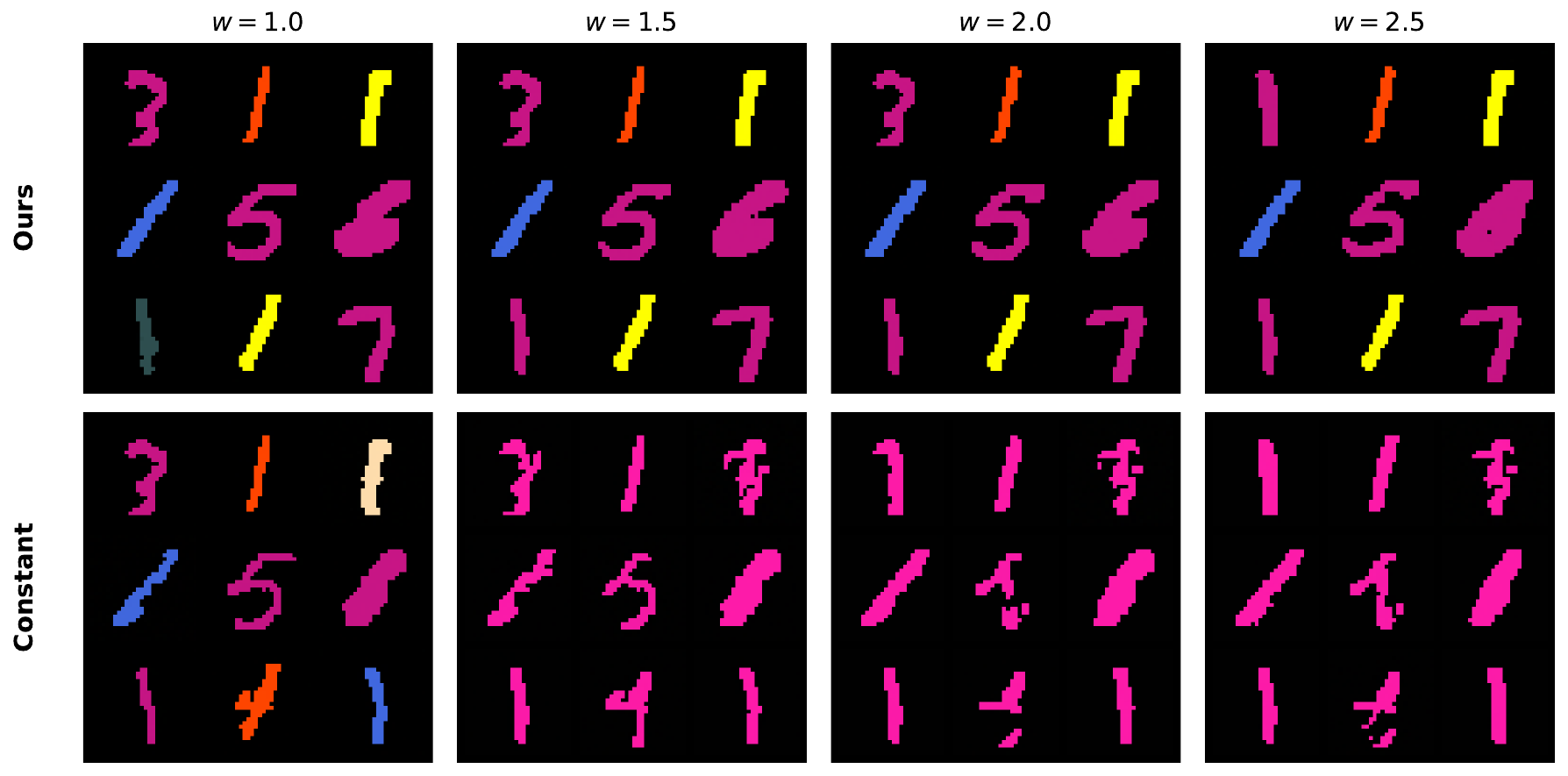}
    \caption{Influence of guidance weight $w$ on conditional independent disjunction (query: 1 $\lor_\text{CI}$ pink). Higher guidance weights result in visibly lower diversity for the constant baseline as it converges to conjunction behavior.}
    \label{fig:gw_cmnist}
\end{figure}

\begin{figure}[h]
    \centering
    \includegraphics[width=0.8\linewidth]{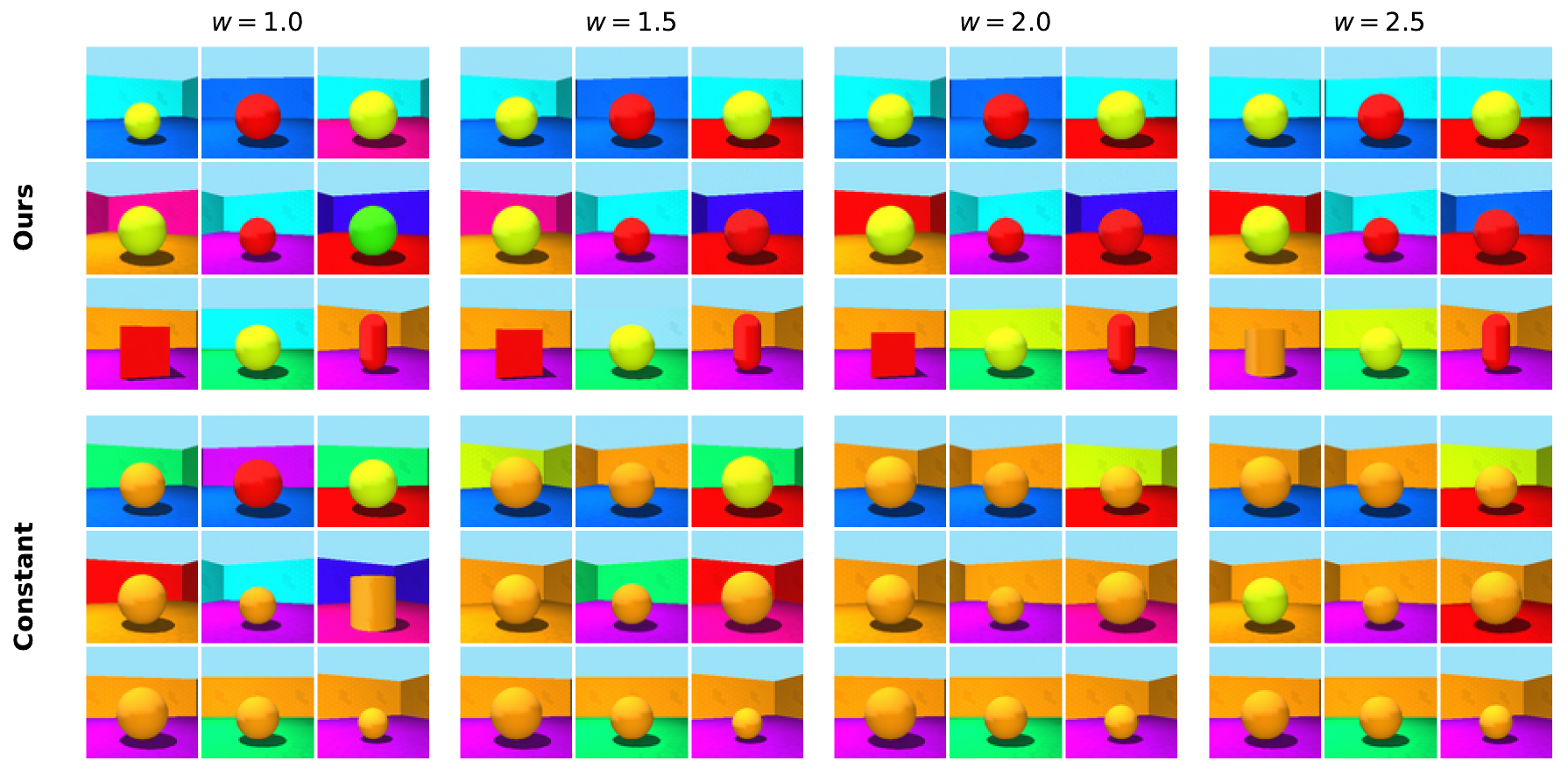}
    \caption{Influence of guidance weight $w$ on mutual exclusive disjunction (query: red object $\lor_\text{ME}$ yellow object). Higher guidance weights result in visibly lower diversity for the constant baseline as it mixes the attributes.}
    \label{fig:gw_shapes}
\end{figure}

\begin{figure}[h]
    \centering
    \includegraphics[width=0.8\linewidth]{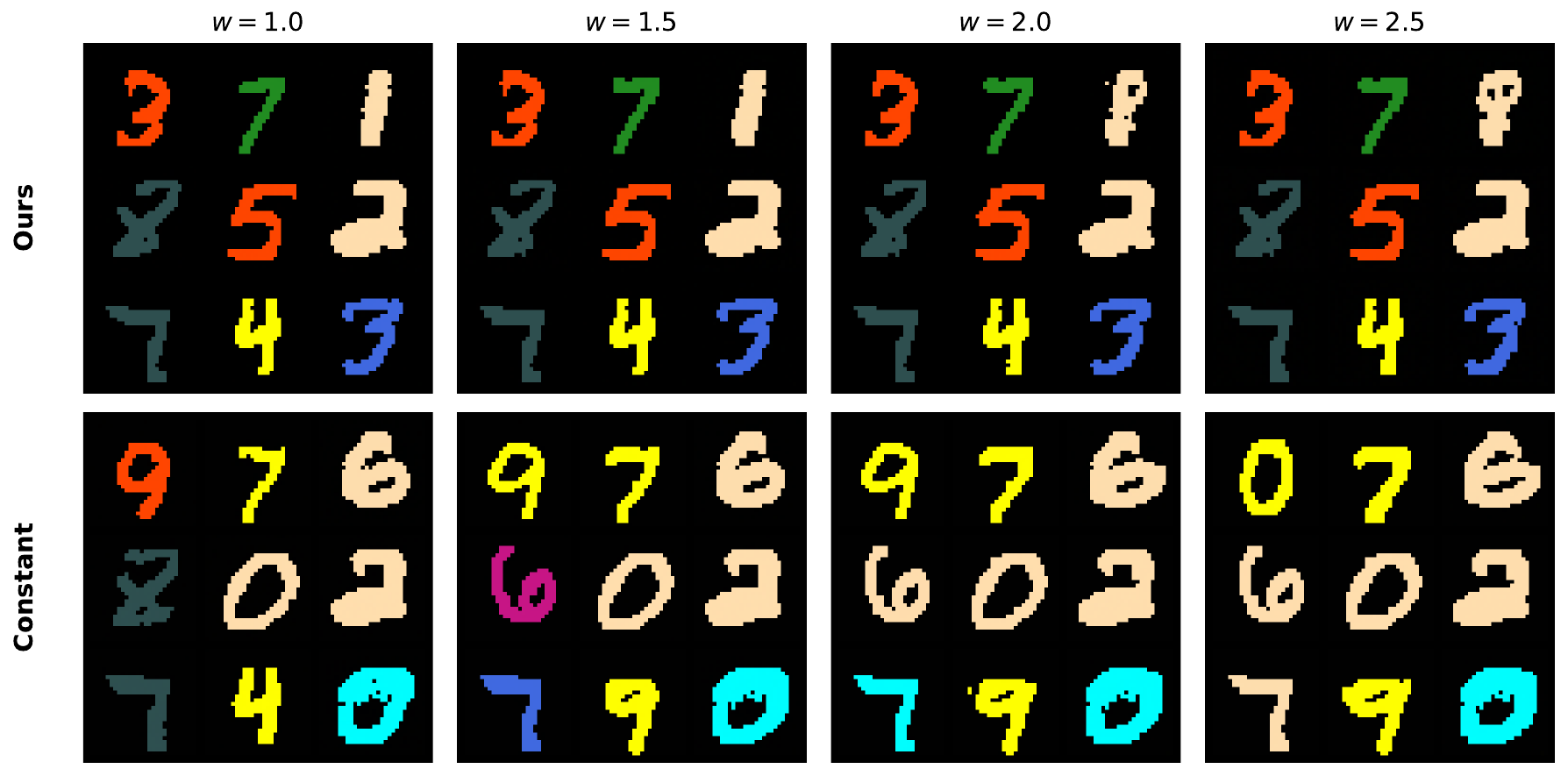}
    \caption{Influence of guidance weight $w$ on negation (query: $\lnot$ 1). Higher guidance scales for the negation reduce the diversity of the constant baseline, while increasing the conformity score of our approach without harming diversity.}
    \label{fig:gw_cmnist_not}
\end{figure}
\subsection{Qualitative results of logical compositions on images}
We present qualitative results across various datasets and logical composition tasks, comparing our approach against the constant baseline. Each figure displays a batch of samples for a given query, generated with a guidance weight of $w=1$. Consistent with the quantitative results in Table \ref{tab:cs_cmnist_shapes3d}, our method performs identically to the baseline for conjunctions. While the constant baseline achieves higher conformity scores on negation tasks, this comes at the cost of reduced diversity. Furthermore, the baseline struggles with disjunctions and complex logical statements.

\begin{figure}[ht]
    \centering
    \begin{minipage}{0.9\linewidth}
    \setkeys{Gin}{width=\linewidth} 

    \begin{minipage}{0.99\linewidth}
        \scriptsize \centering \textbf{\large \textcolor{cmnistLightBlue}{light blue} $\land$ 3}
    \end{minipage}%

    \vspace{1pt}
    \begin{minipage}{0.49\linewidth}
        \scriptsize \centering \textbf{Constant}
    \end{minipage}%
    \hfill
    \begin{minipage}{0.49\linewidth}
        \scriptsize \centering \textbf{Ours}
    \end{minipage}
    

    \begin{subfigure}{0.49\linewidth}
        \includegraphics[trim={0 0 0 1.5cm}, clip, width=\linewidth]{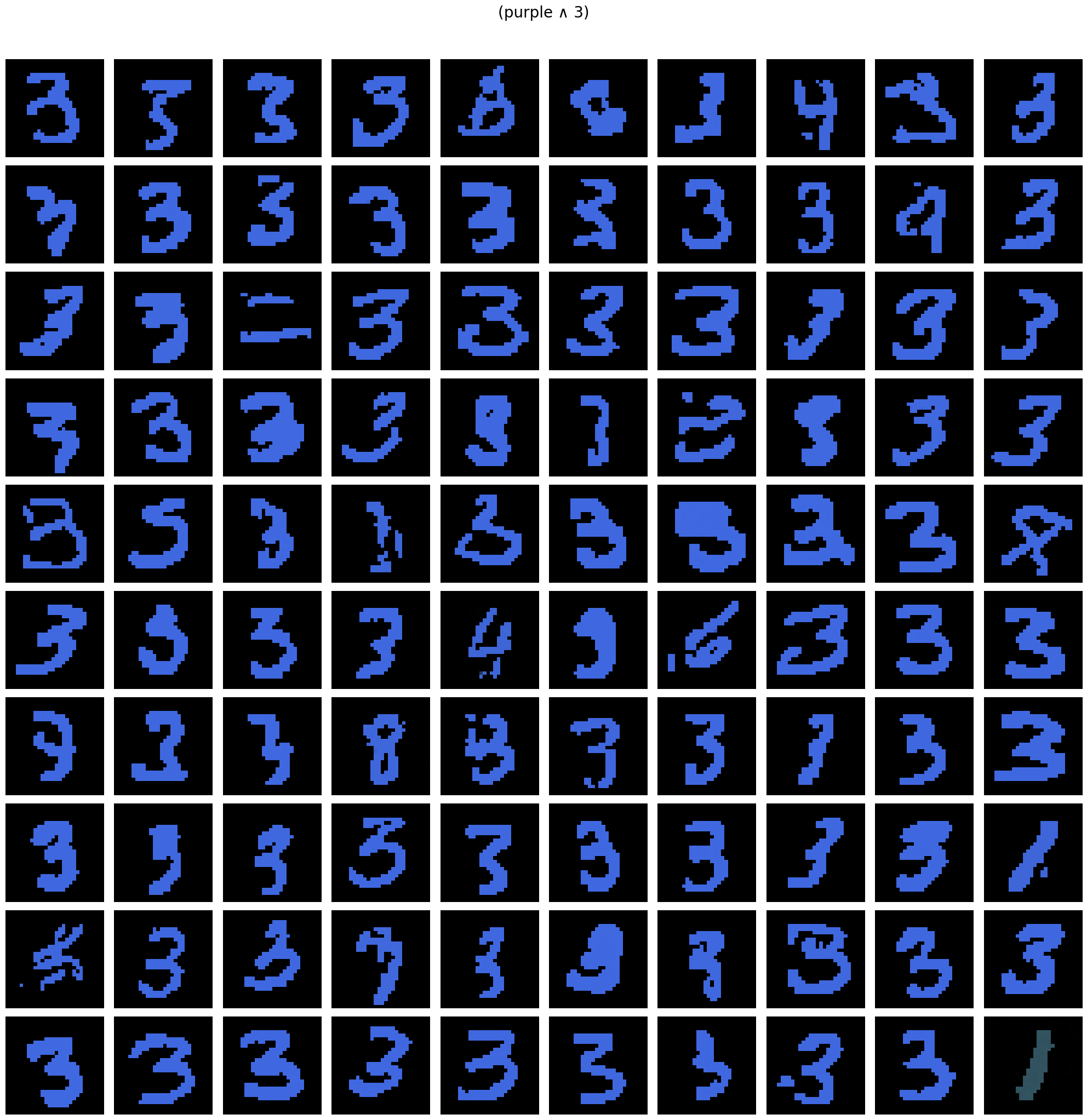}
    \end{subfigure}%
    \hfill
    \begin{subfigure}{0.49\linewidth}
        \includegraphics[trim={0 0 0 1.5cm}, clip, width=\linewidth]{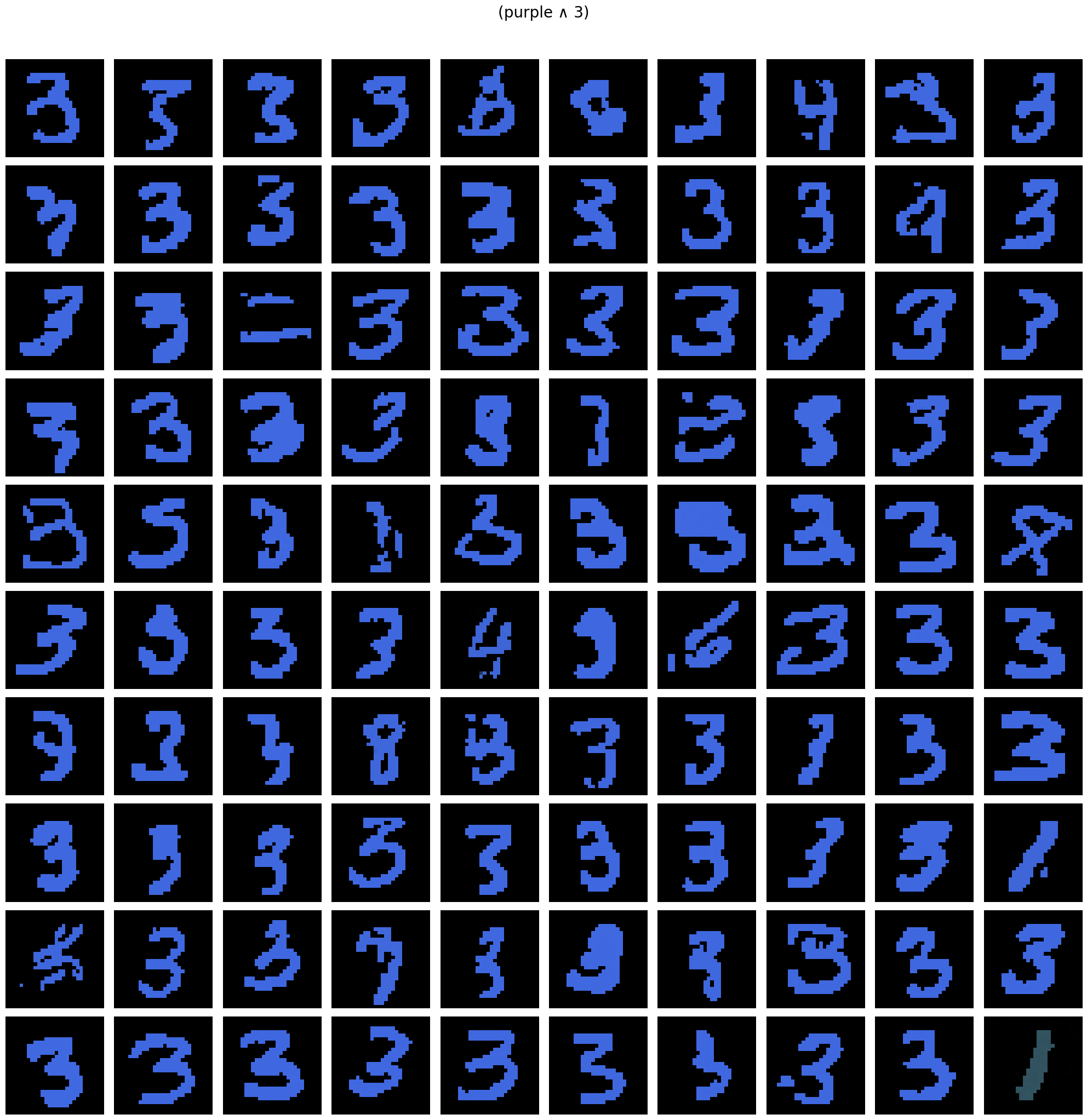}
    \end{subfigure}
    
    \caption{Batch of results for the AND composition (light blue $\land$ 3) on CMNIST. Note that for this single AND operation, the constant baseline and our approach coincide. Conformity scores: Constant = 0.79, Ours = 0.79.}
    \end{minipage}
\end{figure}

\begin{figure}[h!]
    \centering
    \begin{minipage}{0.9\linewidth}
    \setkeys{Gin}{width=\linewidth} 

    \begin{minipage}{0.99\linewidth}
        \scriptsize \centering \textbf{\large $\lnot$ \textcolor{cmnistBlue}{blue}}
    \end{minipage}%

    \vspace{1pt}
    \begin{minipage}{0.49\linewidth}
        \scriptsize \centering \textbf{Constant}
    \end{minipage}%
    \hfill
    \begin{minipage}{0.49\linewidth}
        \scriptsize \centering \textbf{Ours}
    \end{minipage}
    

    \begin{subfigure}{0.49\linewidth}
        \includegraphics[trim={0 0 0 1.5cm}, clip, width=\linewidth]{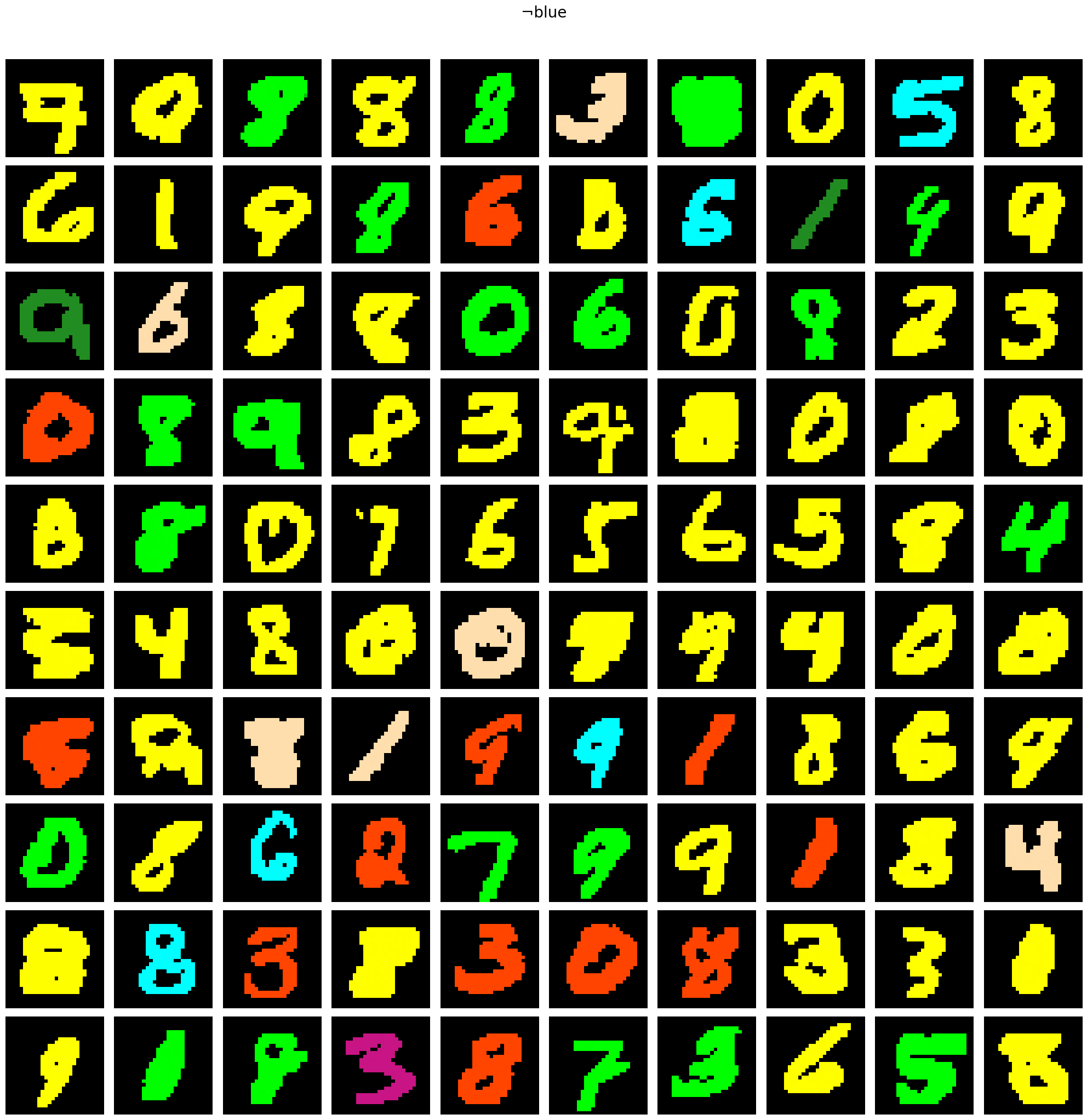}
    \end{subfigure}%
    \hfill
    \begin{subfigure}{0.49\linewidth}
        \includegraphics[trim={0 0 0 1.5cm}, clip, width=\linewidth]{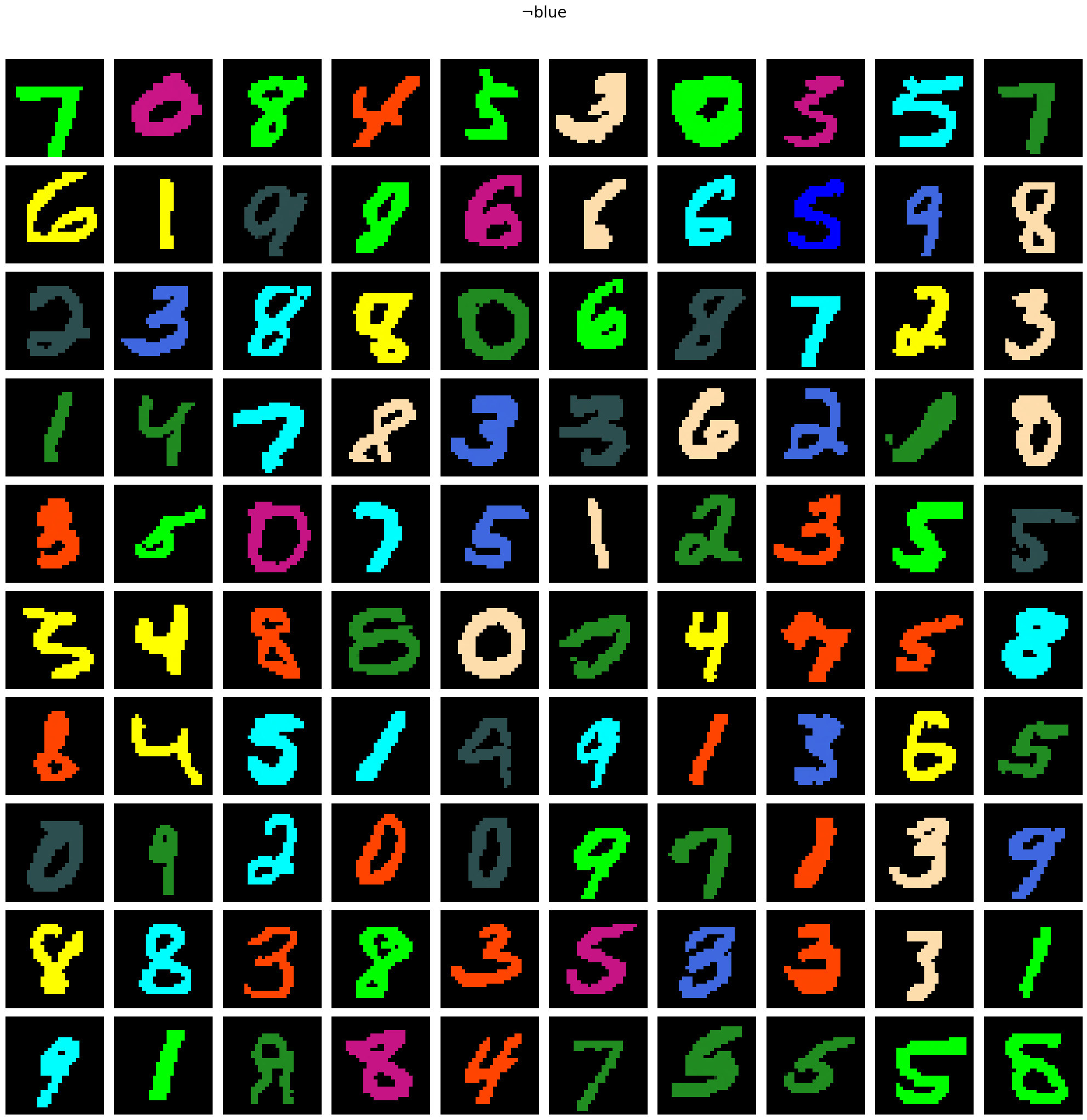}
    \end{subfigure}
    
    \caption{Batch of results for the NOT composition ($\lnot$ blue) on CMNIST. Conformity scores: Constant = 1.0, Ours = 0.99.}
    \end{minipage}
\end{figure}

\begin{figure}[h!]
    \centering
    \begin{minipage}{0.9\linewidth}
    \setkeys{Gin}{width=\linewidth} 

    \begin{minipage}{0.99\linewidth}
        \scriptsize \centering \textbf{\large \textcolor{cmnistRed}{red} $\lor$ \textcolor{cmnistBeige}{beige}}
    \end{minipage}%

    \vspace{1pt}
    \begin{minipage}{0.49\linewidth}
        \scriptsize \centering \textbf{Constant}
    \end{minipage}%
    \hfill
    \begin{minipage}{0.49\linewidth}
        \scriptsize \centering \textbf{Ours}
    \end{minipage}
    

    \begin{subfigure}{0.49\linewidth}
        \includegraphics[trim={0 0 0 1.5cm}, clip, width=\linewidth]{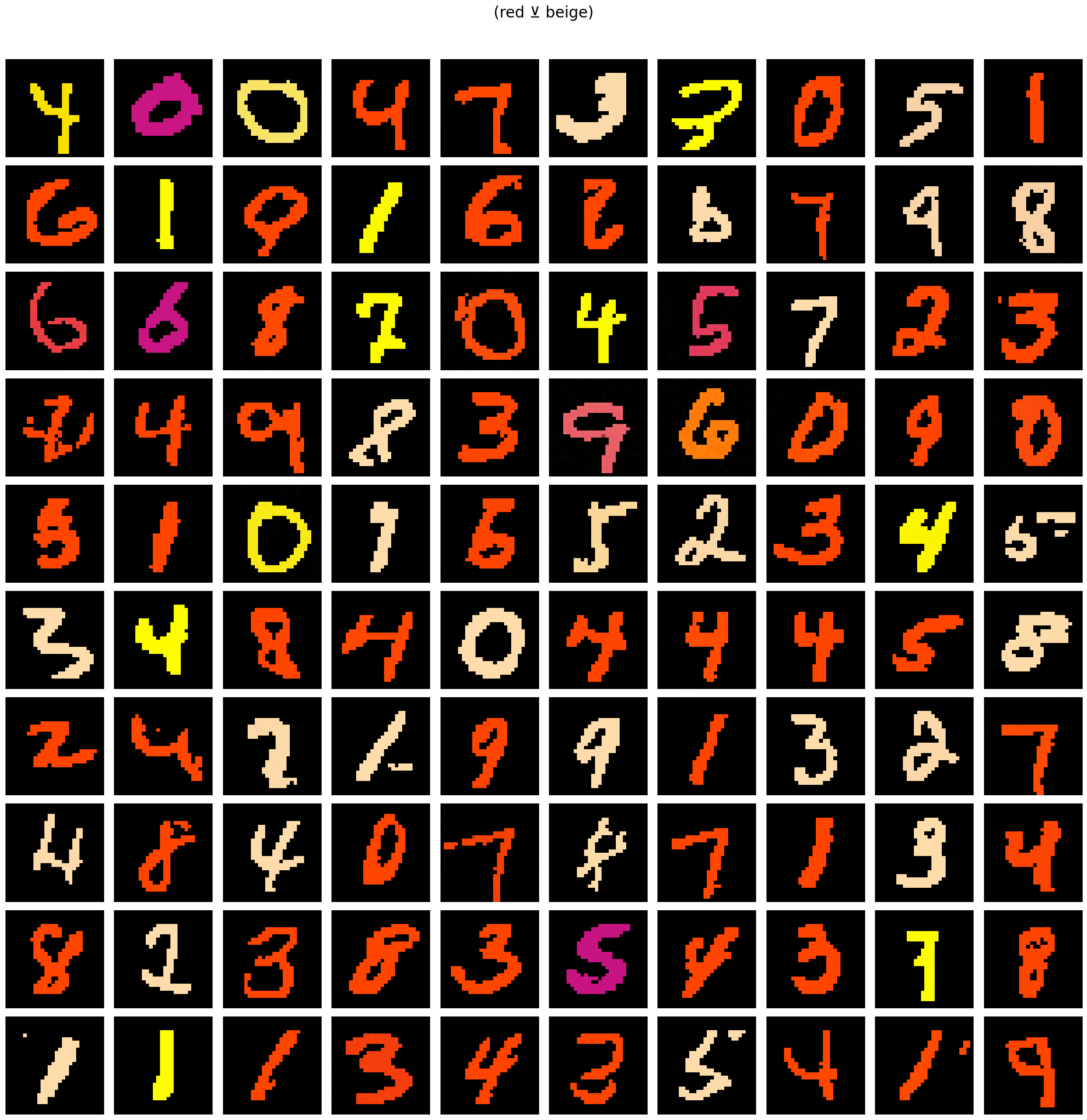}
    \end{subfigure}%
    \hfill
    \begin{subfigure}{0.49\linewidth}
        \includegraphics[trim={0 0 0 1.5cm}, clip, width=\linewidth]{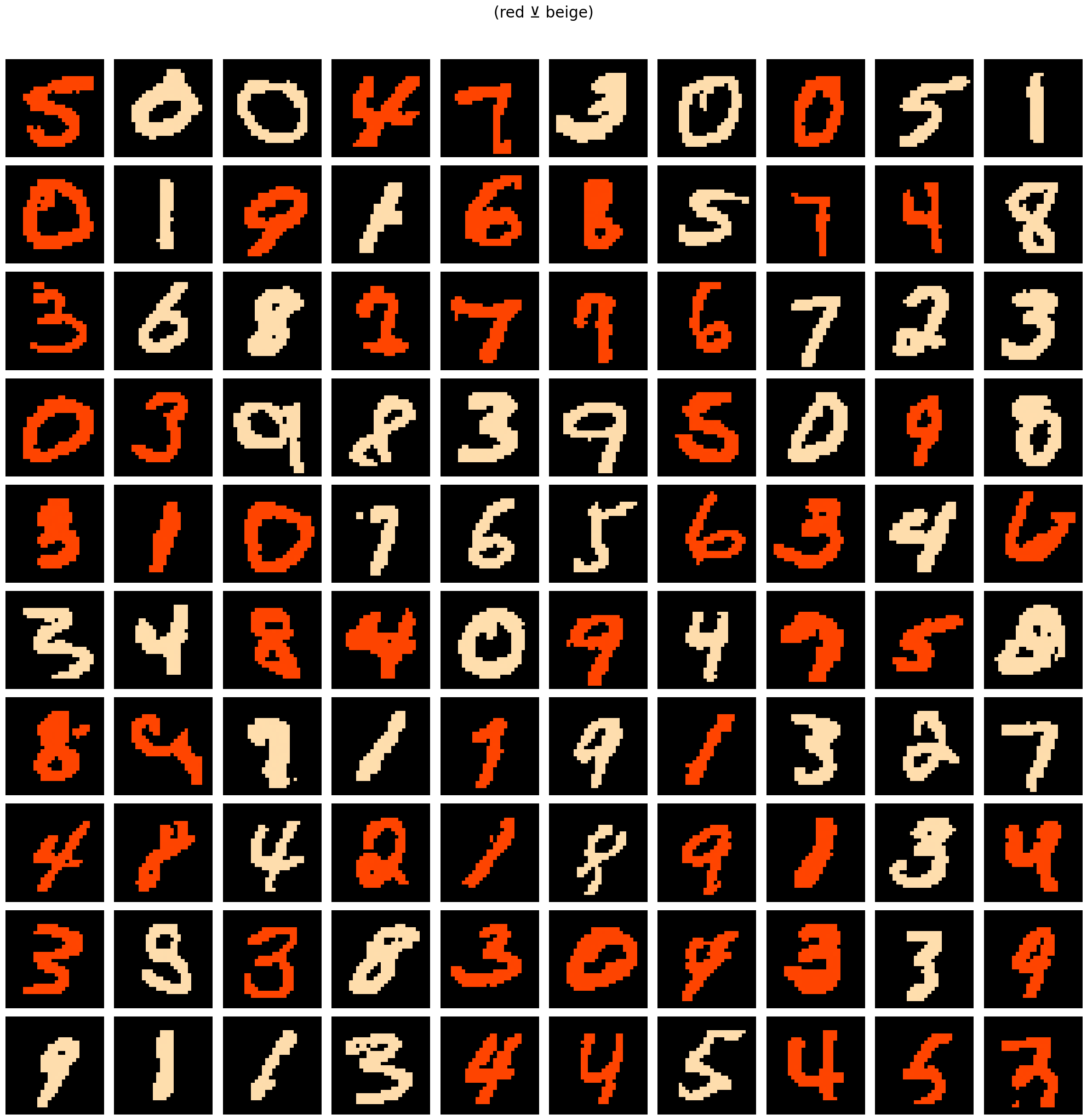}
    \end{subfigure}
    
    \caption{Batch of results for the OR-ME composition (red $\lor$ beige) on CMNIST.Conformity scores: Constant = 0.84, Ours = 1.0.}
    \end{minipage}
\end{figure}

\begin{figure}[h!]
    \centering
    \begin{minipage}{0.9\linewidth}
    \setkeys{Gin}{width=\linewidth}

    \begin{minipage}{0.99\linewidth}
        \scriptsize \centering \textbf{\large 2 $\lor$ \textcolor{cmnistBlue}{blue}}
    \end{minipage}%

    \vspace{1pt}
    \begin{minipage}{0.49\linewidth}
        \scriptsize \centering \textbf{Constant}
    \end{minipage}%
    \hfill
    \begin{minipage}{0.49\linewidth}
        \scriptsize \centering \textbf{Ours}
    \end{minipage}
    

    \begin{subfigure}{0.49\linewidth}
        \includegraphics[trim={0 0 0 1.5cm}, clip, width=\linewidth]{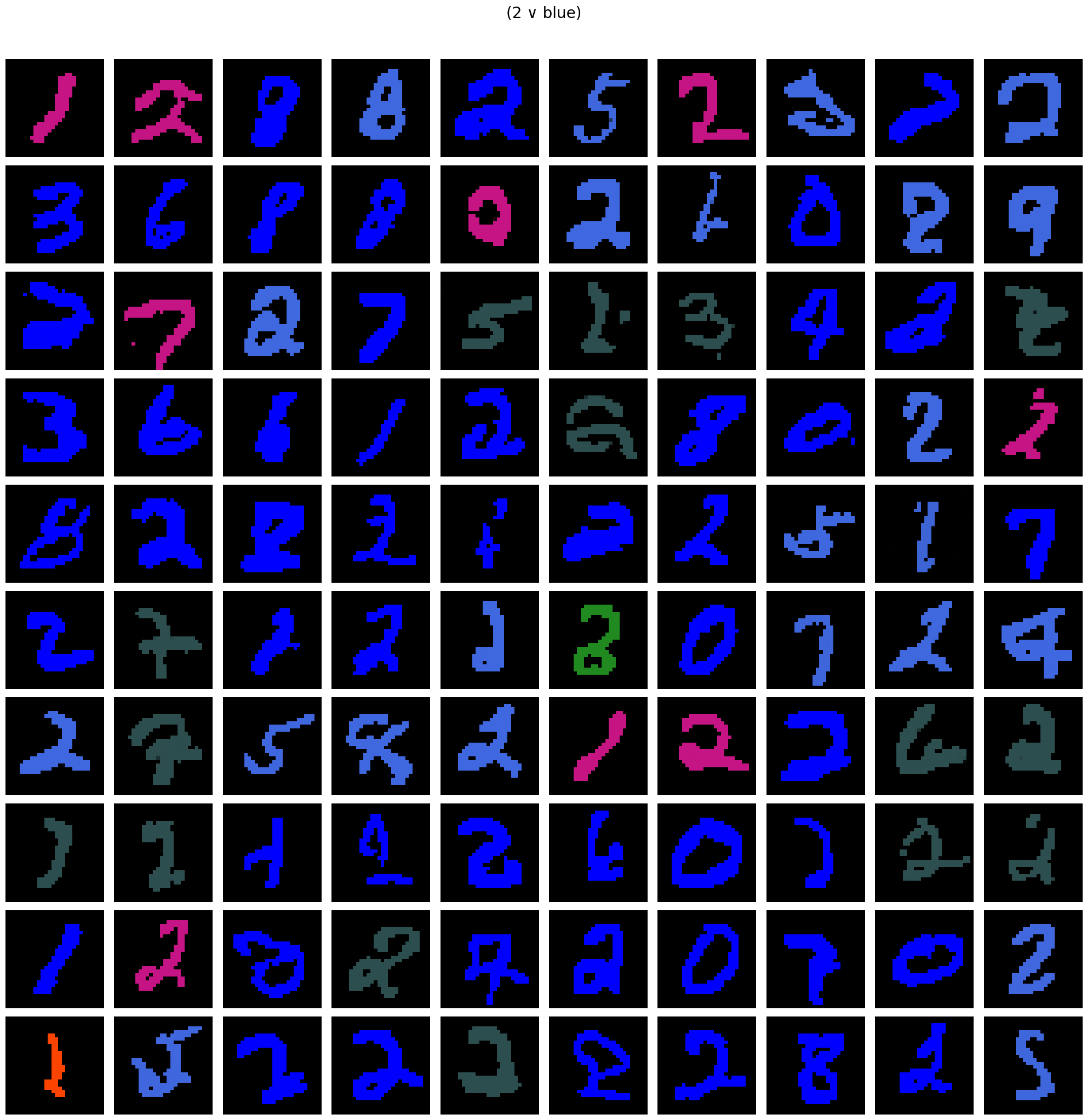}
    \end{subfigure}%
    \hfill
    \begin{subfigure}{0.49\linewidth}
        \includegraphics[trim={0 0 0 1.5cm}, clip, width=\linewidth]{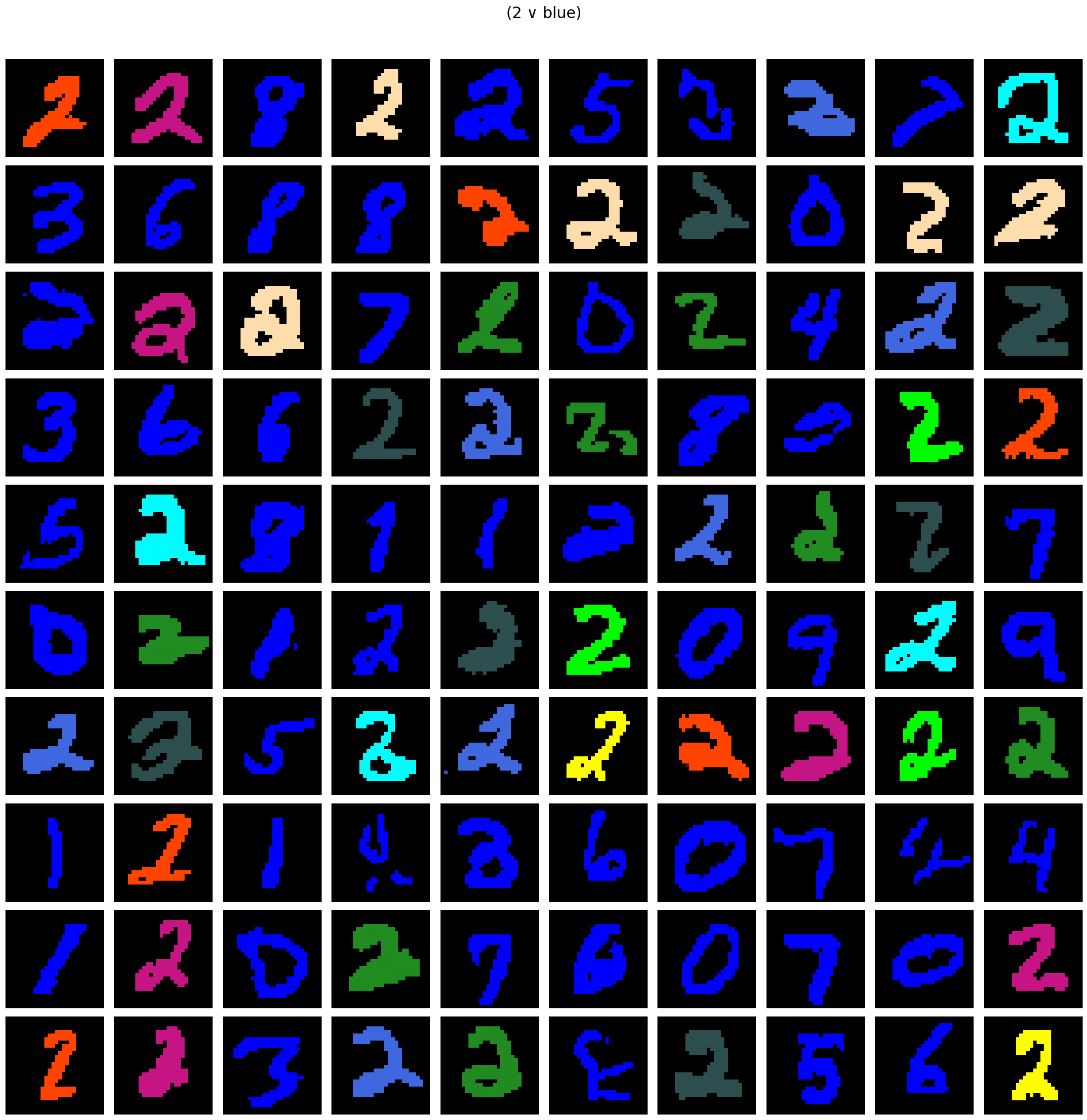}
    \end{subfigure}
    
    \caption{Batch of results for the OR-CI composition (2 $\lor$ blue) on CMNIST. Conformity scores: Constant = 0.78, Ours = 0.98.}
    \end{minipage}
\end{figure}

\begin{figure}[h!]
    \centering
    \begin{minipage}{0.9\linewidth}
    \setkeys{Gin}{width=\linewidth} 

    \begin{minipage}{0.99\linewidth}
        \scriptsize \centering \textbf{\large \textcolor{shapesLightGreen}{light green} floor $\land$ \textcolor{shapesBlue}{blue} wall}
    \end{minipage}%

    \vspace{1pt}
    \begin{minipage}{0.49\linewidth}
        \scriptsize \centering \textbf{Constant}
    \end{minipage}%
    \hfill
    \begin{minipage}{0.49\linewidth}
        \scriptsize \centering \textbf{Ours}
    \end{minipage}
    

    \begin{subfigure}{0.49\linewidth}
        \includegraphics[trim={0 0 0 1.5cm}, clip, width=\linewidth]{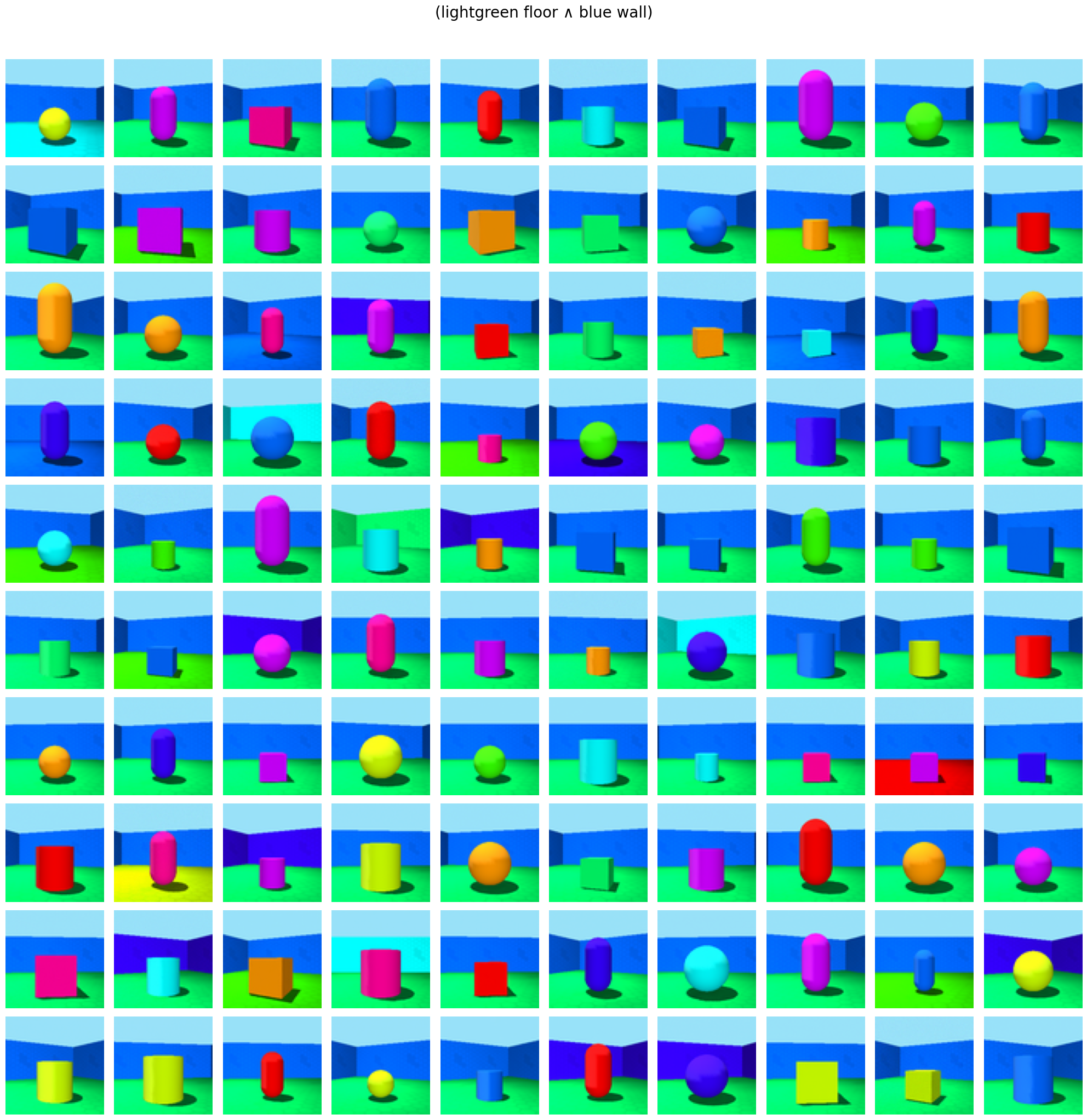}
    \end{subfigure}%
    \hfill
    \begin{subfigure}{0.49\linewidth}
        \includegraphics[trim={0 0 0 1.5cm}, clip, width=\linewidth]{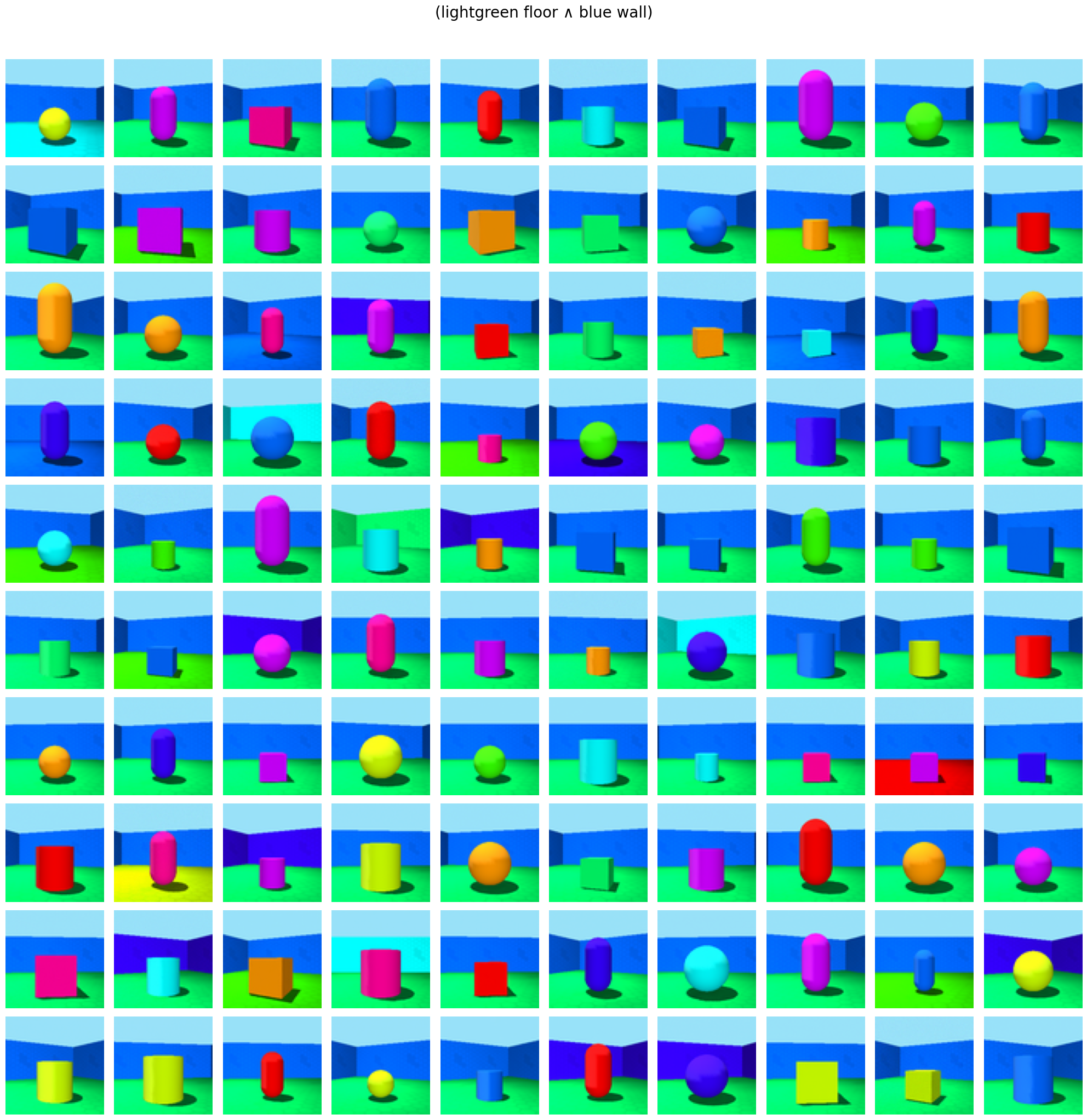}
    \end{subfigure}
    
    \caption{Batch of results for the AND composition (lightgreen floor $\land$ blue wall) on Shapes3D. Conformity scores: Constant = 0.74, Ours = 0.74.}
    \end{minipage}
\end{figure}

\begin{figure}[h!]
    \centering
    \begin{minipage}{0.9\linewidth}
    \setkeys{Gin}{width=\linewidth} 

    \begin{minipage}{0.99\linewidth}
        \scriptsize \centering \textbf{\large $\lnot$ cylinder}
    \end{minipage}%

    \vspace{1pt}
    \begin{minipage}{0.49\linewidth}
        \scriptsize \centering \textbf{Constant}
    \end{minipage}%
    \hfill
    \begin{minipage}{0.49\linewidth}
        \scriptsize \centering \textbf{Ours}
    \end{minipage}
    

    \begin{subfigure}{0.49\linewidth}
        \includegraphics[trim={0 0 0 1.5cm}, clip, width=\linewidth]{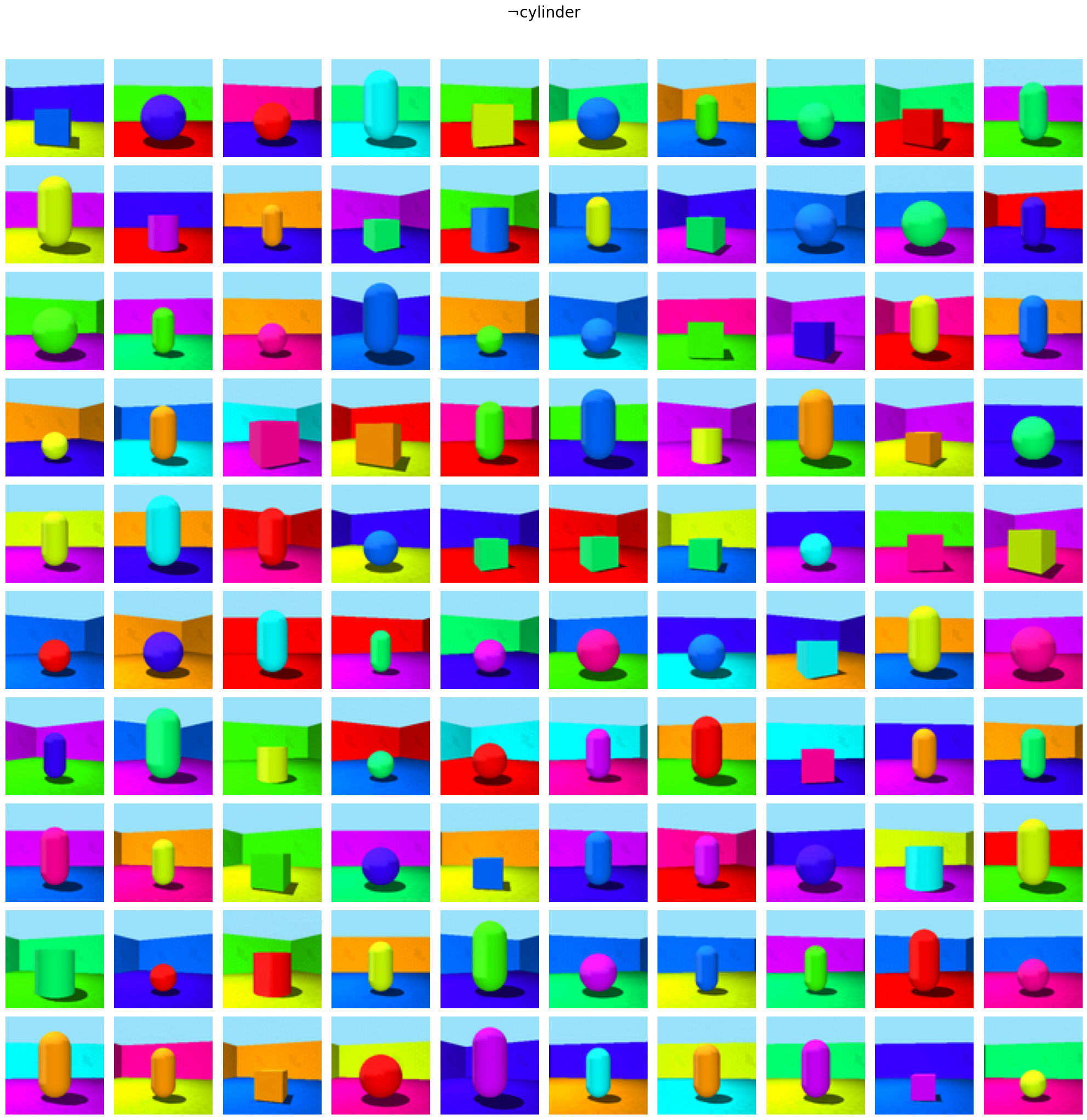}
    \end{subfigure}%
    \hfill
    \begin{subfigure}{0.49\linewidth}
        \includegraphics[trim={0 0 0 1.5cm}, clip, width=\linewidth]{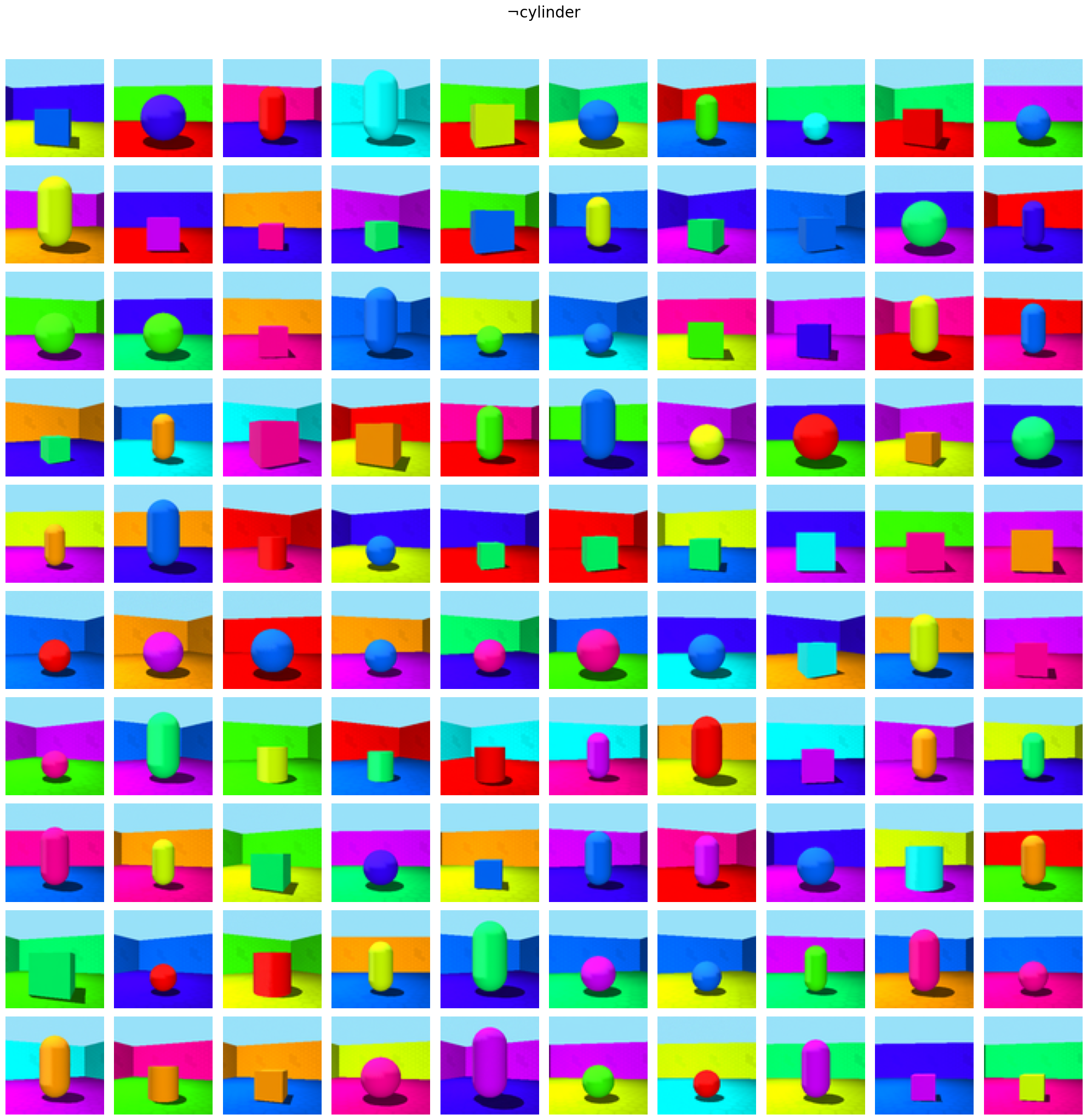}
    \end{subfigure}
    
    \caption{Batch of results for the NOT composition ($\lnot$ cylinder) on Shapes3D. Conformity scores: Constant = 0.93, Ours = 0.93.}
    \end{minipage}
\end{figure}

\begin{figure}[h!]
    \centering
    \begin{minipage}{0.9\linewidth}
    \setkeys{Gin}{width=\linewidth} 

    \begin{minipage}{0.99\linewidth}
        \scriptsize \centering \textbf{\large \textcolor{shapesLightGreen}{light green} floor $\lor$ \textcolor{shapesPink}{pink} object}
    \end{minipage}%

    \vspace{1pt}
    \begin{minipage}{0.49\linewidth}
        \scriptsize \centering \textbf{Constant}
    \end{minipage}%
    \hfill
    \begin{minipage}{0.49\linewidth}
        \scriptsize \centering \textbf{Ours}
    \end{minipage}
    

    \begin{subfigure}{0.49\linewidth}
        \includegraphics[trim={0 0 0 1.5cm}, clip, width=\linewidth]{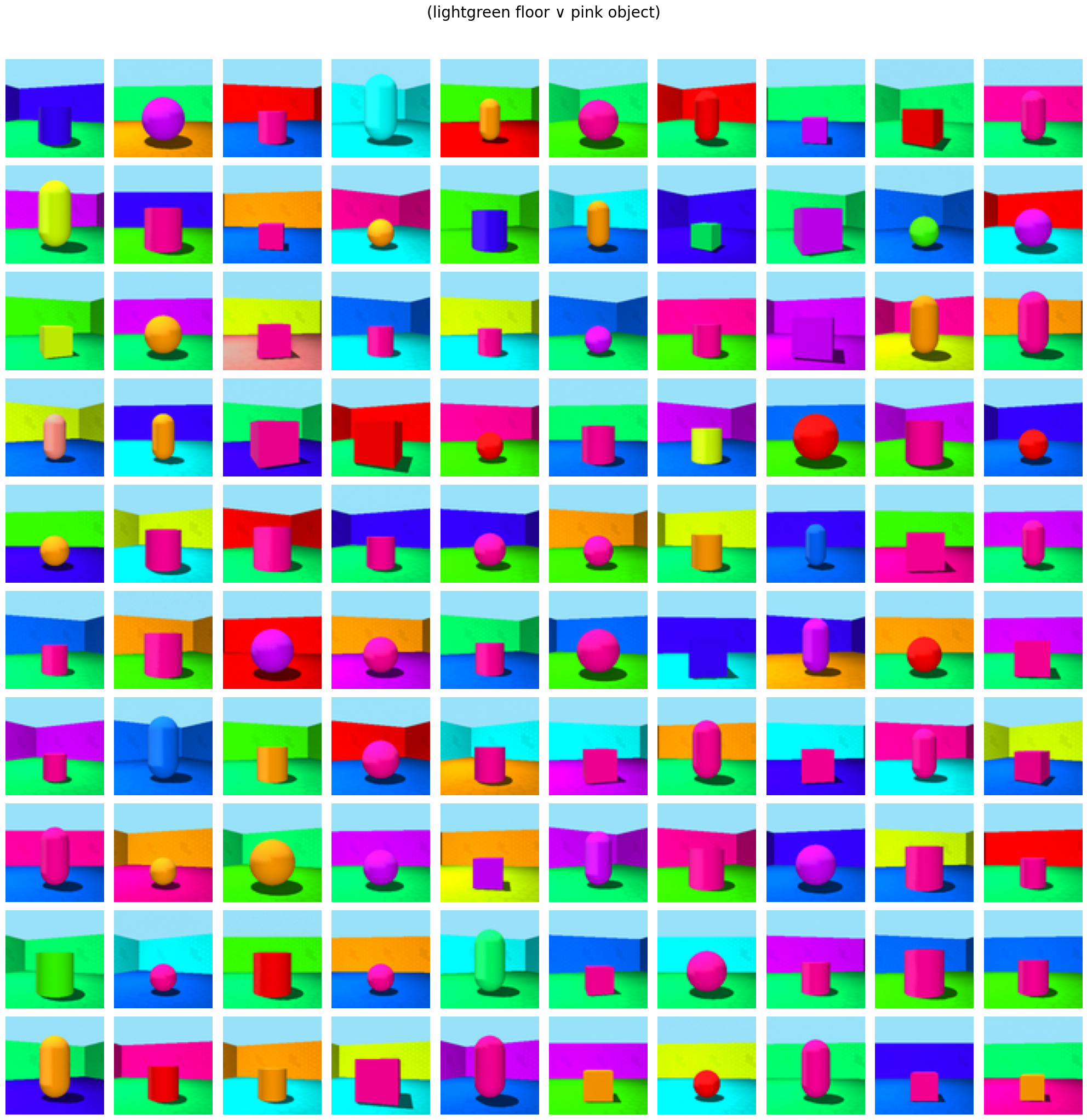}
    \end{subfigure}%
    \hfill
    \begin{subfigure}{0.49\linewidth}
        \includegraphics[trim={0 0 0 1.5cm}, clip, width=\linewidth]{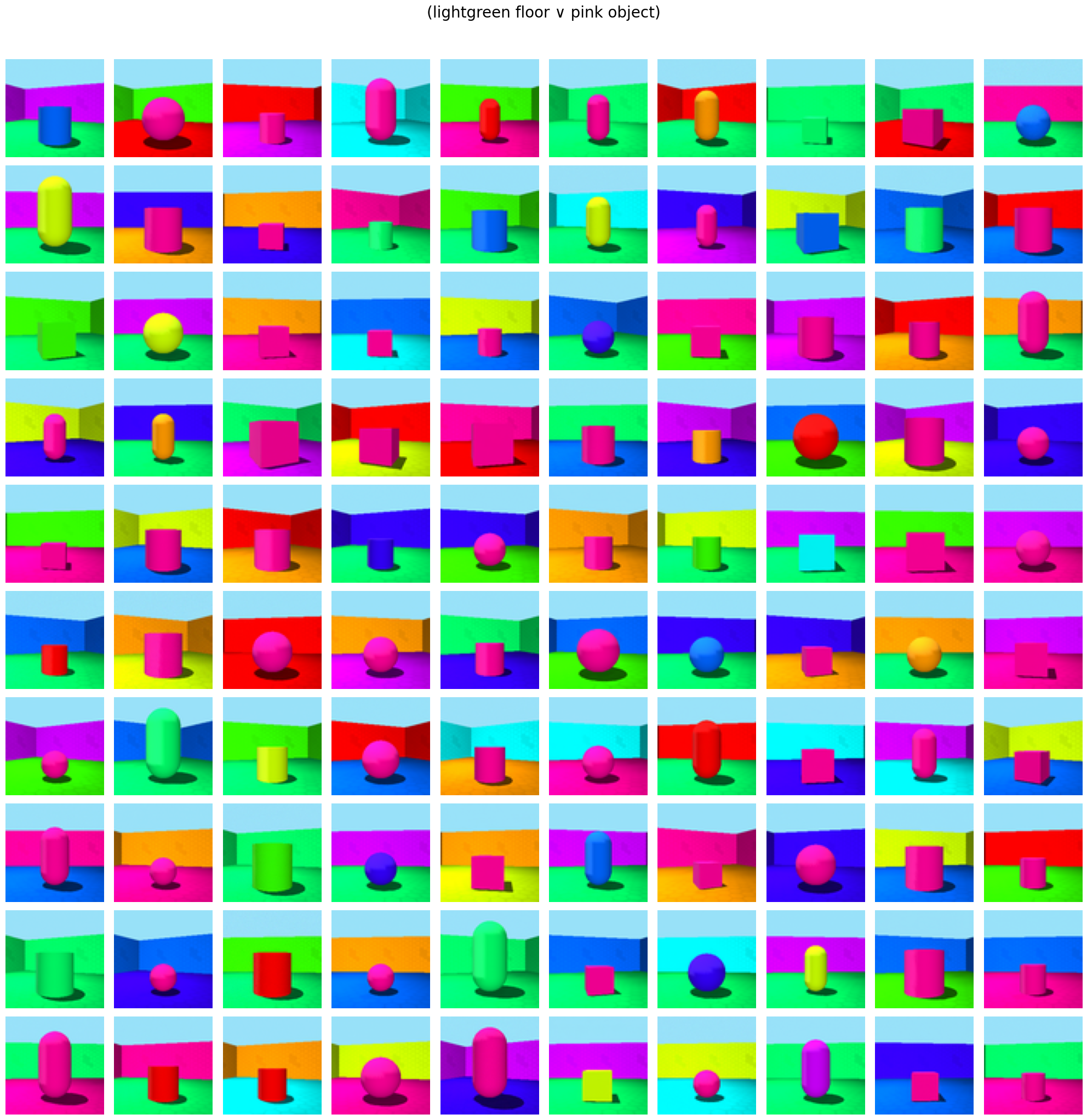}        
    \end{subfigure}
    
    \caption{Batch of results for the OR-CI composition (lightgreen floor $lor$ pink object) on Shapes3D. Conformity scores: Constant = 0.68, Ours = 0.95.}
    \end{minipage}
\end{figure}

\begin{figure}[h!]
    \centering
    \begin{minipage}{0.9\linewidth}
    \setkeys{Gin}{width=\linewidth} 

    \begin{minipage}{0.99\linewidth}
        \scriptsize \centering \textbf{\large ((cube $\lor_\text{CI}$  (\textcolor{shapesRed}{red} wall $\lor_\text{ME}$  \textcolor{shapesPink}{pink} wall)) $\lor_\text{CI}$  (5 scale $\land$ \textcolor{shapesDarkBlue}{darkblue} object))}
    \end{minipage}%

    \vspace{1pt}
    \begin{minipage}{0.49\linewidth}
        \scriptsize \centering \textbf{Constant}
    \end{minipage}%
    \hfill
    \begin{minipage}{0.49\linewidth}
        \scriptsize \centering \textbf{Ours}
    \end{minipage}
    

    \begin{subfigure}{0.49\linewidth}
        \includegraphics[trim={0 0 0 1.7cm}, clip, width=\linewidth]{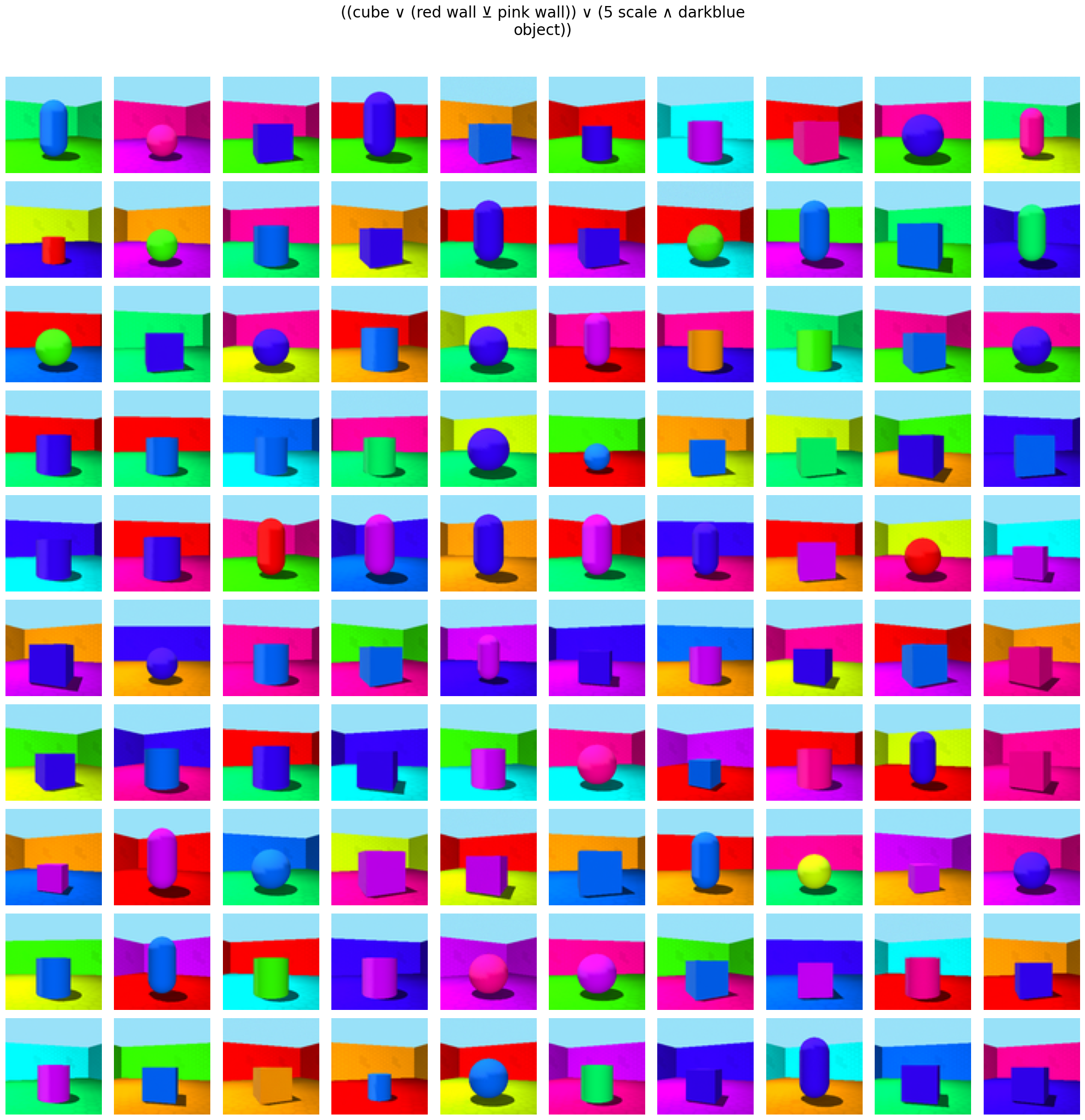}
    \end{subfigure}%
    \hfill
    \begin{subfigure}{0.49\linewidth}
        \includegraphics[trim={0 0 0 1.7cm}, clip, width=\linewidth]{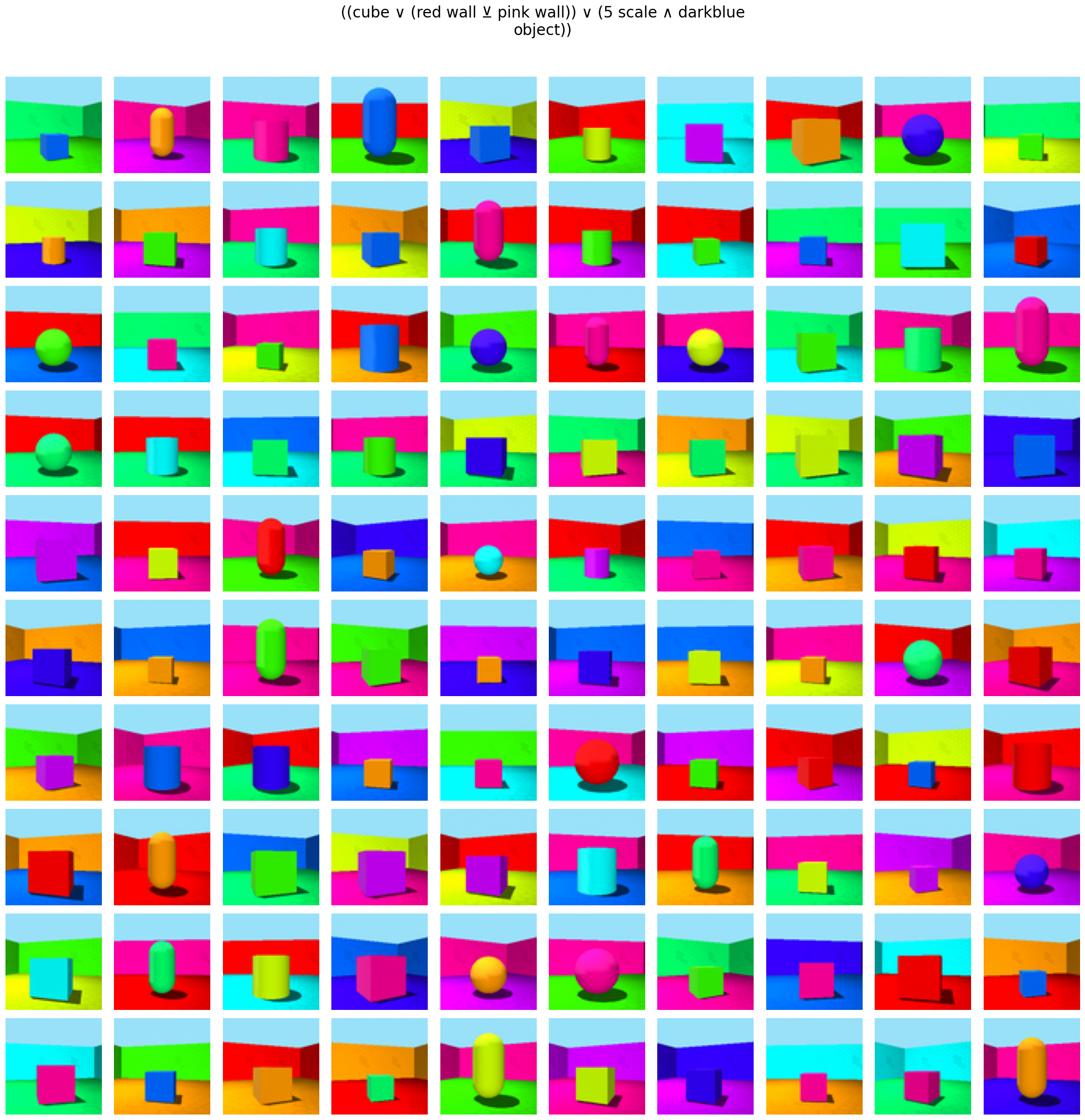}
    \end{subfigure}
    
    \caption{Batch of results for a complex composition with $N=4$ ((cube $\lor_\text{CI}$  (red wall $\lor_\text{ME}$  pink wall)) $\lor_\text{CI}$ (5 scale $\land$ darkblue object)) on Shapes3D. Conformity scores: Constant = 0.71, Ours = 0.98.}
    \end{minipage}
\end{figure}

\begin{figure}[h!]
    \centering
    \begin{minipage}{0.9\linewidth}
    \setkeys{Gin}{width=\linewidth} 

    \begin{minipage}{0.99\linewidth}
        \scriptsize \centering \textbf{\large $\lnot$ blond}
    \end{minipage}%

    \vspace{1pt}
    \begin{minipage}{0.49\linewidth}
        \scriptsize \centering \textbf{Constant}
    \end{minipage}%
    \hfill
    \begin{minipage}{0.49\linewidth}
        \scriptsize \centering \textbf{Ours}
    \end{minipage}
    

    \begin{subfigure}{0.49\linewidth}
        \includegraphics[trim={0 0 0 1.7cm}, clip, width=\linewidth]{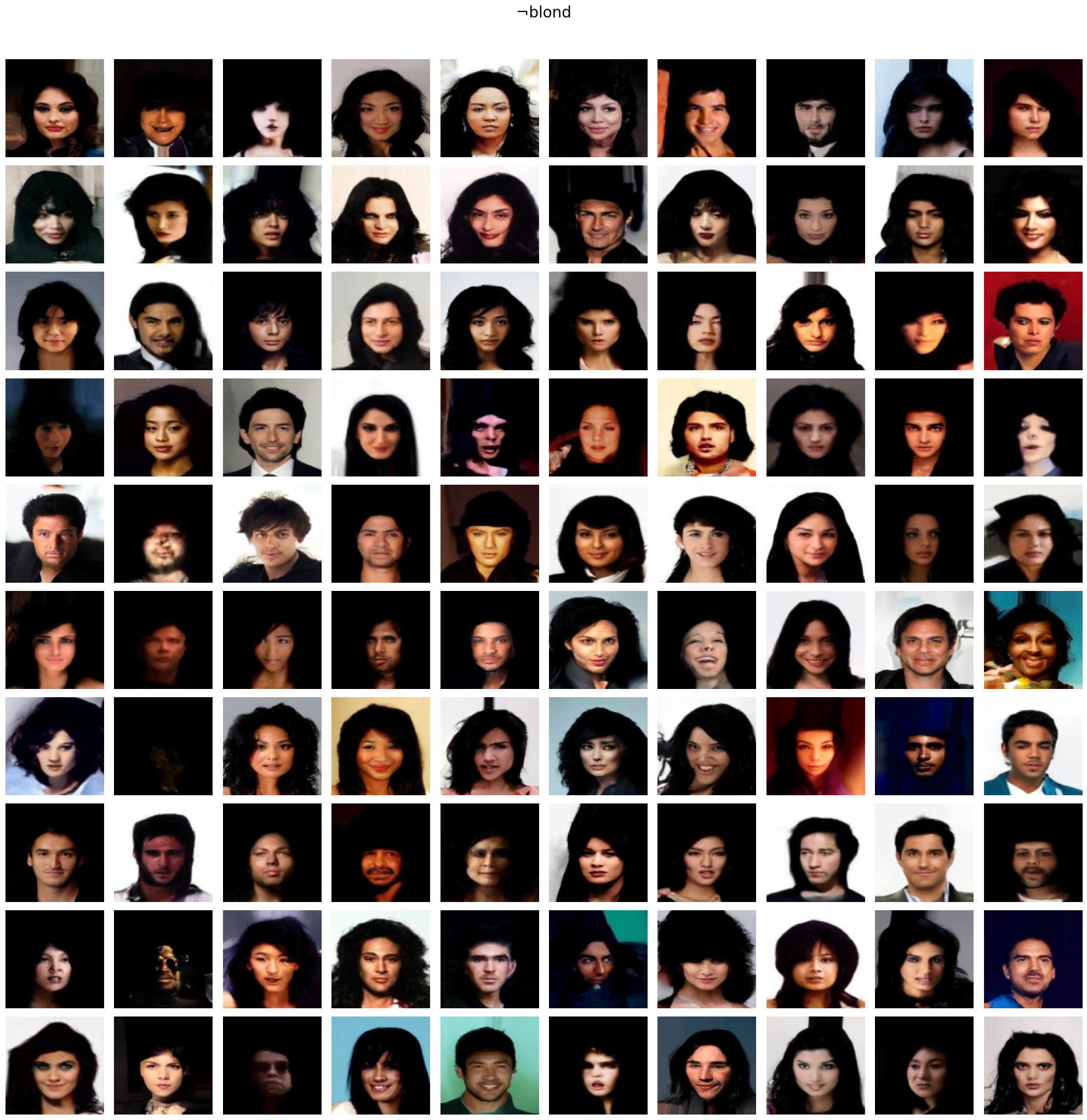}
    \end{subfigure}%
    \hfill
    \begin{subfigure}{0.49\linewidth}
        \includegraphics[trim={0 0 0 1.7cm}, clip, width=\linewidth]{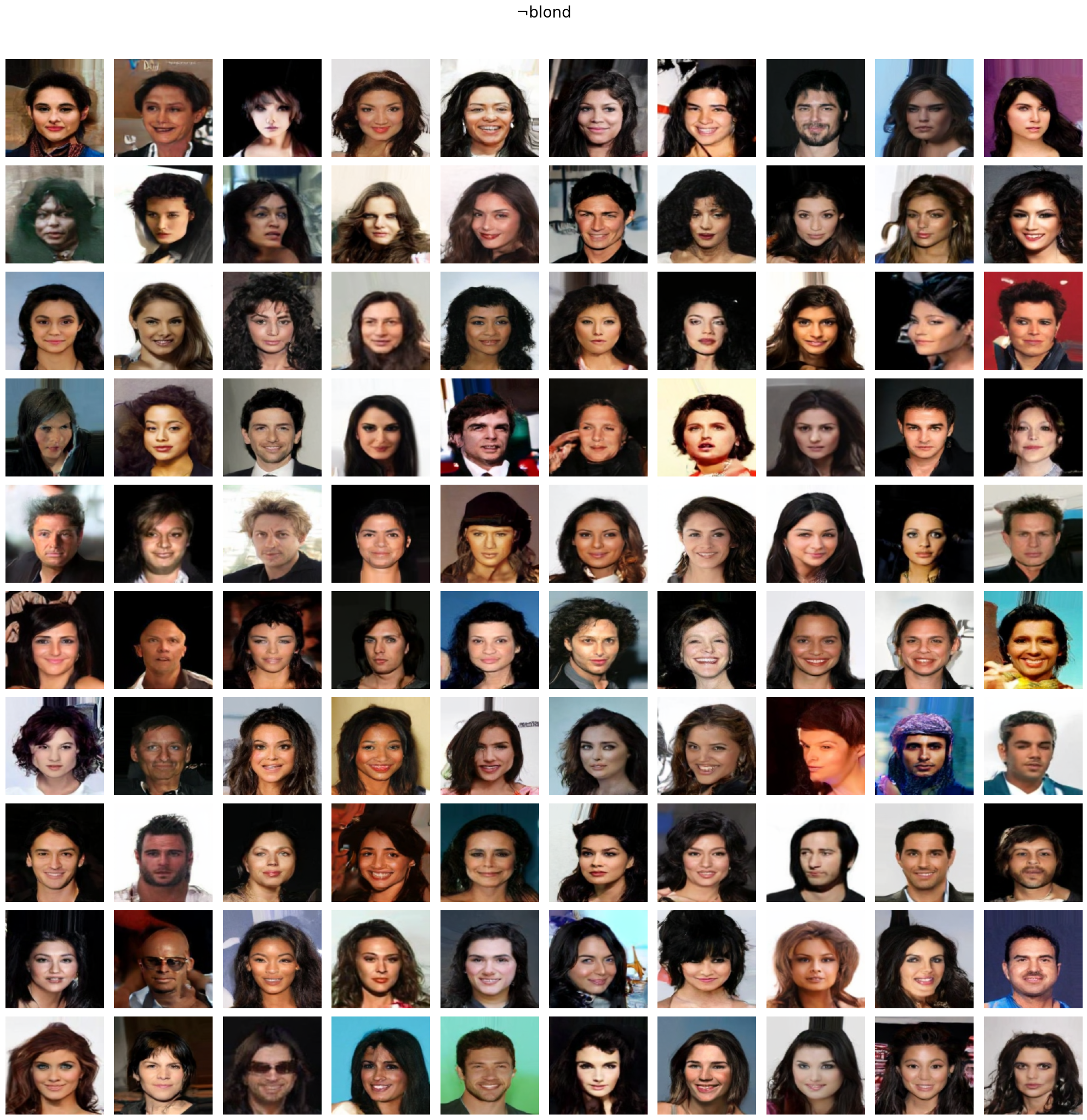}
    \end{subfigure}
    
    \caption{Batch of results for NOT composition ($\lnot$ blond) on CelebA. Conformity scores: Constant = 1.0, Ours = 1.0.}
    \end{minipage}
\end{figure}

\end{document}